\documentclass{article}

\PassOptionsToPackage{numbers, compress}{natbib}
\usepackage[final]{neurips_2022}

\usepackage[utf8]{inputenc} % allow utf-8 input
\usepackage[T1]{fontenc}    % use 8-bit T1 fonts
\usepackage{url}            % simple URL typesetting
\usepackage{booktabs}       % professional-quality tables
\usepackage{amsfonts}       % blackboard math symbols
\usepackage{nicefrac}       % compact symbols for 1/2, etc.
\usepackage{microtype}      % microtypography
\usepackage[usenames, dvipsnames]{color} % Cool colors
\usepackage{enumerate, amsmath, amsthm, amssymb, algorithm, algorithmic, pifont, csquotes, dashrule, tikz, bbm, booktabs, bm, verbatim, url, dsfont}
\usepackage[framemethod=TikZ]{mdframed}
\usepackage[colorlinks,backref=page]{hyperref}
\hypersetup{
     colorlinks = true,
     linkcolor = blue,
     anchorcolor = blue,
     citecolor = red,
     filecolor = blue,
     urlcolor = blue
     }
\usepackage{caption}
\usepackage{subcaption}
\definecolor{dblue}{RGB}{98, 140, 190}
\definecolor{dlblue}{RGB}{216, 235, 255}
\definecolor{dgreen}{RGB}{124, 155, 127}
\definecolor{dpink}{RGB}{207, 166, 208}
\definecolor{dyellow}{RGB}{255, 248, 199}
\definecolor{dgray}{RGB}{46, 49, 49}

\newcommand{\durl}[1]{\textcolor{dblue}{\underline{\url{#1}}}}

\newcommand{\ubr}[1]{\underbrace{#1}}

\newcommand{\eps}{\varepsilon}

\newcommand{\mc}[1]{\mathcal{#1}}

\newcommand{\indic}{\mathbbm{1}}
\newcommand{\bE}{\mathbb{E}}

\newcommand{\bR}{\mathbb{R}}
\newcommand{\bP}{\mathbb{P}}

\newcommand{\bI}{\mathbb{I}}
\newcommand{\bH}{\mathbb{H}}
\newcommand{\bN}{\mathbb{N}}

\newcommand{\bB}{\mathbb{B}}

\newcommand{\bX}{\mathbb{X}}

\newcommand{\kl}[2]{D_{\mathrm{KL}}(#1\text{ }||\text{ }#2)}

\newcommand{\supp}[1]{\text{supp}\left(#1\right)}

\newcommand{\ra}{\rightarrow}

\newcommand{\vol}[1]{\text{Vol}\left(#1\right)}

\DeclareMathOperator*{\argmax}{arg\,max}

\newmdenv[
  topline=false,
  bottomline=false,
  rightline = false,
  leftmargin=10pt,
  rightmargin=0pt,
  innertopmargin=0pt,
  innerbottommargin=0pt
]{innerproof}

\newcounter{DaveDefCounter}
\newtheorem{assumption}{Assumption}
\newtheorem{conjecture}{Conjecture}
\newtheorem{corollary}{Corollary}

\newtheorem{example}{Example}
\newtheorem{lemma}{Lemma}

\newtheorem{theorem}{Theorem}
\newtheorem{fact}{Fact}

\newif\ifsubmit
\submittrue
\ifsubmit
\newcommand{\dnote}[1]{}
\newcommand{\bnote}[1]{}
\else
\newcommand{\dnote}[1]{\textcolor{blue}{Dilip: #1}}
\newcommand{\bnote}[1]{\textcolor{orange}{Ben: #1}}
\fi

\title{Deciding What to Model: Value-Equivalent Sampling for Reinforcement Learning}

\author{%
  Dilip Arumugam\\
  Department of Computer Science\\
  Stanford University\\
  \texttt{dilip@cs.stanford.edu}\\
  \And
  Benjamin Van Roy \\
  Department of Electrical Engineering \\
  Department of Management Science \& Engineering\\
  Stanford University\\
  \texttt{bvr@stanford.edu} \\
}

\begin{document}

\maketitle

\begin{abstract}
  The quintessential model-based reinforcement-learning agent iteratively refines its estimates or prior beliefs about the true underlying model of the environment. Recent empirical successes in model-based reinforcement learning with function approximation, however, eschew the true model in favor of a surrogate that, while ignoring various facets of the environment, still facilitates effective planning over behaviors. Recently formalized as the value equivalence principle, this algorithmic technique is perhaps unavoidable as real-world reinforcement learning demands consideration of a simple, computationally-bounded agent interacting with an overwhelmingly complex environment, whose underlying dynamics likely exceed the agent's capacity for representation. In this work, we consider the scenario where agent limitations may entirely preclude identifying an exactly value-equivalent model, immediately giving rise to a trade-off between identifying a model that is simple enough to learn while only incurring bounded sub-optimality. To address this problem, we introduce an algorithm that, using rate-distortion theory, iteratively computes an approximately-value-equivalent, lossy compression of the environment which an agent may feasibly target in lieu of the true model. We prove an information-theoretic, Bayesian regret bound for our algorithm that holds for any finite-horizon, episodic sequential decision-making problem. Crucially, our regret bound can be expressed in one of two possible forms, providing a performance guarantee for finding either the simplest model that achieves a desired sub-optimality gap or, alternatively, the best model given a limit on agent capacity.
\end{abstract}

\section{Introduction}

A central challenge of the reinforcement-learning problem~\citep{sutton1998introduction,kaelbling1996reinforcement}  is exploration, where a sequential decision-making agent must judiciously balance exploitation of knowledge accumulated thus far against the need to further acquire information for optimal long-term performance. Historically, provably-efficient reinforcement-learning algorithms~\citep{kearns2002near,brafman2002r,kakade2003sample,auer2009near,bartlett2009regal,strehl2009reinforcement,jaksch2010near,osband2013more,dann2015sample,osband2017posterior,azar2017minimax,dann2017unifying,agrawal2017optimistic,jin2018q,zanette2019tighter,dong2021simple,lu2021reinforcement} have often relied upon one of two possible mechanisms for addressing the exploration challenge in a principled manner: optimism in the face of uncertainty or posterior sampling. Briefly, methods in the former category begin with optimistically-biased value estimates for all state-action pairs; an agent acting greedily with respect to these estimates will be incentivized to visit all state-action pairs a sufficient number of times until this bias dissipates and the agent is left with an accurate estimate of the value function for deriving optimal behavior. In contrast, posterior-sampling methods primarily operate based on Thompson sampling~\citep{thompson1933likelihood,russo2018tutorial} whereby the agent begins with a prior belief over the Markov Decision Process (MDP) with which it is interacting and acts optimally with respect to a single sample drawn from these beliefs. The resulting experience sampled from the true environment allows the agent to derive a corresponding posterior distribution and this Posterior Sampling for Reinforcement Learning (PSRL)~\citep{strens2000bayesian} algorithm proceeds iteratively in this manner, eventually arriving at a posterior sharply concentrated around the true environment MDP. While both paradigms have laid down solid theoretical foundations for provably-efficient reinforcement learning, a line of work has demonstrated how posterior-sampling methods can be more favorable both in theory and in practice~\citep{osband2013more,osband2016deep,osband2016generalization,osband2017posterior,o2018uncertainty,osband2019deep,dwaracherla2020hypermodels}.

While existing analyses of reinforcement-learning algorithms have largely focused on providing guarantees for learning optimal solutions, real-world reinforcement learning demands consideration for a computationally-bounded agent interacting with an overwhelmingly complex environment~\citep{lu2021reinforcement}. A simplified view of this notion can be succinctly depicted in the multi-armed bandit setting~\citep{lai1985asymptotically,bubeck2012regret,lattimore2020bandit}; as the number of arms increases, a Thompson sampling agent's relentless pursuit of the optimal arm will lead to large regret~\citep{russo2018satisficing,russo2022satisficing}. On the other hand, one might simply settle for the first $\eps$-optimal arm found, for some $\eps > 0$, which may be identified in far fewer time periods. The goal of this work is to augment PSRL so as to accommodate these satisficing solutions in addition to optimal ones, paralleling existing work for satisficing in multi-armed bandit problems~\citep{russo2017time,russo2022satisficing,arumugam2021deciding,arumugam2021the}. To help elucidate the utility of satisficing solutions in the reinforcement-learning setting, we offer the following illustrative example:
\begin{example}[A Multi-Resolution MDP]
For a large but finite $N \in \bN$, consider a sequence of MDPs, $\{\mc{M}_n\}_{n \in [N]}$, which all share a common action space $\mc{A}$ but vary in state space $\mc{S}_n$, reward function, and transition function. Moreover, for each $n \in [N]$, the rewards of the $n$th MDP are bounded in the interval $[0,\frac{1}{n}]$. An agent is confronted with the resulting product MDP, $\mc{M}$, defined on the state space $\mc{S}_1 \times \ldots \times \mc{S}_N$ with action space $\mc{A}$ and rewards summed across the $N$ constituent reward functions. The transition function is defined such that each action $a \in \mc{A}$ is executed across all $N$ MDPs simultaneously and the resulting individual transitions are combined into a transition of $\mc{M}$.
\label{example:multi_res_mdps}
\end{example}
Example \ref{example:multi_res_mdps} presents a simple scenario where, as $N \uparrow \infty$, a complex environment retains a wealth of information and yet, due to the scale of $N$ and the boundedness of rewards for each constituent MDP $\mc{M}_n$, only a subset of that information is within the agent's reach or even necessary for producing reasonably competent behavior. Despite this fact, PSRL will persistently act to fully identify the transition and reward structure of all $\{\mc{M}_n\}_{n \in [N]}$, for any value of $N$. Without knowing which MDPs are more important \textit{a priori} and even as data accumulates during learning, PSRL is unable to forego learning granular components of $\mc{M}$, eventually accumulating optimal reward at the cost of more time. Intuitively, however, one might anticipate that there exists a value $M \ll N$ such that learning the subsequence of MDPs $\{\mc{M}_n\}_{n \in [M]}$ in fewer time periods is sufficient for achieving a desired degree of sub-optimality, since the rewards of the remaining MDPs $\{\mc{M}_n\}_{n > M}$ make suitably negligible contributions to the overall rewards of $\mc{M}$. Alternatively, for a computationally-bounded decision maker, the agent's resource limitations ought to translate into a value $C \ll N$ such that $\{\mc{M}_n\}_{n \in [C]}$ is feasible and learning this subsequence is the best possible outcome under the agent capacity constraints.  In this work, we introduce an algorithm that, in a purely data-driven and automated fashion, implicitly identifies such a value $M$ or $C$ to facilitate tractable, near-optimal learning in what may otherwise be an intractable problem. Following \citet{arumugam2021deciding}, a key tool for defining a notion of satisficing in reinforcement learning will be rate-distortion theory~\citep{shannon1959coding,berger1971rate}.

% Example \ref{example:multi_res_mdps} presents a simple scenario where, as $N \uparrow \infty$, a complex environment retains a wealth of information and yet, due to the scale of $N$ and the boundedness of rewards for each constituent MDP $\mc{M}_n$, only a subset of that information is within the agent's reach or even necessary for producing reasonably competent behavior. Despite this fact, PSRL will persistently act to fully identify the transition and reward structure of all $\{\mc{M}_n\}_{n \in [N]}$, for any value of $N$. In this work, we introduce an algorithm that, in a purely data-driven and automated fashion, implicitly identifies either a $M \ll N$ such that learning the subsequence of MDPs $\{\mc{M}_n\}_{n \in [M]}$ is sufficient for achieving a desired degree of sub-optimality or a $C \ll N$ such that $\{\mc{M}_n\}_{n \in [C]}$ falls within the known capacity constraints of the agent. Following \citet{arumugam2021deciding}, a key tool for defining a notion of satisficing in reinforcement learning will be rate-distortion theory~\citep{shannon1959coding,berger1971rate}.

The paper proceeds as follows: we introduce our problem formulation in Section \ref{sec:problem_form}, present our generalization of PSRL in Section \ref{sec:sat_psrl}, and provide a complementary regret analysis in Section \ref{sec:analysis}. Due to space constraints, technical proofs, an overview of related work, and discussion of our results in a broader context are relegated to the appendix. 
% We strongly encourage readers to consult Section \ref{sec:prelims} for clarity on our notation.

\section{Preliminaries}
% \label{sec:prelims}

In this section, we provide brief background on information theory and details on our notation. All random variables are defined on a probability space $(\Omega, \mc{F}, \bP)$. For any random variable $X:\Omega \ra \mc{X}$ taking values on the measurable space $(\mc{X}, \bX)$, we use $\sigma(X) \triangleq \{X^{-1}(A) \mid A \in \bX\} \subseteq \mc{F}$ to denote the $\sigma$-algebra generated by $X$. For any natural number $N \in \bN$, we denote the index set as $[N] \triangleq \{1,2,\ldots,N\}$. For any arbitrary set $\mc{X}$, $\Delta(\mc{X})$ denotes the set of all probability distributions with support on $\mc{X}$. For any two arbitrary sets $\mc{X}$ and $\mc{Y}$, we denote the class of all (measurable) functions mapping from $\mc{X}$ to $\mc{Y}$ as $\{\mc{X} \ra \mc{Y}\} \triangleq \{f \mid f:\mc{X} \ra \mc{Y}\}$. While our exposition throughout the paper will consistently refer to bits of information, it will be useful for the purposes of analysis that all logarithms be in base $e$.

\subsection{Information Theory}
% \label{sec:info_theory}

Here we introduce various concepts in probability theory and information theory used throughout this paper. We encourage readers to consult \citep{cover2012elements,gray2011entropy,polyanskiy2019lecture,duchi21ItLectNotes} for more background. 

 We define the mutual information between any two random variables $X,Y$ through the Kullback-Leibler (KL) divergence: $$\bI(X;Y) = \kl{\bP((X,Y) \in \cdot)}{\bP(X \in \cdot) \times \bP(Y \in \cdot)} \qquad \kl{P}{Q} = \begin{cases} \int \log\left(\frac{dP}{dQ}\right) dP & P \ll Q \\ +\infty & P \not\ll Q \end{cases},$$ where $P$ and $Q$ are both probability measures on the same measurable space and $\frac{dP}{dQ}$ denotes the Radon-Nikodym derivative of $P$ with respect to $Q$. An analogous definition of conditional mutual information holds through the expected KL-divergence for any three random variables $X,Y,Z$:
$$\bI(X;Y \mid Z) = \bE\left[\kl{\bP((X,Y) \in \cdot \mid Z)}{\bP(X \in \cdot \mid Z) \times \bP(Y \in \cdot \mid Z)}\right].$$
With these definitions in hand, we may define the entropy and conditional entropy for any two random variables $X,Y$ as $$\bH(X) = \bI(X;X) \qquad \bH(Y \mid X) = \bH(Y) - \bI(X;Y).$$ This yields the following identities for mutual information and conditional mutual information for any three arbitrary random variables $X$, $Y$, and $Z$:
$$\bI(X;Y) = \bH(X) - \bH(X \mid Y) = \bH(Y) - \bH(Y | X), \qquad \bI(X;Y|Z) = \bH(X|Z) - \bH(X \mid Y,Z) = \bH(Y|Z) - \bH(Y | X,Z).$$
Through the chain rule of the KL-divergence and the fact that $\kl{P}{P} = 0$ for any probability measure $P$, we obtain another equivalent definition of mutual information, $$\bI(X;Y) = \bE\left[\kl{\bP(Y \in \cdot \mid X)}{\bP(Y \in \cdot)}\right],$$ as well as the chain rule of mutual information: $\bI(X;Y_1,\ldots,Y_n) = \sum\limits_{i=1}^n \bI(X;Y_i \mid Y_1,\ldots,Y_{i-1}).$ Finally, for any three random variables $X$, $Y$, and $Z$ which form the Markov chain $X \ra Y \ra Z$, we have the following data-processing inequality: $\bI(X;Z) \leq \bI(X;Y).$

% \section{Preliminaries}
\section{Problem Formulation}
\label{sec:problem_form}

We formulate a sequential decision-making problem as a finite-horizon, episodic Markov Decision Process (MDP)~\citep{bellman1957markovian,Puterman94} defined by $\mc{M} = \langle \mc{S}, \mc{A}, \mc{R}, \mc{T}, \beta, H \rangle$. Here $\mc{S}$ denotes a set of states, $\mc{A}$ is a set of actions, $\mc{R}:\mc{S} \times \mc{A} \ra [0,1]$ is a deterministic reward function providing evaluative feedback signals (in the unit interval) to the agent, $\mc{T}:\mc{S} \times \mc{A} \ra \Delta(\mc{S})$ is a transition function prescribing distributions over next states, $\beta \in \Delta(\mc{S})$ is an initial state distribution, and $H \in \bN$ is the maximum episode length or horizon. 

As is standard in Bayesian reinforcement learning~\citep{ghavamzadeh2015bayesian}, neither the transition function nor the reward function are known to the agent and, consequently, both are treated as random variables. Since all other components of the MDP are thought of as known a priori, the randomness in the model $(\mc{R}, \mc{T})$ fully accounts for the randomness in $\mc{M}$, which is also a random variable. We denote by $\mc{M}^\star$ the true MDP with model $(\mc{R}^\star, \mc{T}^\star)$ that the agent interacts with and attempts to solve over the course of $K$ episodes. Within each episode, the agent acts for exactly $H$ steps beginning with an initial state $s_1 \sim \beta$. For each $h \in [H]$, the agent observes the current state $s_h \in \mc{S}$, selects action $a_h \sim \pi_h(\cdot \mid s_h) \in \mc{A}$, enjoys a reward $r_h = \mc{R}(s_h,a_h) \in [0,1]$, and transitions to the next state $s_{h+1} \sim \mc{T}(\cdot \mid s_h, a_h) \in \mc{S}$.

A stationary, stochastic policy for timestep $h \in [H]$, $\pi_h:\mc{S} \ra \Delta(\mc{A})$, encodes a pattern of behavior mapping individual states to distributions over possible actions. Letting $\{\mc{S} \ra \Delta(\mc{A})\}$ denote the class of all stationary, stochastic policies, a non-stationary policy $\pi = (\pi_1,\ldots,\pi_H) \in \{\mc{S} \ra \Delta(\mc{A})\}^H$ is a collection of exactly $H$ stationary, stochastic policies whose overall performance in any MDP $\mc{M}$ at timestep $h \in [H]$ when starting at state $s \in \mc{S}$ and taking action $a \in \mc{A}$ is assessed by its associated action-value function $Q^\pi_{\mc{M},h}(s,a) = \bE\left[\sum\limits_{h'=h}^H \mc{R}(s_{h'},a_{h'}) \bigm| s_h = s, a_h = a\right]$, where the expectation integrates over randomness in the action selections and transition dynamics. Taking the corresponding value function as $V^\pi_{\mc{M},h}(s) = \bE_{a \sim \pi_h(\cdot \mid s)}\left[Q^\pi_{\mc{M},h}(s,a)\right]$, we define the optimal policy $\pi^\star = (\pi^\star_1,\pi^\star_2,\ldots,\pi^\star_H)$ as achieving supremal value $V^\star_{\mc{M},h}(s) = \sup\limits_{\pi \in \{\mc{S} \ra \Delta(\mc{A})\}^H} V^\pi_{\mc{M},h}(s)$ for all $s \in \mc{S}$, $h \in [H]$. For brevity, we will write any value function $V \in \{S \ra \bR\}$ without its argument to implicitly integrate over randomness in the initial state: $V = \bE_{s_1 \sim \beta(\cdot)}\left[V(s_1)\right]$. We let $\tau_k = (s^{(k)}_1, a^{(k)}_1, r^{(k)}_1, \ldots,s^{(k)}_{H}, a^{(k)}_{H}, r^{(k)}_{H}, s^{(k)}_{H+1})$ be the random variable denoting the trajectory experienced by the agent in the $k$th episode. Meanwhile, $H_k = \{\tau_1,\tau_2,\ldots, \tau_{k-1}\} \in \mc{H}_k$ is the random variable representing the entire history of the agent's interaction within the environment at the start of the $k$th episode; the sequence of history random variables $\{H_k\}_{k \in [K]}$ induce and, by definition, are adapted to the filtration $\{\sigma(H_k)\}_{k \in [K]}$ of $(\Omega, \mc{F})$. We call attention to the fact that we have yet to make any further restrictions on the state-action space $\mc{S} \times \mc{A}$, such as finiteness; notably, the main results of this paper are not limited to tabular MDPs. As mentioned by \citet{lattimore2020bandit} (also as Proposition 7.28 of \citet{bertsekas1996stochastic}), the Ionescu-Tulcea Theorem~\citep{tulcea1949mesures} ensures the existence of a probability space upon which $\tau_k$ and $H_k$ are well-defined random variables for all episodes $k \in [K]$.

% \subsection{Assessing Performance}

Abstractly, a reinforcement-learning algorithm is a sequence of non-stationary policies $(\pi^{(1)},\ldots,\pi^{(K)})$ where for each episode $k \in [K]$, $\pi^{(k)}:\mc{H}_k \ra \{\mc{S} \ra \Delta(\mc{A})\}^H$ is a function of the current history $H_k$. We define the regret of a reinforcement-learning algorithm over $K$ episodes as 
$$\textsc{Regret}(K, \pi^{(1)},\ldots,\pi^{(K)}, \mc{M}^\star) = \sum\limits_{k=1}^K \Delta_k \qquad \Delta_k \triangleq V^\star_{\mc{M}^\star,1} - V^{\pi^{(k)}}_{\mc{M}^\star, 1},$$
% $$\textsc{Regret}(K, \pi^{(1)},\ldots,\pi^{(K)}, \mc{M}^\star) = \sum\limits_{k=1}^K \Delta_k \qquad \Delta_k \triangleq \sum\limits_{s_1 \in \mc{S}} \beta(s_1) \left(V^\star_{\mc{M}^\star,1}(s_1) - V^{\pi^{(k)}}_{\mc{M}^\star, 1}(s_1)\right),$$
% \begin{align*}
%     \textsc{Regret}(K, \pi^{(1)},\ldots,\pi^{(K)}, \mc{M}^\star) &= \sum\limits_{k=1}^K \Delta_k,
% \end{align*}
where $\Delta_k$ denotes the episodic regret or regret incurred during the $k$th episode with respect to the true MDP $\mc{M}^\star$.
% \begin{align*}
%     \Delta_k \triangleq \sum\limits_{s_1 \in \mc{S}} \beta(s_1) \left(V^\star_{\mc{M}^\star,1}(s_1) - V^{\pi^{(k)}}_{\mc{M}^\star, 1}(s_1)\right).
% \end{align*}
An agent's initial uncertainty in the (unknown) true MDP $\mc{M}^\star$ is reflected by an arbitrary prior distribution $\bP(\mc{M}^\star \in \cdot \mid H_1)$. Since the regret is a random variable due to our uncertainty in $\mc{M}^\star$, we integrate over this randomness to arrive at the Bayesian regret:
\begin{align*}
    \textsc{BayesRegret}(K, \pi^{(1)},\ldots,\pi^{(K)}) &= \bE\left[\textsc{Regret}(K, \pi^{(1)},\ldots,\pi^{(K)}, \mc{M}^\star)\right].
\end{align*}
Broadly speaking, our goal is to design a provably-efficient reinforcement-learning algorithm that incurs bounded Bayesian regret. 
% We strongly encourage readers to consult Section \ref{sec:prelims} for clarity on our conditioning notation.

Throughout the paper, we will denote the entropy and conditional entropy conditioned upon a specific realization of an agent's history $H_k$, for some episode $k \in [K]$, as $\bH_k(X) \triangleq \bH(X \mid H_k = H_k)$ and $\bH_k(X \mid Y) \triangleq \bH_k(X \mid Y, H_k = H_k)$, for two arbitrary random variables $X$ and $Y$. This notation will also apply analogously to the mutual information $\bI_k(X;Y) \triangleq \bI(X;Y \mid H_k = H_k) = \bH_k(X) - \bH_k(X \mid Y) = \bH_k(Y) - \bH_k(Y \mid X),$ as well as the conditional mutual information $\bI_k(X;Y \mid Z) \triangleq \bI(X;Y \mid H_k = H_k, Z),$ given an arbitrary third random variable, $Z$. Note that their dependence on the realization of random history $H_k$ makes both $\bI_k(X;Y)$ and $\bI_k(X;Y \mid Z)$ random variables themselves. The traditional notion of conditional mutual information given the random variable $H_k$ arises by integrating over this randomness: $$\bE\left[\bI_k(X;Y)\right] = \bI(X;Y \mid H_k) \qquad \bE\left[\bI_k(X;Y \mid Z)\right] = \bI(X;Y \mid H_k,Z).$$ Additionally, we will also adopt a similar notation to express a conditional expectation given the random history $H_k$: $\bE_k\left[X\right] \triangleq \bE\left[X|H_k\right].$

\section{Satisficing Through Posterior Sampling}
\label{sec:sat_psrl}

\subsection{Rate-Distortion Theory}
We begin with a brief, high-level overview of rate-distortion theory~\citep{shannon1959coding,berger1971rate} and encourage readers to consult \citep{cover2012elements} for more details and \citep{berger1998lossy} for a survey of advances in rate-distortion theory towards solving the lossy source coding problem in information theory. A lossy compression problem consumes as input a fixed information source $\bP(X \in \cdot)$ and a measurable distortion function $d: \mc{X} \times \mc{Z} \ra \bR_{\geq 0}$ which quantifies the loss of fidelity by using $Z$ in place of $X$. Then, for any $D \in \bR_{\geq 0}$, the rate-distortion function quantifies the fundamental limit of lossy compression as $$\mc{R}(D) = \inf\limits_{Z \in \Lambda(D)} \bI(X;Z) \qquad \Lambda(D) \triangleq \{Z: \Omega \ra \mc{Z} \mid \bE\left[d(X,Z)\right] \leq D\},$$ where the infimum is taken over all random variables $Z$ that incur bounded expected distortion, $\bE\left[d(X,Z)\right] \leq D$. Naturally, $\mc{R}(D)$ represents the minimum number of bits of information that must be retained from $X$ in order to achieve this bounded expected loss of fidelity\footnote{With a slight abuse of notation, we overload $\mc{R}$.}. Throughout the paper, various facts of the rate-distortion function will be referenced as needed. For now, we simply note that, in keeping with the problem formulation of the previous section which does not automatically assume discrete random variables, the rate-distortion function is well-defined for abstract information source and channel output random variables~\citep{csiszar1974extremum}.

Just as in past work that studies satisficing in multi-armed bandit problems~\citep{russo2018satisficing,russo2022satisficing,arumugam2021deciding}, we will use rate-distortion theory to formalize and identify the best simplified MDP $\widetilde{\mc{M}}_k$ that the agent will attempt to learn over the course of each episode $k \in [K]$. The dependence on the particular episode comes from the fact that this lossy compression mechanism or channel will treat the agent's current beliefs over the true MDP $\bP(\mc{M}^\star \in \cdot \mid H_k)$ as the information source to be compressed.

\subsection{The Value Equivalence Principle}

As outlined in the previous section, the second input for a well-specified lossy-compression problem is a distortion function prescribing non-negative real values to realizations of the information source and channel output random variables $(\mc{M}^\star, \widetilde{\mc{M}})$ that quantify the loss of fidelity incurred by using $\widetilde{\mc{M}}$ in lieu of $\mc{M}^\star$. To define this function, we will leverage an approximate notion of value equivalence~\citep{grimm2020value,grimm2021proper}. For any arbitrary MDP $\mc{M}$ with model $(\mc{R},\mc{T})$ and any stationary, stochastic policy $\pi:\mc{S} \ra \Delta(\mc{A})$, define the Bellman operator $\mc{B}^\pi_\mc{M}: \{\mc{S} \ra \bR\} \ra \{\mc{S} \ra \bR\}$ as follows: $$\mc{B}^\pi_\mc{M}V(s) \triangleq \bE_{a \sim \pi(\cdot \mid s)}\left[\mc{R}(s,a) + \bE_{s' \sim \mc{T}(\cdot \mid s, a)}\left[ V(s')\right]\right], \qquad \forall s \in \mc{S}.$$ The Bellman operator is a foundational tool in dynamic-programming approaches to reinforcement learning~\citep{bertsekas1995dynamic} and gives rise to the classic Bellman equation: for any MDP $\mc{M} = \langle \mc{S}, \mc{A}, \mc{R}, \mc{T}, \beta, H \rangle$ and any non-stationary policy $\pi = (\pi_1,\ldots,\pi_H)$, the value functions induced by $\pi$ satisfy $V^\pi_{\mc{M},h}(s) = \mc{B}^{\pi_h}_{\mc{M}}V^\pi_{\mc{M},h+1}(s),$ for all $h \in [H]$ and with $V^\pi_{\mc{M},H+1}(s) = 0$, $\forall s \in \mc{S}$. For any two MDPs $\mc{M} = \langle \mc{S}, \mc{A}, \mc{R}, \mc{T}, \beta, H \rangle$ and $\widehat{\mc{M}} = \langle \mc{S}, \mc{A}, \widehat{\mc{R}}, \widehat{\mc{T}}, \beta, H \rangle$, \citet{grimm2020value} define a notion of equivalence between them despite their differing models. For any policy class $\Pi \subseteq \{\mc{S} \ra \Delta(\mc{A})\}$ and value function class $\mc{V} \subseteq \{\mc{S} \ra \bR\}$, $\mc{M}$ and $\widehat{\mc{M}}$ are value equivalent with respect to $\Pi$ and $\mc{V}$ if and only if $\mc{B}^\pi_{\mc{M}}V = \mc{B}^\pi_{\widehat{\mc{M}}}V$, $\forall \pi \in \Pi, V \in \mc{V}.$ In words, two different models are deemed value equivalent if they induce identical Bellman updates under any pair of policy and value function from $\Pi \times \mc{V}$. \citet{grimm2020value} prove that when $\Pi = \{\mc{S} \ra \Delta(\mc{A})\}$ and $\mc{V} = \{\mc{S} \ra \bR\}$, the set of all exactly value-equivalent models is a singleton set containing only the true model of the environment. The key insight behind value equivalence, however, is that practical model-based reinforcement-learning algorithms need not be concerned with modeling every granular detail of the underlying environment and may, in fact, stand to benefit by optimizing an alternative criterion besides the traditional maximum-likelihood objective~\citep{silver2017predictron,farahmand2017value,oh2017value,asadi2018lipschitz,farahmand2018iterative,d2020gradient,abachi2020policy,cui2020control,ayoub2020model,schrittwieser2020mastering,nair2020goal,nikishin2022control,voelcker2022value}. Indeed, by restricting focus to decreasing subsets of policies $\Pi \subset \{\mc{S} \ra \Delta(\mc{A})\}$ and value functions $\mc{V} \subset \{\mc{S} \ra \bR\}$, the space of exactly value-equivalent models is monotonically increasing. 

For brevity, let $\mathfrak{R} \triangleq \{\mc{S} \times \mc{A} \ra [0,1]\}$ and $\mathfrak{T} \triangleq \{\mc{S} \times \mc{A} \ra \Delta(\mc{S})\}$ denote the classes of all reward functions and transition functions, respectively. Recall that, with $\langle \mc{S}, \mc{A}, \beta, H \rangle$ all known, the uncertainty in a random MDP $\mc{M}$ is entirely driven by its model $(\mc{R},\mc{T})$ such that we may think of the support of $\mc{M}^\star$ as $\text{supp}(\mc{M}^\star) = \mathfrak{M} \triangleq \mathfrak{R} \times \mathfrak{T}$. We define a distortion function on pairs of MDPs $d:\mathfrak{M} \times \mathfrak{M} \ra \bR_{\geq 0}$ for any $\Pi \subseteq \{\mc{S} \ra \Delta(\mc{A})\}$, $\mc{V} \subseteq \{\mc{S} \ra \bR\}$ as $$d_{\Pi,\mc{V}}(\mc{M},\widehat{\mc{M}}) = \sup\limits_{\substack{\pi \in \Pi \\ V \in \mc{V}}} ||\mc{B}^\pi_{\mc{M}}V - \mc{B}^\pi_{\widehat{\mc{M}}}V||_\infty^2 = \sup\limits_{\substack{\pi \in \Pi \\ V \in \mc{V}}} \left(\sup\limits_{s \in \mc{S}} |\mc{B}^\pi_{\mc{M}}V(s) - \mc{B}^\pi_{\widehat{\mc{M}}}V(s)| \right)^2.$$ In words, $d_{\Pi,\mc{V}}$ is the supremal squared Bellman error between MDPs $\mc{M}$ and $\widehat{\mc{M}}$ across all states $s \in \mc{S}$ with respect to the policy class $\Pi$ and value function class $\mc{V}$.

\subsection{Value-Equivalent Sampling for Reinforcement Learning}

By virtue of the previous two sections, we are now in a position to define the lossy compression problem that characterizes a MDP $\widetilde{\mc{M}}$ that the agent will aspire to learn in each episode $k \in [K]$ instead of the true MDP $\mc{M}^\star$. For any $\Pi \subseteq \{\mc{S} \ra \Delta(\mc{A})\}$; $\mc{V} \subseteq \{\mc{S} \ra \bR\}$; $k \in [K]$; and $D \geq 0$, we define the rate-distortion function
\begin{align}
    \mc{R}^{\Pi,\mc{V}}_k(D) = \inf\limits_{\widetilde{\mc{M}} \in \Lambda_k(D)} \bI_k(\mc{M}^\star; \widetilde{\mc{M}}), \text{          }\Lambda_k(D) \triangleq \{\widetilde{\mc{M}}: \Omega \ra \mathfrak{M} \mid \bE_k[d_{\Pi,\mc{V}}(\mc{M}^\star, \widetilde{\mc{M}})] \leq D\} .
    \label{eq:rdf}
\end{align}
% $$\mc{R}^{\Pi,\mc{V}}_k(D) = \inf\limits_{Q: \mathfrak{M} \times \sigma(\mathfrak{M}) \ra [0,1]} \bI(\mc{M}^\star, \widetilde{\mc{M}}) \text{ such that } \bE\left[d_{\Pi,\mc{V}}(\mc{M}^\star, \widetilde{\mc{M}}) \bigm| H_k \right] \leq D.$$ 
This rate-distortion function characterizes the fundamental limit of MDP compression under our chosen distortion measure resulting in a channel that retains the minimum amount of information from the true MDP $\mc{M}^\star$ while yielding an approximately value-equivalent MDP in expectation. Observe that this distortion constraint is a notion of approximate value equivalence which collapses to the exact value equivalence of \citet{grimm2020value} as $D \ra 0$. Meanwhile, as $D \ra \infty$, we accommodate a more aggressive compression of the true MDP $\mc{M}^\star$ resulting in less faithful Bellman updates.

\begin{center}
\begin{minipage}{0.41\textwidth}
\vspace{-38pt}
\begin{algorithm}[H]
   \caption{Posterior Sampling for Reinforcement Learning (PSRL)~\citep{strens2000bayesian}}
   \label{alg:psrl}
\begin{algorithmic}
   \STATE {\bfseries Input:} Prior $\bP(\mc{M}^\star \in \cdot \mid H_1)$
   \FOR{$k \in [K]$}
   \STATE Sample $M_k \sim \bP(\mc{M}^\star \in \cdot \mid H_k)$
   \STATE Get optimal policy $\pi^{(k)} = \pi^\star_{M_k}$
   \STATE Execute $\pi^{(k)}$ and get trajectory $\tau_k$
   \STATE Update history $H_{k+1} = H_k \cup \tau_k$
   \STATE Induce posterior $\bP(\mc{M}^\star \in \cdot \mid H_{k+1})$
   \ENDFOR
\end{algorithmic}
\end{algorithm}
\end{minipage}
\hfill
\begin{minipage}{0.58\textwidth}
\begin{algorithm}[H]
   \caption{Value-equivalent Sampling for Reinforcement Learning (VSRL)}
   \label{alg:vsrl}
\begin{algorithmic}
   \STATE {\bfseries Input:} Prior $\bP(\mc{M}^\star \in \cdot \mid H_1)$, Threshold $D \in \bR_{\geq 0}$, Distortion function $d_{\Pi,\mc{V}}: \mathfrak{M} \times \mathfrak{M} \ra \bR_{\geq 0}$
   \FOR{$k \in [K]$}
   \STATE Compute $\widetilde{\mc{M}}_k$ achieving $\mc{R}^{\Pi,\mc{V}}_k(D)$ limit (Equation \ref{eq:rdf})
   \STATE Sample MDP $M^\star \sim \bP(\mc{M}^\star \in \cdot \mid H_k)$
   \STATE Sample compression $M_k \sim \bP(\widetilde{\mc{M}}_k \in \cdot \mid \mc{M}^\star = M^\star)$
   \STATE Compute optimal policy $\pi^{(k)} = \pi^\star_{M_k}$
   \STATE Execute $\pi^{(k)}$ and observe trajectory $\tau_k$
   \STATE Update history $H_{k+1} = H_k \cup \tau_k$
   \STATE Induce posterior $\bP(\mc{M}^\star \in \cdot \mid H_{k+1})$
   \ENDFOR
\end{algorithmic}
\end{algorithm}
\end{minipage}
\end{center}

A standard algorithm for our problem setting is widely known as Posterior Sampling for Reinforcement Learning (PSRL)~\citep{strens2000bayesian,osband2017posterior}, which we present as Algorithm \ref{alg:psrl}, while our Value-equivalent Sampling for Reinforcement Learning (VSRL) is given as Algorithm \ref{alg:vsrl}.  The key distinction between them is that, at each episode $k \in [K]$, the latter takes the posterior sample $M^\star \sim \bP(\mc{M}^\star \in \cdot \mid H_k)$ and passes it through the channel that achieves the rate-distortion limit (Equation \ref{eq:rdf}) at this episode to get the $M_k$ whose optimal policy is executed in the environment.

% A standard algorithm for our problem setting is widely known as Posterior Sampling for Reinforcement Learning (PSRL)~\citep{strens2000bayesian,osband2013more,agrawal2017optimistic}, which we present as Algorithm \ref{alg:psrl}. Briefly, in each episode, PSRL samples a single random MDP from the agent's current posterior $M_k \sim \bP(\mc{M}^\star \in \cdot \mid H_k)$. Again, recall that the randomness in $M_k$ is driven entirely by the model of the MDP such that, with a fully specified MDP $M_k$ in hand, any planning algorithm can be used to synthesize the optimal policy for $M_k$: $\pi^{(k)} \in \argmax\limits_{\pi \in \{\mc{S} \ra \Delta(\mc{A})\}^H} V^{\pi}_{M_k,1}$. The resulting trajectory from the episode induces the next history $H_{k+1} = H_k \cup \tau_k$ and yields updated posterior beliefs $\bP(\mc{M}^\star \in \cdot \mid H_{k+1})$ for the subsequent episode.

The core impetus for this work is to recognize that, for complex environments, pursuit of the exact MDP $\mc{M}^\star$ (as in PSRL) may be an entirely infeasible goal. Consider a MDP that represents control of a real-world, physical system; learning a transition function of the associated environment, at some level, demands that the agent internalize the laws of physics and motion with near-perfect accuracy. More formally, identifying $\mc{M}^\star$ demands the agent obtain exactly $\bH_1(\mc{M}^\star)$ bits of information from the environment which, under an uninformative prior, may either be prohibitively large by far exceeding the agent's capacity constraints or be simply impractical under time and resource constraints. 

As a remedy for this problem, we embrace the idea of being ``sufficiently satisfying'' or \textit{satisficing}~\citep{simon1982models,russo2017time,russo2018satisficing,russo2022satisficing,arumugam2021deciding,arumugam2021the}; as succinctly stated by Herbert A. Simon during his 1978 Nobel Memorial Lecture, ``decision makers can satisfice either by finding optimum solutions for a simplified world, or by finding satisfactory solutions for a more realistic world.'' Rather than spend an inordinate amount of time trying to recover an optimum solution to the true environment, we will instead design an algorithm that pursues optimum solutions for a sequence of simplified environments. In the next section, our analysis demonstrates that finding such optimum solutions for simplified worlds ultimately acts as a mechanism for achieving a satisfactory solution for the realistic, complex world. Naturally, the loss of fidelity between the simplified and true environments translates into a fixed amount of regret that an agent designer consciously and willingly accepts for two reasons: (1) they expect a reduction in the amount of time, data, and bits of information needed to identify the simplified environment and (2) in tasks where the environment encodes irrelevant information and exact knowledge is not needed to achieve optimal behavior~\citep{farahmand2017value,grimm2020value,grimm2021proper,voelcker2022value}, this worst-case error term may be negligible anyways while still maintaining greater efficiency than traditional PSRL.

% We present our generalization of PSRL, Value-equivalent Sampling for Reinforcement Learning (VSRL), as Algorithm \ref{alg:vsrl}. Over the course of learning, the agent maintains a belief over statistically-plausible MDPs $\bP(\mc{M}^\star \in \cdot \mid H_k)$ given its history of interaction $H_k$ thus far. At each episode $k \in [K]$, we draw one posterior sample $M^\star \sim \bP(\mc{M}^\star \in \cdot \mid H_k)$ and then pass $M^\star$ through the channel that achieves the rate-distortion limit (Equation \ref{eq:rdf}) at this episode to get $M_k \sim \bP(\widetilde{\mc{M}}_k \in \cdot \mid \mc{M}^\star = M^\star)$. VSRL then proceeds identically to PSRL, computing the optimal policy for $\widetilde{M}_k$: $\pi^{(k)} \in \argmax\limits_{\pi \in \{\mc{S} \ra \Delta(\mc{A})\}^H} \bE_{s_1 \sim \beta}\left[V^{\pi}_{\widetilde{M}_k,1}(s_1)\right]$ and using the resulting trajectory to compute updated posterior beliefs $\bP(\mc{M}^\star \in \cdot \mid H_{k+1})$. 

Recalling Example \ref{example:multi_res_mdps} that revolves around a particular sequence of MDPs, $\{\mc{M}_n\}_{n \in [N]}$, we note that as the distortion threshold $D$ increases, the significance of MDPs in the sequence indexed by larger values of $n \in [N]$ rapidly diminishes. As $D \uparrow \infty$, the lossy compression $\widetilde{\mc{M}}_k$ needn't convey information about any of the MDPs in $\{\mc{M}_n\}_{n \in [N]}$. Conversely, at $D = 0$, a VSRL agent must necessarily obtain enough information about the entire sequence so as to facilitate planning over $\Pi$ and $\mc{V}$. In between, however, the agent need only concern itself with a particular subsequence of $\{\mc{M}_n\}_{n \in [N]}$ while the remaining MDPs can be ignored due to their negligible contribution to overall value and, therefore, expected distortion under $d_{\Pi,\mc{V}}$.

\section{Regret Analysis}
\label{sec:analysis}

In this section, we offer an information-theoretic analysis of VSRL (Algorithm \ref{alg:vsrl}) before refining our regret bounds to the tabular setting. We conclude by highlighting how our performance guarantees can be expressed via a notion of agent capacity that is considerate of real-world reinforcement learning.

\subsection{An Information-Theoretic Bayesian Regret Bound}
\label{sec:info_regret_bounds}

To establish a Bayesian regret bound for VSRL we first require a regret decomposition that acknowledges the agent's new objective of identifying an approximately value-equivalent MDP in each episode, $\widetilde{\mc{M}}_k$, rather than the true MDP $\mc{M}^\star$. Crucially, this regret decomposition leverages the precise form of our distortion function $d_{\Pi,\mc{V}}(\mc{M}^\star,\widetilde{\mc{M}}_k)$.

\begin{theorem}
Take any $\Pi \supseteq \{\mc{S} \ra \mc{A}\}$, any $\mc{V} \supseteq \{V^\pi \mid \pi \in \Pi^H\}$, and fix any $D \geq 0$. For each episode $k \in [K]$, let $\widetilde{\mc{M}}_k$ be any MDP that achieves the rate-distortion limit of $\mc{R}^{\Pi,\mc{V}}_k(D)$ with information source $\bP(\mc{M}^\star \in \cdot \mid H_k)$ and distortion function $d_{\Pi,\mc{V}}$. Then, 
$\textsc{BayesRegret}(K, \pi^{(1)},\ldots,\pi^{(K)}) \leq \bE\left[\sum\limits_{k=1}^K \bE_k\left[V^{\star}_{\widetilde{\mc{M}}_k,1} - V^{\pi^{(k)}}_{\widetilde{\mc{M}}_k,1}\right]\right] + 2KH\sqrt{D}.$
% $$\textsc{BayesRegret}(K, \pi^{(1)},\ldots,\pi^{(K)}) \leq \bE\left[\sum\limits_{k=1}^K \bE\left[\sum\limits_{s_1 \in \mc{S}} \beta(s_1) \left(V^{\star}_{\widetilde{\mc{M}}_k,1}(s_1) - V^{\pi^{(k)}}_{\widetilde{\mc{M}}_k,1}(s_1) \right)\bigm| H_k \right] \Biggm| \mc{M}^\star \right] + 2KH\sqrt{D}.$$
\label{thm:regret_decomp}
\end{theorem}

Theorem \ref{thm:regret_decomp} shows how the Bayesian regret incurred by VSRL can be separated into an error term the agent must pay for learning a simplified MDP $\widetilde{\mc{M}}_k$, rather than $\mc{M}^\star$, and the Bayesian regret incurred while trying to learn $\widetilde{\mc{M}}_k$. This first term mirrors the satisficing regret of \citet{russo2018satisficing,russo2022satisficing} for multi-armed bandits where the performance of the agent in the $k$th episode is being measured with respect to a compressed MDP $\widetilde{\mc{M}}_k$, rather than the true MDP $\mc{M}^\star$. While further discussion on the choices of $\Pi$ and $\mc{V}$ is provided later in this section, we simply note that the conditions placed upon them in Theorem \ref{thm:regret_decomp} are an artifact of VSRL only executing optimal policies in each time period $h \in [H]$ which, under the assumptions of our problem formulation, are deterministic.

The remainder of this section is devoted to an analysis for establishing an information-theoretic bound on the satisficing regret term of Theorem \ref{thm:regret_decomp}. A central tool of our analysis will be the information ratio~\citep{russo2016information,russo2018learning} at the $k$th episode: $$\Gamma_k \triangleq \frac{\bE_k\left[V^{\star}_{\widetilde{\mc{M}}_k,1} - V^{\pi^{(k)}}_{\widetilde{\mc{M}}_k,1}\right]^2}{\bI_k(\widetilde{\mc{M}}_k; \tau_k, M_k)} \qquad \forall k \in [K].$$ In words, the information ratio is the ratio between squared expected regret in the $k$th episode with respect to $\widetilde{\mc{M}}_k$ and the information gained about $\widetilde{\mc{M}}_k$ in the $k$th episode by sampling MDP $M_k$ and observing trajectory $\tau_k$, given the current history $H_k$. Numerous prior works have leveraged similar or generalized types of information ratios for analyzing multi-armed bandit problems~\citep{russo2014learning,russo2016information,russo2018learning,russo2018satisficing,russo2022satisficing,dong2018information,lattimore2019information,zimmert2019connections,bubeck2020first,arumugam2021deciding,lattimore2021mirror} as well as reinforcement-learning problems~\citep{lu2019information}; in comparison to the latter, we simply note that our analysis bears stronger resemblance to those in multi-armed bandits by not constructing confidence sets over MDPs~\citep{osband2013more,osband2017posterior,lu2019information}, avoiding a restricted focus to tabular problems. That said, our results are contingent upon the existence of a uniform upper bound to the information ratios across all episodes, a non-trivial result~\citep{hao2022regret} that we leave to future work.

% A simple observation regarding the information gain denominator of $\Gamma_k$ is that the sampled $M_k$ by itself offers no information about $\widetilde{\mc{M}}_k$. Consequently, by the chain rule of mutual information, we have $$\bI_k(\widetilde{\mc{M}}_k; \tau_k, M_k) = \bI_k(\widetilde{\mc{M}}_k; M_k) + \bI_k(\widetilde{\mc{M}}_k; \tau_k \mid M_k) = \bI_k(\widetilde{\mc{M}}_k; \tau_k \mid M_k).$$ The following lemma demonstrates how the cumulative information gained about each $\widetilde{\mc{M}}_k$ across all $K$ episodes of learning can be controlled by the very first rate-distortion function using the prior over MDPs as the information source:
% \begin{lemma}
% For any $k \in [K]$, $\bE_k\left[\sum\limits_{k'=k}^K \bI_{k'}(\widetilde{\mc{M}}_{k'}; \tau_{k'} \mid M_{k'})\right] \leq \mc{R}^{\Pi,\mc{V}}_k(D).$
% \label{lemma:cum_info_bound}
% \end{lemma}

% This result gives rise to the following information-theoretic regret bound on satisficing Bayesian regret:
Through our information-ratio analysis, we obtain the following information-theoretic bound on satisficing Bayesian regret:

\begin{theorem}
If $\Gamma_k \leq \overline{\Gamma}$, for all $k \in [K]$, then $\bE\left[\sum\limits_{k=1}^K \bE_k\left[V^{\star}_{\widetilde{\mc{M}}_k,1} - V^{\pi^{(k)}}_{\widetilde{\mc{M}}_k,1}\right]\right] \leq \sqrt{\overline{\Gamma}K\mc{R}^{\Pi,\mc{V}}_1(D)}.$
\label{thm:info_sat_regret_bound}
\end{theorem}

An immediate consequence of the preceding theorems is the following corollary which establishes our main result, an information-theoretic Bayesian regret bound for VSRL. We omit the proof as it follows directly from applying Theorems \ref{thm:regret_decomp} and \ref{thm:info_sat_regret_bound} in sequence.

\begin{corollary}
Take any $\Pi \supseteq \{\mc{S} \ra \mc{A}\}$, any $\mc{V} \supseteq \{V^\pi \mid \pi \in \Pi^H\}$, and fix any $D > 0$. For any prior distribution $\bP(\mc{M}^\star \in \cdot \mid H_1)$, if $\Gamma_k \leq \overline{\Gamma}$ for all $k \in [K]$, then VSRL (Algorithm \ref{alg:vsrl}) has
% If $\Gamma_k \leq \overline{\Gamma}$ for all $k \in [K]$, then
$\textsc{BayesRegret}(K, \pi^{(1)},\ldots,\pi^{(K)}) \leq \sqrt{\overline{\Gamma}K\mc{R}^{\Pi,\mc{V}}_1(D)} + 2KH\sqrt{D}.$
% $$\textsc{BayesRegret}(K, \pi^{(1)},\ldots,\pi^{(K)}) \leq \sqrt{\overline{\Gamma}K\mc{R}^{\Pi,\mc{V}}_1(D)} + 2KH\sqrt{D}.$$
\label{thm:info_regret_bound}
\end{corollary}

Once again we recall that, since the rate-distortion function is well-defined for arbitrary source and channel output random variables defined on abstract alphabets~\citep{csiszar1974extremum}, the Bayesian regret bound of Corollary \ref{thm:info_regret_bound} holds for any finite-horizon, episodic MDP, extending beyond past analyses of PSRL constrained only to tabular MDPs. We defer a discussion of practical considerations for implementing VSRL to the appendix. 
% Moreover, concrete implementations of our generalization can follow suit with prior work~\citep{arumugam2021deciding,arumugam2021the} and leverage the Blahut-Arimoto algorithm~\citep{blahut1972computation,arimoto1972algorithm} to compute the channel achieving the rate-distortion limit.

% Recalling Example \ref{example:multi_res_mdps} that revolves around a particular sequence of MDPs, $\{\mc{M}_n\}_{n \in [N]}$, we note that as the distortion threshold $D$ increases, the significance of MDPs in the sequence indexed by larger values of $n \in [N]$ rapidly diminishes. At $D = \infty$, the lossy compression $\widetilde{\mc{M}}_k$ needn't convey information about any of the MDPs in $\{\mc{M}_n\}_{n \in [N]}$. Conversely, at $D = 0$, a VSRL agent must necessarily obtain enough information about the entire sequence so as to facilitate planning over $\Pi$ and $\mc{V}$. In between, however, the agent need only concern itself with a particular subsequence of $\{\mc{M}_n\}_{n \in [N]}$ while the remaining MDPs can be ignored due to their negligible contribution to overall value and, therefore, expected distortion under $d_{\Pi,\mc{V}}$.

At this point, we call attention to the parameterization of our lossy compression problem by a particular policy class $\Pi$ and value function class $\mc{V}$, whose dependence we inherit from the value equivalence principle~\citep{grimm2020value}. The next result clarifies how the performance of VSRL is affected by fluctuations in these classes via a dominance relationship~\citep{stjernvall1983dominance} between the induced distortion functions.

\begin{lemma}
For any two $\Pi,\Pi'$ and any $\mc{V}, \mc{V}'$ such that $\Pi' \subseteq \Pi \subseteq \{\mc{S} \ra \Delta(\mc{A})\}$ and $\mc{V}' \subseteq \mc{V} \subseteq \{\mc{S} \ra \bR\}$, we have $\mc{R}^{\Pi,\mc{V}}_k(D) \geq \mc{R}^{\Pi',\mc{V}'}_k(D)$, $\forall k \in [K], D > 0.$
\label{lemma:dominance}
\end{lemma}

Property 3 of \citet{grimm2020value} highlights how the set of value-equivalent MDPs grows as the policy and value function classes shrink. Lemma \ref{lemma:dominance} provides an intuitive, information-theoretic counterpart to their result where, as the sets of policies and value functions over which models will be assessed diminish, an agent may naturally compress more aggressively and throw away larger quantities of bits from each source distribution over the true MDP $\mc{M}^\star$. 

Since a compressed MDP $\widetilde{\mc{M}}_k$ that achieves the rate-distortion limit has \textit{expected} distortion bounded by $D$, one may wonder how the probability of not recovering an approximately-value-equivalent MDP scales as $D \uparrow \infty$. To that end, we conclude this section with a final result that brings clarity to this via a generalization~\citep{duchi2013distance} of Fano's inequality~\citep{fano1952TransInfoLectNotes}. We leave investigation of other generalizations of Fano's inequality that might yield similarly interesting results to future work~\citep{verdu1994generalizing,aeron2010information}.

\begin{lemma}
Take any $\Pi \subseteq \{\mc{S} \ra \Delta(\mc{A})\}$ and $\mc{V} \subseteq \{\mc{S} \ra \bR\}$. For any $D \geq 0$ and any $k \in [K]$, define $\delta = \sup\limits_{\widehat{M} \in \mathfrak{M}} \bP(d_{\Pi,\mc{V}}(\mc{M}^\star,\widehat{M}) \leq D \mid H_k).$ Then, 
\vspace{-5pt}
$$\sup\limits_{\widetilde{\mc{M}} \in \Lambda_k(D)} \bP(d_{\Pi,\mc{V}}(\mc{M}^\star, \widetilde{\mc{M}}) > D  \mid H_k) \geq 1 - \frac{\mc{R}^{\Pi,\mc{V}}_k(D) + \log(2)}{\log\left(\frac{1}{\delta}\right)}.$$
\label{lemma:fano_error_lowerbound}
\end{lemma}
\vspace{-5pt}

For any episode $k \in [K]$, the left-hand side of the inequality in Lemma \ref{lemma:fano_error_lowerbound} denotes the worst-case error probability of sampling a compressed MDP $\widetilde{\mc{M}}$ that is not approximately-value-equivalent to $\mc{M}^\star$. The right-hand side conveys that, in order to avoid such an error with reasonable probability, one requires a setting of $D < \infty$ such that $\mc{R}^{\Pi,\mc{V}}_k(D) \approx \log\left(\frac{1}{\delta}\right)$. 

\subsection{Specializing to Tabular MDPs}
\label{sec:tabular_regret_bounds}

While the preceding subsection constitutes the main contribution of this paper, the presence of information-theoretic terms makes it difficult to compare our guarantees to those obtained in prior work, which typically focuses on the tabular setting. To help remedy this, we offer the following theorem which restricts focus to the case where the agent pursues an exactly value-equivalent model of the tabular environment. Notably, the results of this section still retain a dependence on a uniform upper bound to the information ratio whose exact form is a result left to future work.

\begin{theorem}
Take any $\Pi \supseteq \{\mc{S} \ra \mc{A}\}$, any $\mc{V} \supseteq \{V^\pi \mid \pi \in \Pi^H\}$, and let $D = 0$. For any prior distribution $\bP(\mc{M}^\star \in \cdot \mid H_1)$ over tabular MDPs, if $\Gamma_k \leq \overline{\Gamma}$ for all $k \in [K]$, then VSRL (Algorithm \ref{alg:vsrl}) has
$\textsc{BayesRegret}(K, \pi^{(1)},\ldots,\pi^{(K)}) \leq \mc{O}\left(|\mc{S}|\sqrt{\overline{\Gamma}|\mc{A}|K}\right).$
\label{thm:tabular_regret}
\end{theorem}

An immediate observation is that the Bayesian regret bound of Theorem \ref{thm:tabular_regret} matches the dependence on the number of states, $|\mc{S}|$, obtained in the first (weaker) guarantee established for PSRL by \citet{osband2013more}; we suspect that this guarantee for VSRL is unimprovable without further distributional assumptions~\citep{osband2017posterior,osband2017gaussian}. As an alternative, we contemplate how a change in the distortion measure used by VSRL might incur an improved regret bound when specialized to the tabular setting. 

Specifically, notice that the only piece of the VSRL analysis tethered to the particular form of the distortion function $d_{\Pi,\mc{V}}(\mc{M}, \widehat{\mc{M}})$ is Theorem \ref{thm:regret_decomp}, while all other components remain agnostic to the precise criterion for assessing the loss of fidelity between original and compressed MDPs. Consequently, there is potential for a modified distortion function to offer an improved regret analysis relative to Theorem \ref{thm:tabular_regret}. Rather than concerning ourselves with planning over multiple behaviors, we consider a distortion function based solely on the optimal action-value functions: $$d_{Q^\star}(\mc{M}, \widehat{\mc{M}}) = \sup\limits_{h \in [H]} ||Q^\star_{\mc{M},h} - Q^\star_{\widehat{\mc{M}},h}||_\infty^2 = \sup\limits_{h \in [H]} \sup\limits_{(s,a) \in \mc{S} \times \mc{A}} | Q^\star_{\mc{M},h}(s,a) - Q^\star_{\widehat{\mc{M}},h}(s,a)|^2.$$ We use $\mc{R}^{Q^\star}_k(D)$ to denote the rate-distortion function under this new measure of distortion, $d_{Q^\star}(\mc{M}, \widehat{\mc{M}})$. In order for this new distortion function to be compatible with VSRL, we require an analogue to the regret decomposition of Theorem \ref{thm:regret_decomp}.

\begin{theorem}
Fix any $D \geq 0$ and, for each episode $k \in [K]$, let $\widetilde{\mc{M}}_k$ be any MDP that achieves the rate-distortion limit of $\mc{R}^{Q^\star}_k(D)$ with information source $\bP(\mc{M}^\star \in \cdot \mid H_k)$ and distortion function $d_{Q^\star}$. Then, 
$\textsc{BayesRegret}(K, \pi^{(1)},\ldots,\pi^{(K)}) \leq \bE\left[\sum\limits_{k=1}^K \bE_k\left[V^{\star}_{\widetilde{\mc{M}}_k,1} - V^{\pi^{(k)}}_{\widetilde{\mc{M}}_k,1}\right]\right] + 2K(H+1)\sqrt{D}.$
\label{thm:regret_decomp_qdist}
\end{theorem}

With this regret decomposition in hand, we may recover the analogue to Corollary \ref{thm:info_regret_bound}, whose proof is immediate and, therefore, omitted.

\begin{corollary}
Fix any $D > 0$. For any prior distribution $\bP(\mc{M}^\star \in \cdot \mid H_1)$, if $\Gamma_k \leq \overline{\Gamma}$ for all $k \in [K]$, then VSRL (Algorithm \ref{alg:vsrl}) with distortion function $d_{Q^\star}$ has
$\textsc{BayesRegret}(K, \pi^{(1)},\ldots,\pi^{(K)}) \leq \sqrt{\overline{\Gamma}K\mc{R}^{Q^\star}_1(D)} + 2K(H+1)\sqrt{D}.$
\label{thm:info_regret_bound_qdist}
\end{corollary}

As illustrated by the following lemma, the significance of this change in distortion measure from $d_{\Pi,\mc{V}}$ to $d_{Q^\star}$ is that the optimal action-value functions may now act as an information bottleneck~\citep{tishby2000information} between the original MDP $\mc{M}^\star$ and compressed MDP $\widetilde{\mc{M}}_k$. 
\begin{lemma}
For each episode $k \in [K]$ and for $D = 0$, let $\widetilde{\mc{M}}_k$ be any MDP that achieves the rate-distortion limit of $\mc{R}^{Q^\star}_k(D)$ with information source $\bP(\mc{M}^\star \in \cdot \mid H_k)$ and distortion function $d_{Q^\star}$. Then, we have the Markov chain $\mc{M}^\star \ra Q^\star_{\mc{M}^\star} \ra \widetilde{\mc{M}}_k$, where $Q^\star_{\mc{M}^\star} = \{Q^\star_{\mc{M}^\star,h}\}_{h \in [H]}$ is the collection of random variables denoting the optimal action-value functions of $\mc{M}^\star$.
\label{lemma:info_bottle_qdist}
\end{lemma}

Lemma \ref{lemma:info_bottle_qdist}, through the data-processing inequality, immediately leads us to an analogue of Theorem \ref{thm:tabular_regret} that matches the dependence on $|\mc{S}|$ in the best known Bayesian regret bound for PSRL~\citep{osband2017posterior}.
\begin{theorem}
For $D = 0$ and any prior distribution $\bP(\mc{M}^\star \in \cdot \mid H_1)$ over tabular MDPs, if $\Gamma_k \leq \overline{\Gamma}$ for all $k \in [K]$, then VSRL with distortion function $d_{Q^\star}$ has
$\textsc{BayesRegret}(K, \pi^{(1)},\ldots,\pi^{(K)}) \leq \widetilde{\mc{O}}\left(\sqrt{\overline{\Gamma}|\mc{S}||\mc{A}|KH}\right).$
\label{thm:tabular_regret_qdist}
\end{theorem}

Ultimately, Theorem \ref{thm:tabular_regret_qdist} confirms that while there is great flexibility in the original definition of value equivalence to support planning across multiple policies and value functions, focusing on optimal value functions gives rise to more efficient learning. Moreover, comparing the result with the PSRL regret bound of \citet{osband2017posterior} for tabular MDPs, this suggests an achievable uniform upper bound to the information ratio as $\overline{\Gamma} \lesssim H^2$, where the $\lesssim$ accounts for numerical constants and logarithmic factors.

\subsection{Capacity-Sensitive Performance Guarantees}
\label{sec:dist_rate_bounds}

We recognize that the information-theoretic regret bounds of the previous two sections, like many other guarantees for provably-efficient reinforcement learning before them, implicitly and unrealistically assume that an agent is of unbounded capacity and may pursue any approximately-value-equivalent model under a given distortion threshold $D$. In the context of real-world reinforcement learning~\citep{dulac2021challenges,lu2021reinforcement}, however, fundamental limits on computational resources and time leave an agent designer with a bounded agent to be deployed within an overwhelmingly complex environment. As such, this designer may seldom be in a position to dictate an ideal or desired sub-optimality threshold $D$, but rather must make do with a known constraint on agent capacity; guarantees on sample-efficient reinforcement learning cognizant of such a fundamental constraint are nascent. 

While there are numerous possibilities for how one might choose to formally characterize agent capacity, we here adopt a fundamental perspective that learning is the process of acquiring information and so take this capacity to imply the existence of a non-negative real value $R \in \bR_{> 0}$ such that the agent may only acquire and retain exactly $R$ bits of information. To help contextualize this notion of agent capacity, we introduce the distortion-rate function~\citep{shannon1959coding,berger1971rate,cover2012elements} which quantifies the fundamental limit of expected distortion under an information constraint:
\begin{align}
    \mc{D}^{Q^\star}_k(R) = \inf\limits_{\widetilde{\mc{M}} \in \Upsilon_k(R)} \bE_k\left[d_{Q^\star}(\mc{M}^\star, \widetilde{\mc{M}})\right] \qquad \mc{D}^{Q^\star}_k(R) = \inf\limits_{\widetilde{\mc{M}} \in \Upsilon_k(R)} \bE_k\left[d_{Q^\star}(\mc{M}^\star, \widetilde{\mc{M}})\right],
    \label{eq:drf}
\end{align}
% $$\mc{D}^{\Pi,\mc{V}}_k(R) = \inf\limits_{\substack{\bQ_{\widetilde{\mc{M}} \mid \mc{M}^\star} \in \mc{C}(\mathfrak{M} \mid \mathfrak{M}), \\ \bI_k(\mc{M}^\star; \widetilde{\mc{M}}) \leq R}} \bE_k\left[d_{\Pi,\mc{V}}(\mc{M}^\star, \widetilde{\mc{M}})\right] \qquad \mc{D}^{Q^\star}_k(R) = \inf\limits_{\substack{\bQ_{\widetilde{\mc{M}} \mid \mc{M}^\star} \in \mc{C}(\mathfrak{M} \mid \mathfrak{M}), \\ \bI_k(\mc{M}^\star; \widetilde{\mc{M}}) \leq R}} \bE_k\left[d_{Q^\star}(\mc{M}^\star, \widetilde{\mc{M}})\right],$$ 
where the infimum is taken over all channels with bounded rate, $\Upsilon_k(R) \triangleq \{\widetilde{\mc{M}}: \Omega \ra \mathfrak{M} \mid  \bI_k(\mc{M}^\star; \widetilde{\mc{M}}) \leq R\}$. In words, given the agent's current beliefs over the true MDP $\bP(\mc{M}^\star \in \cdot \mid H_k)$, the infimum of the distortion-rate function is taken over all potential lossy compressions of the environment that fall within the agent's capacity constraint of $R$ bits and identifies the one that preserves the most useful information, as measured by the distortion function. Conveniently, the rate-distortion function and distortion-rate function are inverses of one another~\citep{cover2012elements} $(\mc{R}(\mc{D}(R)) = R$) such that we recover the following two capacity-sensitive regret bounds directly from Corollaries \ref{thm:info_regret_bound} and \ref{thm:info_regret_bound_qdist} by simply taking the input distortion threshold of VSRL equal to the associated distortion-rate function in the first episode ($D = \mc{D}^{\Pi,\mc{V}}_1(R)$ and $D = \mc{D}^{Q^\star}_1(R)$, respectively).

\begin{corollary}
Take any $\Pi \supseteq \{\mc{S} \ra \mc{A}\}$, any $\mc{V} \supseteq \{V^\pi \mid \pi \in \Pi^H\}$, and let $R > 0$ be the agent capacity. For any prior distribution $\bP(\mc{M}^\star \in \cdot \mid H_1)$, if $\Gamma_k \leq \overline{\Gamma}$ for all $k \in [K]$, then VSRL (Algorithm \ref{alg:vsrl}) with distortion function $d_{\Pi,\mc{V}}$ has
$\textsc{BayesRegret}(K, \pi^{(1)},\ldots,\pi^{(K)}) \leq \sqrt{\overline{\Gamma}KR} + 2KH\sqrt{\mc{D}^{\Pi,\mc{V}}_1(R)}.$
\label{thm:info_regret_bound_capacity}
\end{corollary}

\begin{corollary}
Let $R > 0$ be the agent capacity. For any prior distribution $\bP(\mc{M}^\star \in \cdot \mid H_1)$, if $\Gamma_k \leq \overline{\Gamma}$ for all $k \in [K]$, then VSRL (Algorithm \ref{alg:vsrl}) with distortion function $d_{Q^\star}$ has
$\textsc{BayesRegret}(K, \pi^{(1)},\ldots,\pi^{(K)}) \leq \sqrt{\overline{\Gamma}KR} + 2K(H+1)\sqrt{\mc{D}^{Q^\star}_1(R)}.$
\label{thm:info_regret_bound_qdist_capacity}
\end{corollary}

Turning back to Example \ref{example:multi_res_mdps}, note how an agent with significantly limited capacity cannot possibly hope to capture all the granularity contained in the entire MDP sequence $\{\mc{M}_n\}_{n \in [N]}$, for large values of $N$. For a capacity of exactly $R$ bits, Corollaries \ref{thm:info_regret_bound_capacity} and \ref{thm:info_regret_bound_qdist_capacity} immediately translate this fundamental limit into a corresponding performance guarantee, allowing the agent to identify a subsequence $\{\mc{M}_n\}_{n \in [C]}$ for some $C \ll N$ which only requires gathering $R$ bits of information from the environment.

\section{Conclusion}
\label{sec:conc}

In this paper, we began with a finite-horizon, episodic MDP and considered the ramifications of a real-world reinforcement-learning scenario wherein the relative complexity of the environment is so immense that an agent may find itself incapable of perfectly recovering optimal behavior. An immediate consequence of this reality is the need to strike an appropriate balance between what is performant and what is achievable. We introduced the VSRL algorithm for incrementally synthesizing \textit{simple} and \textit{useful} approximations of the environment from which an agent might still recover near-optimal behaviors. Recognizing the information-theoretic nature of this lossy MDP compression, we provided an analysis of VSRL whose performance guarantees, by virtue of rate-distortion theory, are twofold. The first set of guarantees ensure VSRL recovers the simplest compression of the environment which still incurs bounded sub-optimality, as specified by the agent designer. Alternatively, the second set of guarantees maintain that VSRL finds the best compression of the environment subject to constraints on agent capacity. Through our general problem formulation and information-theoretic analysis, both regret bounds hold for any finite-horizon, episodic MDP, regardless of whether or not the state-action space is finite. That said, the question of how to practically instantiate VSRL for high-dimensional settings is an open problem left to future work.

\section*{Acknowledgements}

The authors gratefully acknowledge Christopher Grimm for initial discussions that provided an impetus for this work. The authors also thank Adithya Devraj, Shi Dong, John Duchi, Hong Jun Jeon, Saurabh Kumar, and Xiuyuan (Lucy) Lu for insightful comments on various components of the paper. Financial support from Army Research Office (ARO) grant W911NF2010055 is gratefully acknowledged.

% --- Bibliography ---
\bibliographystyle{plainnat}
\bibliography{references}

\begin{thebibliography}{170}
\providecommand{\natexlab}[1]{#1}
\providecommand{\url}[1]{\texttt{#1}}
\expandafter\ifx\csname urlstyle\endcsname\relax
  \providecommand{\doi}[1]{doi: #1}\else
  \providecommand{\doi}{doi: \begingroup \urlstyle{rm}\Url}\fi

\bibitem[Abachi et~al.(2020)Abachi, Ghavamzadeh, and
  Farahmand]{abachi2020policy}
Romina Abachi, Mohammad Ghavamzadeh, and Amir-massoud Farahmand.
\newblock Policy-aware model learning for policy gradient methods.
\newblock \emph{arXiv preprint arXiv:2003.00030}, 2020.

\bibitem[Abbasi-Yadkori and Szepesvari(2014)]{abbasi2014bayesian}
Yasin Abbasi-Yadkori and Csaba Szepesvari.
\newblock Bayesian optimal control of smoothly parameterized systems: The lazy
  posterior sampling algorithm.
\newblock \emph{arXiv preprint arXiv:1406.3926}, 2014.

\bibitem[Abel(2020)]{abel2020thesis}
David Abel.
\newblock \emph{A Theory of Abstraction in Reinforcement Learning}.
\newblock PhD thesis, Brown University, 2020.

\bibitem[Abel et~al.(2016)Abel, Hershkowitz, and Littman]{abel2016near}
David Abel, David Hershkowitz, and Michael Littman.
\newblock Near optimal behavior via approximate state abstraction.
\newblock In \emph{International Conference on Machine Learning}, pages
  2915--2923. PMLR, 2016.

\bibitem[Abel et~al.(2018)Abel, Arumugam, Lehnert, and Littman]{abel2018state}
David Abel, Dilip Arumugam, Lucas Lehnert, and Michael Littman.
\newblock State abstractions for lifelong reinforcement learning.
\newblock In \emph{International Conference on Machine Learning}, pages 10--19,
  2018.

\bibitem[Abel et~al.(2019)Abel, Arumugam, Asadi, Jinnai, Littman, and
  Wong]{abel2019state}
David Abel, Dilip Arumugam, Kavosh Asadi, Yuu Jinnai, Michael~L Littman, and
  Lawson~LS Wong.
\newblock State abstraction as compression in apprenticeship learning.
\newblock In \emph{Proceedings of the AAAI Conference on Artificial
  Intelligence}, volume~33, pages 3134--3142, 2019.

\bibitem[Abel et~al.(2020)Abel, Umbanhowar, Khetarpal, Arumugam, Precup, and
  Littman]{abel2020value}
David Abel, Nate Umbanhowar, Khimya Khetarpal, Dilip Arumugam, Doina Precup,
  and Michael Littman.
\newblock Value preserving state-action abstractions.
\newblock In \emph{International Conference on Artificial Intelligence and
  Statistics}, pages 1639--1650. PMLR, 2020.

\bibitem[Aeron et~al.(2010)Aeron, Saligrama, and Zhao]{aeron2010information}
Shuchin Aeron, Venkatesh Saligrama, and Manqi Zhao.
\newblock Information theoretic bounds for compressed sensing.
\newblock \emph{IEEE Transactions on Information Theory}, 56\penalty0
  (10):\penalty0 5111--5130, 2010.

\bibitem[Agarwal et~al.(2020)Agarwal, Kakade, Krishnamurthy, and
  Sun]{agarwal2020flambe}
Alekh Agarwal, Sham Kakade, Akshay Krishnamurthy, and Wen Sun.
\newblock {FLAMBE}: Structural complexity and representation learning of low
  rank {MDP}s.
\newblock In H.~Larochelle, M.~Ranzato, R.~Hadsell, M.~F. Balcan, and H.~Lin,
  editors, \emph{Advances in Neural Information Processing Systems}, volume~33,
  pages 20095--20107. Curran Associates, Inc., 2020.

\bibitem[Agrawal and Jia(2017)]{agrawal2017optimistic}
Shipra Agrawal and Randy Jia.
\newblock Optimistic posterior sampling for reinforcement learning: worst-case
  regret bounds.
\newblock In \emph{Advances in Neural Information Processing Systems}, pages
  1184--1194, 2017.

\bibitem[Alemi et~al.(2018)Alemi, Poole, Fischer, Dillon, Saurous, and
  Murphy]{alemi2018fixing}
Alexander Alemi, Ben Poole, Ian Fischer, Joshua Dillon, Rif~A Saurous, and
  Kevin Murphy.
\newblock Fixing a broken {ELBO}.
\newblock In \emph{International Conference on Machine Learning}, pages
  159--168. PMLR, 2018.

\bibitem[Araya-L{\'o}pez et~al.(2012)Araya-L{\'o}pez, Thomas, and
  Buffet]{araya2012near}
Mauricio Araya-L{\'o}pez, Vincent Thomas, and Olivier Buffet.
\newblock Near-optimal {BRL} using optimistic local transitions.
\newblock In \emph{Proceedings of the 29th International Conference on Machine
  Learning}, pages 515--522, 2012.

\bibitem[Arimoto(1972)]{arimoto1972algorithm}
Suguru Arimoto.
\newblock An algorithm for computing the capacity of arbitrary discrete
  memoryless channels.
\newblock \emph{IEEE Transactions on Information Theory}, 18\penalty0
  (1):\penalty0 14--20, 1972.

\bibitem[Arumugam and Van~Roy(2020)]{arumugam2020randomized}
Dilip Arumugam and Benjamin Van~Roy.
\newblock Randomized value functions via posterior state-abstraction sampling.
\newblock \emph{arXiv preprint arXiv:2010.02383}, 2020.

\bibitem[Arumugam and Van~Roy(2021{\natexlab{a}})]{arumugam2021deciding}
Dilip Arumugam and Benjamin Van~Roy.
\newblock Deciding what to learn: {A} rate-distortion approach.
\newblock In \emph{International Conference on Machine Learning}, pages
  373--382. PMLR, 2021{\natexlab{a}}.

\bibitem[Arumugam and Van~Roy(2021{\natexlab{b}})]{arumugam2021the}
Dilip Arumugam and Benjamin Van~Roy.
\newblock The value of information when deciding what to learn.
\newblock \emph{Advances in Neural Information Processing Systems}, 34,
  2021{\natexlab{b}}.

\bibitem[Asadi et~al.(2018)Asadi, Misra, and Littman]{asadi2018lipschitz}
Kavosh Asadi, Dipendra Misra, and Michael Littman.
\newblock Lipschitz continuity in model-based reinforcement learning.
\newblock In \emph{International Conference on Machine Learning}, pages
  264--273. PMLR, 2018.

\bibitem[Asmuth et~al.(2009)Asmuth, Li, Littman, Nouri, and
  Wingate]{asmuth2009bayesian}
John Asmuth, Lihong Li, Michael~L Littman, Ali Nouri, and David Wingate.
\newblock A {B}ayesian sampling approach to exploration in reinforcement
  learning.
\newblock In \emph{Proceedings of the Twenty-Fifth Conference on Uncertainty in
  Artificial Intelligence}, pages 19--26, 2009.

\bibitem[{\AA}str{\"o}m(1965)]{aastrom1965optimal}
Karl~Johan {\AA}str{\"o}m.
\newblock Optimal control of {M}arkov processes with incomplete state
  information.
\newblock \emph{Journal of Mathematical Analysis and Applications}, 10\penalty0
  (1):\penalty0 174--205, 1965.

\bibitem[Auer et~al.(2009)Auer, Jaksch, and Ortner]{auer2009near}
Peter Auer, Thomas Jaksch, and Ronald Ortner.
\newblock Near-optimal regret bounds for reinforcement learning.
\newblock In \emph{Advances in Neural Information Processing Systems}, pages
  89--96, 2009.

\bibitem[Ayoub et~al.(2020)Ayoub, Jia, Szepesvari, Wang, and
  Yang]{ayoub2020model}
Alex Ayoub, Zeyu Jia, Csaba Szepesvari, Mengdi Wang, and Lin Yang.
\newblock Model-based reinforcement learning with value-targeted regression.
\newblock In \emph{International Conference on Machine Learning}, pages
  463--474. PMLR, 2020.

\bibitem[Azar et~al.(2017)Azar, Osband, and Munos]{azar2017minimax}
Mohammad~Gheshlaghi Azar, Ian Osband, and R{\'e}mi Munos.
\newblock Minimax regret bounds for reinforcement learning.
\newblock In \emph{International Conference on Machine Learning}, pages
  263--272. PMLR, 2017.

\bibitem[Bartlett and Tewari(2009)]{bartlett2009regal}
Peter~L Bartlett and Ambuj Tewari.
\newblock {REGAL}: a regularization based algorithm for reinforcement learning
  in weakly communicating {MDP}s.
\newblock In \emph{Proceedings of the Twenty-Fifth Conference on Uncertainty in
  Artificial Intelligence}, pages 35--42, 2009.

\bibitem[Bellman(1957)]{bellman1957markovian}
Richard Bellman.
\newblock A {M}arkovian decision process.
\newblock \emph{Journal of Mathematics and Mechanics}, pages 679--684, 1957.

\bibitem[Berger(1971)]{berger1971rate}
Toby Berger.
\newblock \emph{{Rate Distortion Theory: A Mathematical Basis for Data
  Compression}}.
\newblock Prentice-Hall, 1971.

\bibitem[Berger and Gibson(1998)]{berger1998lossy}
Toby Berger and Jerry~D Gibson.
\newblock Lossy source coding.
\newblock \emph{IEEE Transactions on Information Theory}, 44\penalty0
  (6):\penalty0 2693--2723, 1998.

\bibitem[Berry et~al.(1997)Berry, Chen, Zame, Heath, and Shepp]{berry1997}
Donald~A. Berry, Robert~W. Chen, Alan Zame, David~C. Heath, and Larry~A. Shepp.
\newblock Bandit problems with infinitely many arms.
\newblock \emph{Ann. Statist.}, 25\penalty0 (5):\penalty0 2103--2116, 10 1997.

\bibitem[Bertsekas and Shreve(1996)]{bertsekas1996stochastic}
Dimitri Bertsekas and Steven~E Shreve.
\newblock \emph{Stochastic optimal control: the discrete-time case}, volume~5.
\newblock Athena Scientific, 1996.

\bibitem[Bertsekas(1995)]{bertsekas1995dynamic}
Dimitri~P. Bertsekas.
\newblock \emph{Dynamic Programming and Optimal Control}.
\newblock Athena Scientific, 1995.

\bibitem[Bertsekas and Casta{\~n}on(1989)]{bertsekas1989adaptive}
Dimitri~P. Bertsekas and David~A. Casta{\~n}on.
\newblock Adaptive aggregation methods for infinite horizon dynamic
  programming.
\newblock \emph{IEEE Transactions on Automatic Control}, 34\penalty0
  (6):\penalty0 589--598, 1989.

\bibitem[Blahut(1972)]{blahut1972computation}
Richard Blahut.
\newblock Computation of channel capacity and rate-distortion functions.
\newblock \emph{IEEE Transactions on Information Theory}, 18\penalty0
  (4):\penalty0 460--473, 1972.

\bibitem[Bonald and Proutiere(2013)]{bonald2013two}
Thomas Bonald and Alexandre Proutiere.
\newblock Two-target algorithms for infinite-armed bandits with {B}ernoulli
  rewards.
\newblock In \emph{Advances in Neural Information Processing Systems}, pages
  2184--2192, 2013.

\bibitem[Borkar et~al.(2001)Borkar, Mitter, and Tatikonda]{borkar2001markov}
Vivek~S Borkar, Sanjoy Mitter, and Sekhar Tatikonda.
\newblock Markov control problems under communication contraints.
\newblock \emph{Communications in Information and Systems}, 1\penalty0
  (1):\penalty0 15--32, 2001.

\bibitem[Boukris(1973)]{boukris1973upper}
Pinhas Boukris.
\newblock An upper bound on the speed of convergence of the {B}lahut algorithm
  for computing rate-distortion functions (corresp.).
\newblock \emph{IEEE Transactions on Information Theory}, 19\penalty0
  (5):\penalty0 708--709, 1973.

\bibitem[Brafman and Tennenholtz(2002)]{brafman2002r}
Ronen~I Brafman and Moshe Tennenholtz.
\newblock {R-MAX}-a general polynomial time algorithm for near-optimal
  reinforcement learning.
\newblock \emph{Journal of Machine Learning Research}, 3\penalty0
  (Oct):\penalty0 213--231, 2002.

\bibitem[Bubeck and Sellke(2020)]{bubeck2020first}
S{\'e}bastien Bubeck and Mark Sellke.
\newblock First-order {B}ayesian regret analysis of {T}hompson sampling.
\newblock In \emph{Algorithmic Learning Theory}, pages 196--233. PMLR, 2020.

\bibitem[Bubeck et~al.(2011)Bubeck, Munos, Stoltz, and
  Szepesv{\'a}ri]{bubeck2011x}
S{\'e}bastien Bubeck, R{\'e}mi Munos, Gilles Stoltz, and Csaba Szepesv{\'a}ri.
\newblock X-armed bandits.
\newblock \emph{Journal of Machine Learning Research}, 12\penalty0 (5), 2011.

\bibitem[Bubeck et~al.(2012)Bubeck, Cesa-Bianchi, et~al.]{bubeck2012regret}
S{\'e}bastien Bubeck, Nicol{\`o} Cesa-Bianchi, et~al.
\newblock Regret analysis of stochastic and nonstochastic multi-armed bandit
  problems.
\newblock \emph{Foundations and Trends{\textregistered} in Machine Learning},
  5\penalty0 (1):\penalty0 1--122, 2012.

\bibitem[Castro and Precup(2007)]{castro2007using}
Pablo~Samuel Castro and Doina Precup.
\newblock Using linear programming for {B}ayesian exploration in {M}arkov
  decision processes.
\newblock In \emph{IJCAI}, volume 24372442, 2007.

\bibitem[Chiang and Boyd(2004)]{chiang2004geometric}
Mung Chiang and Stephen Boyd.
\newblock Geometric programming duals of channel capacity and rate distortion.
\newblock \emph{IEEE Transactions on Information Theory}, 50\penalty0
  (2):\penalty0 245--258, 2004.

\bibitem[Cover and Thomas(2012)]{cover2012elements}
Thomas~M Cover and Joy~A Thomas.
\newblock \emph{Elements of {I}nformation {T}heory}.
\newblock John Wiley \& Sons, 2012.

\bibitem[Csisz{\'a}r(1974{\natexlab{a}})]{csiszar1974computation}
Imre Csisz{\'a}r.
\newblock On the computation of rate-distortion functions (corresp.).
\newblock \emph{IEEE Transactions on Information Theory}, 20\penalty0
  (1):\penalty0 122--124, 1974{\natexlab{a}}.

\bibitem[Csisz{\'a}r(1974{\natexlab{b}})]{csiszar1974extremum}
Imre Csisz{\'a}r.
\newblock On an extremum problem of information theory.
\newblock \emph{Studia Scientiarum Mathematicarum Hungarica}, 9,
  1974{\natexlab{b}}.

\bibitem[Cui et~al.(2020)Cui, Chow, and Ghavamzadeh]{cui2020control}
Brandon Cui, Yinlam Chow, and Mohammad Ghavamzadeh.
\newblock Control-aware representations for model-based reinforcement learning.
\newblock In \emph{International Conference on Learning Representations}, 2020.

\bibitem[Dann and Brunskill(2015)]{dann2015sample}
Christoph Dann and Emma Brunskill.
\newblock Sample complexity of episodic fixed-horizon reinforcement learning.
\newblock In \emph{Proceedings of the 28th International Conference on Neural
  Information Processing Systems-Volume 2}, pages 2818--2826, 2015.

\bibitem[Dann et~al.(2017)Dann, Lattimore, and Brunskill]{dann2017unifying}
Christoph Dann, Tor Lattimore, and Emma Brunskill.
\newblock Unifying {PAC} and regret: uniform {PAC} bounds for episodic
  reinforcement learning.
\newblock In \emph{Proceedings of the 31st International Conference on Neural
  Information Processing Systems}, pages 5717--5727, 2017.

\bibitem[Dann et~al.(2018)Dann, Jiang, Krishnamurthy, Agarwal, Langford, and
  Schapire]{dann2018oracle}
Christoph Dann, Nan Jiang, Akshay Krishnamurthy, Alekh Agarwal, John Langford,
  and Robert~E Schapire.
\newblock On oracle-efficient {PAC} {RL} with rich observations.
\newblock In \emph{Proceedings of the 32nd International Conference on Neural
  Information Processing Systems}, pages 1429--1439, 2018.

\bibitem[Dauwels(2005)]{dauwels2005numerical}
Justin Dauwels.
\newblock Numerical computation of the capacity of continuous memoryless
  channels.
\newblock In \emph{Proceedings of the 26th Symposium on Information Theory in
  the BENELUX}, pages 221--228. Citeseer, 2005.

\bibitem[Dayan(1993)]{dayan1993improving}
Peter Dayan.
\newblock Improving generalization for temporal difference learning: {T}he
  successor representation.
\newblock \emph{Neural computation}, 5\penalty0 (4):\penalty0 613--624, 1993.

\bibitem[Dean and Givan(1997)]{dean1997model}
Thomas Dean and Robert Givan.
\newblock Model minimization in {M}arkov decision processes.
\newblock In \emph{Proceedings of the AAAI Conference on Artificial
  Intelligence}, pages 106--111. AAAI Press, 1997.

\bibitem[Dearden et~al.(1998)Dearden, Friedman, and
  Russell]{dearden1998bayesian}
Richard Dearden, Nir Friedman, and Stuart Russell.
\newblock {B}ayesian {Q}-learning.
\newblock In \emph{Proceedings of the Fifteenth National/Tenth Conference on
  Artificial Intelligence/Innovative Applications of Artificial Intelligence},
  pages 761--768, 1998.

\bibitem[Dearden et~al.(1999)Dearden, Friedman, and Andre]{dearden1999model}
Richard Dearden, Nir Friedman, and David Andre.
\newblock Model based {B}ayesian exploration.
\newblock In \emph{Proceedings of the Fifteenth Conference on Uncertainty in
  Artificial Intelligence}, pages 150--159, 1999.

\bibitem[Deshpande and Montanari(2012)]{deshpande2012linear}
Yash Deshpande and Andrea Montanari.
\newblock Linear bandits in high dimension and recommendation systems.
\newblock In \emph{2012 50th Annual Allerton Conference on Communication,
  Control, and Computing (Allerton)}, pages 1750--1754. IEEE, 2012.

\bibitem[Dong and Van~Roy(2018)]{dong2018information}
Shi Dong and Benjamin Van~Roy.
\newblock An information-theoretic analysis for {T}hompson sampling with many
  actions.
\newblock In \emph{Advances in Neural Information Processing Systems}, pages
  4157--4165, 2018.

\bibitem[Dong et~al.(2019)Dong, Van~Roy, and Zhou]{dong2019provably}
Shi Dong, Benjamin Van~Roy, and Zhengyuan Zhou.
\newblock Provably efficient reinforcement learning with aggregated states.
\newblock \emph{arXiv preprint arXiv:1912.06366}, 2019.

\bibitem[Dong et~al.(2021)Dong, Van~Roy, and Zhou]{dong2021simple}
Shi Dong, Benjamin Van~Roy, and Zhengyuan Zhou.
\newblock Simple agent, complex environment: Efficient reinforcement learning
  with agent state.
\newblock \emph{arXiv preprint arXiv:2102.05261}, 2021.

\bibitem[D'Oro et~al.(2020)D'Oro, Metelli, Tirinzoni, Papini, and
  Restelli]{d2020gradient}
Pierluca D'Oro, Alberto~Maria Metelli, Andrea Tirinzoni, Matteo Papini, and
  Marcello Restelli.
\newblock Gradient-aware model-based policy search.
\newblock In \emph{Proceedings of the AAAI Conference on Artificial
  Intelligence}, volume~34, pages 3801--3808, 2020.

\bibitem[Du et~al.(2019)Du, Krishnamurthy, Jiang, Agarwal, Dudik, and
  Langford]{du2019provably}
Simon Du, Akshay Krishnamurthy, Nan Jiang, Alekh Agarwal, Miroslav Dudik, and
  John Langford.
\newblock Provably efficient {RL} with rich observations via latent state
  decoding.
\newblock In \emph{International Conference on Machine Learning}, pages
  1665--1674. PMLR, 2019.

\bibitem[Duchi(2021)]{duchi21ItLectNotes}
John~C. Duchi.
\newblock \emph{Lecture Notes for {S}tatistics 311/{E}lectrical {E}ngineering
  377, {S}tanford University.}
\newblock 2021.

\bibitem[Duchi and Wainwright(2013)]{duchi2013distance}
John~C. Duchi and Martin~J. Wainwright.
\newblock Distance-based and continuum {F}ano inequalities with applications to
  statistical estimation.
\newblock \emph{arXiv preprint arXiv:1311.2669}, 2013.

\bibitem[Duff(2002)]{duff2002optimal}
Michael~O'Gordon Duff.
\newblock \emph{Optimal Learning: Computational procedures for {B}ayes-adaptive
  {M}arkov decision processes}.
\newblock University of Massachusetts Amherst, 2002.

\bibitem[Dulac-Arnold et~al.(2021)Dulac-Arnold, Levine, Mankowitz, Li,
  Paduraru, Gowal, and Hester]{dulac2021challenges}
Gabriel Dulac-Arnold, Nir Levine, Daniel~J Mankowitz, Jerry Li, Cosmin
  Paduraru, Sven Gowal, and Todd Hester.
\newblock Challenges of real-world reinforcement learning: definitions,
  benchmarks and analysis.
\newblock \emph{Machine Learning}, pages 1--50, 2021.

\bibitem[Dwaracherla et~al.(2020)Dwaracherla, Lu, Ibrahimi, Osband, Wen, and
  Van~Roy]{dwaracherla2020hypermodels}
Vikranth Dwaracherla, Xiuyuan Lu, Morteza Ibrahimi, Ian Osband, Zheng Wen, and
  Benjamin Van~Roy.
\newblock Hypermodels for exploration.
\newblock In \emph{International Conference on Learning Representations}, 2020.

\bibitem[Fano(1952)]{fano1952TransInfoLectNotes}
Robert~M. Fano.
\newblock \emph{Class Notes for {MIT} Course 6.574: Transmission of
  Information, MIT, Cambridge, MA.}
\newblock 1952.

\bibitem[Farahmand(2018)]{farahmand2018iterative}
Amir-massoud Farahmand.
\newblock Iterative value-aware model learning.
\newblock In \emph{Proceedings of the 32nd International Conference on Neural
  Information Processing Systems}, pages 9090--9101, 2018.

\bibitem[Farahmand et~al.(2017)Farahmand, Barreto, and
  Nikovski]{farahmand2017value}
Amir-massoud Farahmand, Andre Barreto, and Daniel Nikovski.
\newblock Value-aware loss function for model-based reinforcement learning.
\newblock In \emph{Artificial Intelligence and Statistics}, pages 1486--1494.
  PMLR, 2017.

\bibitem[Ferns et~al.(2004)Ferns, Panangaden, and Precup]{ferns2004metrics}
Norm Ferns, Prakash Panangaden, and Doina Precup.
\newblock Metrics for finite {M}arkov {D}ecision {P}rocesses.
\newblock In \emph{Proceedings of the 20th Conference on Uncertainty in
  Artificial Intelligence}, pages 162--169, 2004.

\bibitem[Ferns et~al.(2012)Ferns, Castro, Precup, and
  Panangaden]{ferns2012methods}
Norman Ferns, Pablo~Samuel Castro, Doina Precup, and Prakash Panangaden.
\newblock Methods for computing state similarity in {M}arkov {D}ecision
  {P}rocesses.
\newblock \emph{arXiv preprint arXiv:1206.6836}, 2012.

\bibitem[Flennerhag et~al.(2020)Flennerhag, Wang, Sprechmann, Visin, Galashov,
  Kapturowski, Borsa, Heess, Barreto, and Pascanu]{flennerhag2020temporal}
Sebastian Flennerhag, Jane~X Wang, Pablo Sprechmann, Francesco Visin, Alexandre
  Galashov, Steven Kapturowski, Diana~L Borsa, Nicolas Heess, Andre Barreto,
  and Razvan Pascanu.
\newblock Temporal difference uncertainties as a signal for exploration.
\newblock \emph{arXiv preprint arXiv:2010.02255}, 2020.

\bibitem[Ghavamzadeh et~al.(2015)Ghavamzadeh, Mannor, Pineau, and
  Tamar]{ghavamzadeh2015bayesian}
Mohammad Ghavamzadeh, Shie Mannor, Joelle Pineau, and Aviv Tamar.
\newblock {B}ayesian reinforcement learning: A survey.
\newblock \emph{Foundations and Trends{\textregistered} in Machine Learning},
  8\penalty0 (5-6):\penalty0 359--483, 2015.

\bibitem[Gray(2011)]{gray2011entropy}
Robert~M. Gray.
\newblock \emph{Entropy and information theory}.
\newblock Springer Science \& Business Media, 2011.

\bibitem[Grimm et~al.(2020)Grimm, Barreto, Singh, and Silver]{grimm2020value}
Christopher Grimm, Andre Barreto, Satinder Singh, and David Silver.
\newblock The value equivalence principle for model-based reinforcement
  learning.
\newblock \emph{Advances in Neural Information Processing Systems}, 33, 2020.

\bibitem[Grimm et~al.(2021)Grimm, Barreto, Farquhar, Silver, and
  Singh]{grimm2021proper}
Christopher Grimm, Andr{\'e} Barreto, Greg Farquhar, David Silver, and Satinder
  Singh.
\newblock Proper value equivalence.
\newblock \emph{Advances in Neural Information Processing Systems}, 34, 2021.

\bibitem[Guez et~al.(2012)Guez, Silver, and Dayan]{guez2012efficient}
Arthur Guez, David Silver, and Peter Dayan.
\newblock Efficient {B}ayes-adaptive reinforcement learning using sample-based
  search.
\newblock In \emph{Proceedings of the 25th International Conference on Neural
  Information Processing Systems-Volume 1}, pages 1025--1033, 2012.

\bibitem[Guez et~al.(2013)Guez, Silver, and Dayan]{guez2013scalable}
Arthur Guez, David Silver, and Peter Dayan.
\newblock Scalable and efficient {B}ayes-adaptive reinforcement learning based
  on {M}onte-{C}arlo tree search.
\newblock \emph{Journal of Artificial Intelligence Research}, 48:\penalty0
  841--883, 2013.

\bibitem[Guez et~al.(2014)Guez, Heess, Silver, and Dayan]{guez2014bayes}
Arthur Guez, Nicolas Heess, David Silver, and Peter Dayan.
\newblock {B}ayes-adaptive simulation-based search with value function
  approximation.
\newblock In \emph{Advances in Neural Information Processing Systems}, pages
  451--459, 2014.

\bibitem[Hallak et~al.(2015)Hallak, Di~Castro, and
  Mannor]{hallak2015contextual}
Assaf Hallak, Dotan Di~Castro, and Shie Mannor.
\newblock Contextual {M}arkov decision processes.
\newblock \emph{arXiv preprint arXiv:1502.02259}, 2015.

\bibitem[Hao and Lattimore(2022)]{hao2022regret}
Botao Hao and Tor Lattimore.
\newblock Regret bounds for information-directed reinforcement learning.
\newblock \emph{arXiv preprint arXiv:2206.04640}, 2022.

\bibitem[Harrison and Kontoyiannis(2008)]{harrison2008estimation}
Matthew~T Harrison and Ioannis Kontoyiannis.
\newblock Estimation of the rate--distortion function.
\newblock \emph{IEEE Transactions on Information Theory}, 54\penalty0
  (8):\penalty0 3757--3762, 2008.

\bibitem[Ionescu-Tulcea(1949)]{tulcea1949mesures}
Cassius~T. Ionescu-Tulcea.
\newblock Mesures dans les espaces produits.
\newblock \emph{Atti Acad. Naz. Lincei Rend. Cl Sci. Fis. Mat. Nat}, 8\penalty0
  (7), 1949.

\bibitem[Jaksch et~al.(2010)Jaksch, Ortner, and Auer]{jaksch2010near}
Thomas Jaksch, Ronald Ortner, and Peter Auer.
\newblock Near-optimal regret bounds for reinforcement learning.
\newblock \emph{Journal of Machine Learning Research}, 11\penalty0 (4), 2010.

\bibitem[Janz et~al.(2019)Janz, Hron, Mazur, Hofmann, Hern{\'a}ndez-Lobato, and
  Tschiatschek]{janz2019successor}
David Janz, Jiri Hron, Przemys{\l}aw Mazur, Katja Hofmann, Jos{\'e}~Miguel
  Hern{\'a}ndez-Lobato, and Sebastian Tschiatschek.
\newblock Successor uncertainties: exploration and uncertainty in temporal
  difference learning.
\newblock \emph{Advances in Neural Information Processing Systems}, 32, 2019.

\bibitem[Jiang et~al.(2015)Jiang, Kulesza, and Singh]{jiang2015abstraction}
Nan Jiang, Alex Kulesza, and Satinder Singh.
\newblock Abstraction selection in model-based reinforcement learning.
\newblock In \emph{International Conference on Machine Learning}, pages
  179--188, 2015.

\bibitem[Jiang et~al.(2017)Jiang, Krishnamurthy, Agarwal, Langford, and
  Schapire]{jiang2017contextual}
Nan Jiang, Akshay Krishnamurthy, Alekh Agarwal, John Langford, and Robert~E
  Schapire.
\newblock Contextual decision processes with low {B}ellman rank are
  {PAC}-learnable.
\newblock In \emph{International Conference on Machine Learning}, pages
  1704--1713. PMLR, 2017.

\bibitem[Jin et~al.(2018)Jin, Allen-Zhu, Bubeck, and Jordan]{jin2018q}
Chi Jin, Zeyuan Allen-Zhu, Sebastien Bubeck, and Michael~I Jordan.
\newblock Is {$Q$}-learning provably efficient?
\newblock In \emph{Proceedings of the 32nd International Conference on Neural
  Information Processing Systems}, pages 4868--4878, 2018.

\bibitem[Jong and Stone(2005)]{jong2005state}
Nicholas~K Jong and Peter Stone.
\newblock State abstraction discovery from irrelevant state variables.
\newblock In \emph{Proceedings of the 19th International Joint Conference on
  Artificial Intelligence}, pages 752--757, 2005.

\bibitem[Kaelbling et~al.(1996)Kaelbling, Littman, and
  Moore]{kaelbling1996reinforcement}
Leslie~Pack Kaelbling, Michael~L Littman, and Andrew~W Moore.
\newblock Reinforcement learning: A survey.
\newblock \emph{Journal of Artificial Intelligence Research}, 4:\penalty0
  237--285, 1996.

\bibitem[Kaelbling et~al.(1998)Kaelbling, Littman, and
  Cassandra]{kaelbling1998planning}
Leslie~Pack Kaelbling, Michael~L Littman, and Anthony~R Cassandra.
\newblock Planning and acting in partially observable stochastic domains.
\newblock \emph{Artificial intelligence}, 101\penalty0 (1-2):\penalty0 99--134,
  1998.

\bibitem[Kakade and Langford(2002)]{kakade2002approximately}
Sham Kakade and John Langford.
\newblock Approximately optimal approximate reinforcement learning.
\newblock In \emph{Proceedings of the Nineteenth International Conference on
  Machine Learning}, pages 267--274, 2002.

\bibitem[Kakade(2003)]{kakade2003sample}
Sham~Machandranath Kakade.
\newblock \emph{On the {S}ample {C}omplexity of {R}einforcement {L}earning}.
\newblock PhD thesis, Gatsby Computational Neuroscience Unit, University
  College London, 2003.

\bibitem[Kearns and Singh(2002)]{kearns2002near}
Michael Kearns and Satinder Singh.
\newblock Near-optimal reinforcement learning in polynomial time.
\newblock \emph{Machine Learning}, 49\penalty0 (2-3):\penalty0 209--232, 2002.

\bibitem[Kleinberg et~al.(2008)Kleinberg, Slivkins, and
  Upfal]{kleinberg2008multi}
Robert Kleinberg, Aleksandrs Slivkins, and Eli Upfal.
\newblock Multi-armed bandits in metric spaces.
\newblock In \emph{Proceedings of the 40th Annual ACM Symposium on Theory of
  Computing}, pages 681--690, 2008.

\bibitem[Kolmogorov and Tikhomirov(1959)]{kolmogorov1959varepsilon}
Andrei~Nikolaevich Kolmogorov and Vladimir~Mikhailovich Tikhomirov.
\newblock $\varepsilon$-entropy and $\varepsilon$-capacity of sets in function
  spaces.
\newblock \emph{Uspekhi Matematicheskikh Nauk}, 14\penalty0 (2):\penalty0
  3--86, 1959.

\bibitem[Kolter and Ng(2009)]{kolter2009near}
J~Zico Kolter and Andrew~Y Ng.
\newblock Near-{B}ayesian exploration in polynomial time.
\newblock In \emph{Proceedings of the 26th Annual International Conference on
  machine Learning}, pages 513--520, 2009.

\bibitem[Kostina and Hassibi(2019)]{kostina2019rate}
Victoria Kostina and Babak Hassibi.
\newblock Rate-cost tradeoffs in control.
\newblock \emph{IEEE Transactions on Automatic Control}, 64\penalty0
  (11):\penalty0 4525--4540, 2019.

\bibitem[Krishnamurthy et~al.(2016)Krishnamurthy, Agarwal, and
  Langford]{krishnamurthy2016pac}
Akshay Krishnamurthy, Alekh Agarwal, and John Langford.
\newblock {PAC} reinforcement learning with rich observations.
\newblock In \emph{Proceedings of the 30th International Conference on Neural
  Information Processing Systems}, pages 1848--1856, 2016.

\bibitem[Lai and Robbins(1985)]{lai1985asymptotically}
Tze~Leung Lai and Herbert Robbins.
\newblock Asymptotically efficient adaptive allocation rules.
\newblock \emph{Advances in applied mathematics}, 6\penalty0 (1):\penalty0
  4--22, 1985.

\bibitem[Lattimore and Gyorgy(2021)]{lattimore2021mirror}
Tor Lattimore and Andras Gyorgy.
\newblock Mirror descent and the information ratio.
\newblock In \emph{Conference on Learning Theory}, pages 2965--2992. PMLR,
  2021.

\bibitem[Lattimore and Szepesv{\'a}ri(2019)]{lattimore2019information}
Tor Lattimore and Csaba Szepesv{\'a}ri.
\newblock An information-theoretic approach to minimax regret in partial
  monitoring.
\newblock In \emph{Conference on Learning Theory}, pages 2111--2139. PMLR,
  2019.

\bibitem[Lattimore and Szepesv{\'a}ri(2020)]{lattimore2020bandit}
Tor Lattimore and Csaba Szepesv{\'a}ri.
\newblock \emph{Bandit algorithms}.
\newblock Cambridge University Press, 2020.

\bibitem[Li et~al.(2006)Li, Walsh, and Littman]{li2006towards}
Lihong Li, Thomas~J Walsh, and Michael~L Littman.
\newblock Towards a unified theory of state abstraction for {MDP}s.
\newblock \emph{ISAIM}, 4:\penalty0 5, 2006.

\bibitem[Liu et~al.(2021)Liu, Raghunathan, Liang, and Finn]{liu2021decoupling}
Evan~Z Liu, Aditi Raghunathan, Percy Liang, and Chelsea Finn.
\newblock Decoupling exploration and exploitation for meta-reinforcement
  learning without sacrifices.
\newblock In \emph{International conference on machine learning}, pages
  6925--6935. PMLR, 2021.

\bibitem[Lu and Van~Roy(2017)]{lu2017ensemble}
Xiuyuan Lu and Benjamin Van~Roy.
\newblock Ensemble sampling.
\newblock In \emph{Proceedings of the 31st International Conference on Neural
  Information Processing Systems}, pages 3260--3268, 2017.

\bibitem[Lu and Van~Roy(2019)]{lu2019information}
Xiuyuan Lu and Benjamin Van~Roy.
\newblock Information-theoretic confidence bounds for reinforcement learning.
\newblock \emph{Advances in Neural Information Processing Systems},
  32:\penalty0 2461--2470, 2019.

\bibitem[Lu et~al.(2021)Lu, Van~Roy, Dwaracherla, Ibrahimi, Osband, and
  Wen]{lu2021reinforcement}
Xiuyuan Lu, Benjamin Van~Roy, Vikranth Dwaracherla, Morteza Ibrahimi, Ian
  Osband, and Zheng Wen.
\newblock {R}einforcement {L}earning, {B}it by {B}it.
\newblock \emph{arXiv preprint arXiv:2103.04047}, 2021.

\bibitem[Matz and Duhamel(2004)]{matz2004information}
Gerald Matz and Pierre Duhamel.
\newblock Information geometric formulation and interpretation of accelerated
  {B}lahut-{A}rimoto-type algorithms.
\newblock In \emph{Information theory workshop}, pages 66--70. IEEE, 2004.

\bibitem[Misra et~al.(2020)Misra, Henaff, Krishnamurthy, and
  Langford]{misra2020kinematic}
Dipendra Misra, Mikael Henaff, Akshay Krishnamurthy, and John Langford.
\newblock Kinematic state abstraction and provably efficient rich-observation
  reinforcement learning.
\newblock In \emph{International Conference on Machine Learning}, pages
  6961--6971. PMLR, 2020.

\bibitem[Mitter and Sahai(1999)]{mitter1999information}
Sanjoy Mitter and Anant Sahai.
\newblock Information and control: {W}itsenhausen revisited.
\newblock In \emph{Learning, Control and Hybrid Systems}, pages 281--293.
  Springer, 1999.

\bibitem[Mitter(2001)]{mitter2001control}
Sanjoy~K Mitter.
\newblock Control with limited information.
\newblock \emph{European Journal of Control}, 7\penalty0 (2-3):\penalty0
  122--131, 2001.

\bibitem[Nair et~al.(2020)Nair, Savarese, and Finn]{nair2020goal}
Suraj Nair, Silvio Savarese, and Chelsea Finn.
\newblock Goal-aware prediction: Learning to model what matters.
\newblock In \emph{International Conference on Machine Learning}, pages
  7207--7219. PMLR, 2020.

\bibitem[Naja et~al.(2009)Naja, Alberge, and Duhamel]{naja2009geometrical}
Ziad Naja, Florence Alberge, and Pierre Duhamel.
\newblock Geometrical interpretation and improvements of the
  {B}lahut-{A}rimoto's algorithm.
\newblock In \emph{2009 IEEE International Conference on Acoustics, Speech and
  Signal Processing}, pages 2505--2508. IEEE, 2009.

\bibitem[Niesen et~al.(2007)Niesen, Shah, and Wornell]{niesen2007adaptive}
Urs Niesen, Devavrat Shah, and Gregory Wornell.
\newblock Adaptive alternating minimization algorithms.
\newblock In \emph{2007 IEEE International Symposium on Information Theory},
  pages 1641--1645. IEEE, 2007.

\bibitem[Nikishin et~al.(2022)Nikishin, Abachi, Agarwal, and
  Bacon]{nikishin2022control}
Evgenii Nikishin, Romina Abachi, Rishabh Agarwal, and Pierre-Luc Bacon.
\newblock Control-oriented model-based reinforcement learning with implicit
  differentiation.
\newblock In \emph{Proceedings of the AAAI Conference on Artificial
  Intelligence}, 2022.

\bibitem[O'Donoghue et~al.(2018)O'Donoghue, Osband, Munos, and
  Mnih]{o2018uncertainty}
Brendan O'Donoghue, Ian Osband, Remi Munos, and Volodymyr Mnih.
\newblock The uncertainty {B}ellman equation and exploration.
\newblock In \emph{International Conference on Machine Learning}, pages
  3836--3845, 2018.

\bibitem[O'Donoghue et~al.(2020)O'Donoghue, Osband, and Ionescu]{o2020making}
Brendan O'Donoghue, Ian Osband, and Catalin Ionescu.
\newblock Making sense of reinforcement learning and probabilistic inference.
\newblock In \emph{International Conference on Learning Representations}, 2020.

\bibitem[Oh et~al.(2017)Oh, Singh, and Lee]{oh2017value}
Junhyuk Oh, Satinder Singh, and Honglak Lee.
\newblock Value prediction network.
\newblock In \emph{Proceedings of the 31st International Conference on Neural
  Information Processing Systems}, pages 6120--6130, 2017.

\bibitem[Osband(2016)]{osband2016deepthesis}
Ian Osband.
\newblock \emph{Deep Exploration via Randomized Value Functions}.
\newblock PhD thesis, Stanford University, 2016.

\bibitem[Osband and Van~Roy(2014)]{osband2014model}
Ian Osband and Benjamin Van~Roy.
\newblock Model-based reinforcement learning and the {E}luder dimension.
\newblock \emph{Advances in Neural Information Processing Systems}, 27, 2014.

\bibitem[Osband and Van~Roy(2017{\natexlab{a}})]{osband2017gaussian}
Ian Osband and Benjamin Van~Roy.
\newblock Gaussian-{D}irichlet posterior dominance in sequential learning.
\newblock \emph{arXiv preprint arXiv:1702.04126}, 2017{\natexlab{a}}.

\bibitem[Osband and Van~Roy(2017{\natexlab{b}})]{osband2017posterior}
Ian Osband and Benjamin Van~Roy.
\newblock Why is posterior sampling better than optimism for reinforcement
  learning?
\newblock In \emph{International Conference on Machine Learning}, pages
  2701--2710. PMLR, 2017{\natexlab{b}}.

\bibitem[Osband et~al.(2013)Osband, Russo, and Van~Roy]{osband2013more}
Ian Osband, Daniel Russo, and Benjamin Van~Roy.
\newblock ({M}ore) efficient reinforcement learning via posterior sampling.
\newblock \emph{Advances in Neural Information Processing Systems},
  26:\penalty0 3003--3011, 2013.

\bibitem[Osband et~al.(2016{\natexlab{a}})Osband, Blundell, Pritzel, and
  Van~Roy]{osband2016deep}
Ian Osband, Charles Blundell, Alexander Pritzel, and Benjamin Van~Roy.
\newblock Deep exploration via {B}ootstrapped {DQN}.
\newblock In \emph{Advances in Neural Information Processing Systems}, pages
  4026--4034, 2016{\natexlab{a}}.

\bibitem[Osband et~al.(2016{\natexlab{b}})Osband, Van~Roy, and
  Wen]{osband2016generalization}
Ian Osband, Benjamin Van~Roy, and Zheng Wen.
\newblock Generalization and exploration via randomized value functions.
\newblock In \emph{International Conference on Machine Learning}, pages
  2377--2386, 2016{\natexlab{b}}.

\bibitem[Osband et~al.(2018)Osband, Aslanides, and
  Cassirer]{osband2018randomized}
Ian Osband, John Aslanides, and Albin Cassirer.
\newblock Randomized prior functions for deep reinforcement learning.
\newblock \emph{Advances in Neural Information Processing Systems}, 31, 2018.

\bibitem[Osband et~al.(2019)Osband, Van~Roy, Russo, and Wen]{osband2019deep}
Ian Osband, Benjamin Van~Roy, Daniel~J Russo, and Zheng Wen.
\newblock Deep exploration via randomized value functions.
\newblock \emph{Journal of Machine Learning Research}, 20\penalty0
  (124):\penalty0 1--62, 2019.

\bibitem[Osband et~al.(2021{\natexlab{a}})Osband, Wen, Asghari, Ibrahimi, Lu,
  and Van~Roy]{osband2021epistemic}
Ian Osband, Zheng Wen, Mohammad Asghari, Morteza Ibrahimi, Xiyuan Lu, and
  Benjamin Van~Roy.
\newblock Epistemic neural networks.
\newblock \emph{arXiv preprint arXiv:2107.08924}, 2021{\natexlab{a}}.

\bibitem[Osband et~al.(2021{\natexlab{b}})Osband, Wen, Asghari, Dwaracherla,
  Hao, Ibrahimi, Lawson, Lu, O'Donoghue, and Van~Roy]{osband2021evaluating}
Ian Osband, Zheng Wen, Seyed~Mohammad Asghari, Vikranth Dwaracherla, Botao Hao,
  Morteza Ibrahimi, Dieterich Lawson, Xiuyuan Lu, Brendan O'Donoghue, and
  Benjamin Van~Roy.
\newblock Evaluating predictive distributions: Does {B}ayesian deep learning
  work?
\newblock \emph{arXiv preprint arXiv:2110.04629}, 2021{\natexlab{b}}.

\bibitem[Palaiyanur and Sahai(2008)]{palaiyanur2008uniform}
Hari Palaiyanur and Anant Sahai.
\newblock On the uniform continuity of the rate-distortion function.
\newblock In \emph{2008 IEEE International Symposium on Information Theory},
  pages 857--861. IEEE, 2008.

\bibitem[Polyanskiy and Wu(2019)]{polyanskiy2019lecture}
Yury Polyanskiy and Yihong Wu.
\newblock Lecture notes on information theory.
\newblock 2019.

\bibitem[Puterman(1994)]{Puterman94}
Martin~L. Puterman.
\newblock \emph{{M}arkov Decision Processes---Discrete Stochastic Dynamic
  Programming}.
\newblock John Wiley \& Sons, Inc., New York, NY, 1994.

\bibitem[Rakelly et~al.(2019)Rakelly, Zhou, Finn, Levine, and
  Quillen]{rakelly2019efficient}
Kate Rakelly, Aurick Zhou, Chelsea Finn, Sergey Levine, and Deirdre Quillen.
\newblock Efficient off-policy meta-reinforcement learning via probabilistic
  context variables.
\newblock In \emph{International conference on machine learning}, pages
  5331--5340. PMLR, 2019.

\bibitem[Rose(1994)]{rose1994mapping}
Kenneth Rose.
\newblock A mapping approach to rate-distortion computation and analysis.
\newblock \emph{IEEE Transactions on Information Theory}, 40\penalty0
  (6):\penalty0 1939--1952, 1994.

\bibitem[Rubin et~al.(2012)Rubin, Shamir, and Tishby]{rubin2012trading}
Jonathan Rubin, Ohad Shamir, and Naftali Tishby.
\newblock Trading value and information in {MDP}s.
\newblock In \emph{Decision Making with Imperfect Decision Makers}, pages
  57--74. Springer, 2012.

\bibitem[Rusmevichientong and Tsitsiklis(2010)]{rusmevichientong2010linearly}
Paat Rusmevichientong and John~N Tsitsiklis.
\newblock Linearly parameterized bandits.
\newblock \emph{Mathematics of Operations Research}, 35\penalty0 (2):\penalty0
  395--411, 2010.

\bibitem[Russo and Van~Roy(2014)]{russo2014learning}
Daniel Russo and Benjamin Van~Roy.
\newblock Learning to optimize via information-directed sampling.
\newblock \emph{Advances in Neural Information Processing Systems},
  27:\penalty0 1583--1591, 2014.

\bibitem[Russo and Van~Roy(2016)]{russo2016information}
Daniel Russo and Benjamin Van~Roy.
\newblock An information-theoretic analysis of {T}hompson sampling.
\newblock \emph{The Journal of Machine Learning Research}, 17\penalty0
  (1):\penalty0 2442--2471, 2016.

\bibitem[Russo and Van~Roy(2018{\natexlab{a}})]{russo2018learning}
Daniel Russo and Benjamin Van~Roy.
\newblock Learning to optimize via information-directed sampling.
\newblock \emph{Operations Research}, 66\penalty0 (1):\penalty0 230--252,
  2018{\natexlab{a}}.

\bibitem[Russo and Van~Roy(2018{\natexlab{b}})]{russo2018satisficing}
Daniel Russo and Benjamin Van~Roy.
\newblock Satisficing in time-sensitive bandit learning.
\newblock \emph{arXiv preprint arXiv:1803.02855}, 2018{\natexlab{b}}.

\bibitem[Russo and Van~Roy(2022)]{russo2022satisficing}
Daniel Russo and Benjamin Van~Roy.
\newblock Satisficing in time-sensitive bandit learning.
\newblock \emph{Mathematics of Operations Research}, 2022.

\bibitem[Russo et~al.(2017)Russo, Tse, and Van~Roy]{russo2017time}
Daniel Russo, David Tse, and Benjamin Van~Roy.
\newblock Time-sensitive bandit learning and satisficing {T}hompson sampling.
\newblock \emph{arXiv preprint arXiv:1704.09028}, 2017.

\bibitem[Russo et~al.(2018)Russo, Van~Roy, Kazerouni, Osband, and
  Wen]{russo2018tutorial}
Daniel~J Russo, Benjamin Van~Roy, Abbas Kazerouni, Ian Osband, and Zheng Wen.
\newblock A tutorial on {T}hompson sampling.
\newblock \emph{Foundations and Trends{\textregistered} in Machine Learning},
  11\penalty0 (1):\penalty0 1--96, 2018.

\bibitem[Ryzhov et~al.(2012)Ryzhov, Powell, and Frazier]{ryzhov2012knowledge}
Ilya~O Ryzhov, Warren~B Powell, and Peter~I Frazier.
\newblock The knowledge gradient algorithm for a general class of online
  learning problems.
\newblock \emph{Operations Research}, 60\penalty0 (1):\penalty0 180--195, 2012.

\bibitem[Sayir(2000)]{sayir2000iterating}
Jossy Sayir.
\newblock Iterating the {A}rimoto-{B}lahut algorithm for faster convergence.
\newblock In \emph{2000 IEEE International Symposium on Information Theory
  (Cat. No. 00CH37060)}, page 235. IEEE, 2000.

\bibitem[Schrittwieser et~al.(2020)Schrittwieser, Antonoglou, Hubert, Simonyan,
  Sifre, Schmitt, Guez, Lockhart, Hassabis, Graepel,
  et~al.]{schrittwieser2020mastering}
Julian Schrittwieser, Ioannis Antonoglou, Thomas Hubert, Karen Simonyan,
  Laurent Sifre, Simon Schmitt, Arthur Guez, Edward Lockhart, Demis Hassabis,
  Thore Graepel, et~al.
\newblock Mastering {A}tari, {G}o, {C}hess and {S}hogi by planning with a
  learned model.
\newblock \emph{Nature}, 588\penalty0 (7839):\penalty0 604--609, 2020.

\bibitem[Shafieepoorfard et~al.(2016)Shafieepoorfard, Raginsky, and
  Meyn]{shafieepoorfard2016rationally}
Ehsan Shafieepoorfard, Maxim Raginsky, and Sean~P Meyn.
\newblock Rationally inattentive control of {M}arkov processes.
\newblock \emph{SIAM Journal on Control and Optimization}, 54\penalty0
  (2):\penalty0 987--1016, 2016.

\bibitem[Shannon(1959)]{shannon1959coding}
Claude~E. Shannon.
\newblock Coding theorems for a discrete source with a fidelity criterion.
\newblock \emph{IRE Nat. Conv. Rec., March 1959}, 4:\penalty0 142--163, 1959.

\bibitem[Silver et~al.(2017)Silver, Hasselt, Hessel, Schaul, Guez, Harley,
  Dulac-Arnold, Reichert, Rabinowitz, Barreto, et~al.]{silver2017predictron}
David Silver, Hado Hasselt, Matteo Hessel, Tom Schaul, Arthur Guez, Tim Harley,
  Gabriel Dulac-Arnold, David Reichert, Neil Rabinowitz, Andre Barreto, et~al.
\newblock The {P}redictron: End-to-end learning and planning.
\newblock In \emph{International Conference on Machine Learning}, pages
  3191--3199. PMLR, 2017.

\bibitem[Simon(1982)]{simon1982models}
Herbert~A. Simon.
\newblock Models of bounded rationality.
\newblock \emph{Economic Analysis and Public Policy, MIT Press, Cambridge,
  Mass}, 1982.

\bibitem[Sorg et~al.(2010)Sorg, Singh, and Lewis]{sorg2010variance}
Jonathan Sorg, Satinder Singh, and Richard~L Lewis.
\newblock Variance-based rewards for approximate {B}ayesian reinforcement
  learning.
\newblock In \emph{Proceedings of the Twenty-Sixth Conference on Uncertainty in
  Artificial Intelligence}, pages 564--571, 2010.

\bibitem[Stjernvall(1983)]{stjernvall1983dominance}
Jan-Erik Stjernvall.
\newblock Dominance--a relation between distortion measures.
\newblock \emph{IEEE Transactions on Information Theory}, 29\penalty0
  (6):\penalty0 798--807, 1983.

\bibitem[Strehl et~al.(2009)Strehl, Li, and Littman]{strehl2009reinforcement}
Alexander~L Strehl, Lihong Li, and Michael~L Littman.
\newblock Reinforcement learning in finite {MDP}s: {PAC} analysis.
\newblock \emph{Journal of Machine Learning Research}, 10\penalty0
  (Nov):\penalty0 2413--2444, 2009.

\bibitem[Strens(2000)]{strens2000bayesian}
Malcolm~JA Strens.
\newblock A {B}ayesian framework for reinforcement learning.
\newblock In \emph{Proceedings of the Seventeenth International Conference on
  Machine Learning}, pages 943--950, 2000.

\bibitem[Sun et~al.(2019)Sun, Jiang, Krishnamurthy, Agarwal, and
  Langford]{sun2019model}
Wen Sun, Nan Jiang, Akshay Krishnamurthy, Alekh Agarwal, and John Langford.
\newblock Model-based {RL} in contextual decision processes: {PAC} bounds and
  exponential improvements over model-free approaches.
\newblock In \emph{Conference on Learning Theory}, pages 2898--2933. PMLR,
  2019.

\bibitem[Sutton and Barto(1998)]{sutton1998introduction}
Richard~S Sutton and Andrew~G Barto.
\newblock Introduction to reinforcement learning.
\newblock 1998.

\bibitem[Tatikonda and Mitter(2004)]{tatikonda2004control}
Sekhar Tatikonda and Sanjoy Mitter.
\newblock Control under communication constraints.
\newblock \emph{IEEE Transactions on Automatic Control}, 49\penalty0
  (7):\penalty0 1056--1068, 2004.

\bibitem[Thompson(1933)]{thompson1933likelihood}
William~R Thompson.
\newblock On the likelihood that one unknown probability exceeds another in
  view of the evidence of two samples.
\newblock \emph{Biometrika}, 25\penalty0 (3/4):\penalty0 285--294, 1933.

\bibitem[Tishby and Polani(2011)]{tishby2011information}
Naftali Tishby and Daniel Polani.
\newblock Information theory of decisions and actions.
\newblock In \emph{Perception-action cycle}, pages 601--636. Springer, 2011.

\bibitem[Tishby et~al.(2000)Tishby, Pereira, and Bialek]{tishby2000information}
Naftali Tishby, Fernando~C Pereira, and William Bialek.
\newblock The information bottleneck method.
\newblock \emph{arXiv preprint physics/0004057}, 2000.

\bibitem[Van~Roy(2006)]{van2006performance}
Benjamin Van~Roy.
\newblock Performance loss bounds for approximate value iteration with state
  aggregation.
\newblock \emph{Mathematics of Operations Research}, 31\penalty0 (2):\penalty0
  234--244, 2006.

\bibitem[Verd{\'u} et~al.(1994)]{verdu1994generalizing}
Sergio Verd{\'u} et~al.
\newblock Generalizing the {F}ano inequality.
\newblock \emph{IEEE Transactions on Information Theory}, 40\penalty0
  (4):\penalty0 1247--1251, 1994.

\bibitem[Voelcker et~al.(2022)Voelcker, Liao, Garg, and
  Farahmand]{voelcker2022value}
Claas~A Voelcker, Victor Liao, Animesh Garg, and Amir-massoud Farahmand.
\newblock Value gradient weighted model-based reinforcement learning.
\newblock In \emph{International Conference on Learning Representations}, 2022.

\bibitem[Vontobel et~al.(2008)Vontobel, Kavcic, Arnold, and
  Loeliger]{vontobel2008generalization}
Pascal~O Vontobel, Aleksandar Kavcic, Dieter~M Arnold, and Hans-Andrea
  Loeliger.
\newblock A generalization of the {B}lahut--{A}rimoto algorithm to finite-state
  channels.
\newblock \emph{IEEE Transactions on Information Theory}, 54\penalty0
  (5):\penalty0 1887--1918, 2008.

\bibitem[Wang et~al.(2005)Wang, Lizotte, Bowling, and
  Schuurmans]{wang2005bayesian}
Tao Wang, Daniel Lizotte, Michael Bowling, and Dale Schuurmans.
\newblock {B}ayesian sparse sampling for on-line reward optimization.
\newblock In \emph{Proceedings of the 22nd International Conference on Machine
  Learning}, pages 956--963, 2005.

\bibitem[Wang et~al.(2008)Wang, Audibert, and Munos]{wang2008algorithms}
Yizao Wang, Jean-Yves Audibert, and R{\'e}mi Munos.
\newblock Algorithms for infinitely many-armed bandits.
\newblock In \emph{Proceedings of the 21st International Conference on Neural
  Information Processing Systems}, pages 1729--1736, 2008.

\bibitem[Whitt(1978)]{whitt1978approximations}
Ward Whitt.
\newblock Approximations of dynamic programs, {I}.
\newblock \emph{Mathematics of Operations Research}, 3\penalty0 (3):\penalty0
  231--243, 1978.

\bibitem[Witsenhausen(1971)]{witsenhausen1971separation}
Hans~S Witsenhausen.
\newblock Separation of estimation and control for discrete time systems.
\newblock \emph{Proceedings of the IEEE}, 59\penalty0 (11):\penalty0
  1557--1566, 1971.

\bibitem[Yu(2010)]{yu2010squeezing}
Yaming Yu.
\newblock Squeezing the {A}rimoto--{B}lahut algorithm for faster convergence.
\newblock \emph{IEEE Transactions on Information Theory}, 56\penalty0
  (7):\penalty0 3149--3157, 2010.

\bibitem[Zanette and Brunskill(2019)]{zanette2019tighter}
Andrea Zanette and Emma Brunskill.
\newblock Tighter problem-dependent regret bounds in reinforcement learning
  without domain knowledge using value function bounds.
\newblock In \emph{International Conference on Machine Learning}, pages
  7304--7312. PMLR, 2019.

\bibitem[Zimmert and Lattimore(2019)]{zimmert2019connections}
Julian Zimmert and Tor Lattimore.
\newblock Connections between mirror descent, {T}hompson sampling and the
  information ratio.
\newblock In \emph{Advances in Neural Information Processing Systems}, pages
  11973--11982, 2019.

\bibitem[Zintgraf et~al.(2019)Zintgraf, Shiarlis, Igl, Schulze, Gal, Hofmann,
  and Whiteson]{zintgraf2019varibad}
Luisa Zintgraf, Kyriacos Shiarlis, Maximilian Igl, Sebastian Schulze, Yarin
  Gal, Katja Hofmann, and Shimon Whiteson.
\newblock Vari{BAD}: {A} {V}ery {G}ood {M}ethod for {B}ayes-{A}daptive {D}eep
  {RL} via {M}eta-{L}earning.
\newblock In \emph{International Conference on Learning Representations}, 2019.

\end{thebibliography}

%%%%%%%%%%%%%%%%%%%%%%%%%%%%%%%%%%%%%%%%%%%%%%%%%%%%%%%%%%%%
\section*{Checklist}

\begin{enumerate}
\item For all authors...
\begin{enumerate}
  \item Do the main claims made in the abstract and introduction accurately reflect the paper's contributions and scope?
    \answerYes{}
  \item Did you describe the limitations of your work?
    \answerYes{}
  \item Did you discuss any potential negative societal impacts of your work?
    \answerNA{}{}
  \item Have you read the ethics review guidelines and ensured that your paper conforms to them?
    \answerYes{}
\end{enumerate}

\item If you are including theoretical results...
\begin{enumerate}
  \item Did you state the full set of assumptions of all theoretical results?
    \answerYes{}
        \item Did you include complete proofs of all theoretical results?
    \answerYes{}
\end{enumerate}

\item If you ran experiments...
\begin{enumerate}
  \item Did you include the code, data, and instructions needed to reproduce the main experimental results (either in the supplemental material or as a URL)?
    \answerNA{}
  \item Did you specify all the training details (e.g., data splits, hyperparameters, how they were chosen)?
    \answerNA{}
        \item Did you report error bars (e.g., with respect to the random seed after running experiments multiple times)?
    \answerNA{}
        \item Did you include the total amount of compute and the type of resources used (e.g., type of GPUs, internal cluster, or cloud provider)?
    \answerNA{}
\end{enumerate}

\item If you are using existing assets (e.g., code, data, models) or curating/releasing new assets...
\begin{enumerate}
  \item If your work uses existing assets, did you cite the creators?
    \answerNA{}
  \item Did you mention the license of the assets?
    \answerNA{}
  \item Did you include any new assets either in the supplemental material or as a URL?
    \answerNA{}
  \item Did you discuss whether and how consent was obtained from people whose data you're using/curating?
    \answerNA{}
  \item Did you discuss whether the data you are using/curating contains personally identifiable information or offensive content?
    \answerNA{}
\end{enumerate}

\item If you used crowdsourcing or conducted research with human subjects...
\begin{enumerate}
  \item Did you include the full text of instructions given to participants and screenshots, if applicable?
    \answerNA{}
  \item Did you describe any potential participant risks, with links to Institutional Review Board (IRB) approvals, if applicable?
    \answerNA{}
  \item Did you include the estimated hourly wage paid to participants and the total amount spent on participant compensation?
    \answerNA{}
\end{enumerate}

\end{enumerate}

%%%%%%%%%%%%%%%%%%%%%%%%%%%%%%%%%%%%%%%%%%%%%%%%%%%%%%%%%%%%

%%%%%%%%%%%%%%%%%%%%%%%%%%%%%%%%%%%%%%%%%%%%%%%%%%%%%%%%%%%%

\newpage
\appendix

\section{Related Work}
\label{sec:related}

This paper follows suit with a long line of work on provably-efficient reinforcement learning~\citep{kearns2002near,brafman2002r,kakade2003sample,auer2009near,bartlett2009regal,strehl2009reinforcement,jaksch2010near,osband2013more,dann2015sample,osband2017posterior,azar2017minimax,dann2017unifying,agrawal2017optimistic,jin2018q,zanette2019tighter,dong2021simple,lu2021reinforcement}. As previously discussed, these methods can be categorized based on their use of optimism in the face of uncertainty or posterior sampling to address the exploration challenge. Notably, methods in the latter category are Bayesian reinforcement-learning algorithms~\citep{ghavamzadeh2015bayesian} that, through their use of Thompson sampling~\citep{thompson1933likelihood,russo2018tutorial}, are exclusively concerned with identifying optimal solutions. The notable exception to this statement is the method of \citet{lu2021reinforcement}, which is based on information-directed sampling~\citep{russo2014learning,russo2018learning}; while their analysis does accommodate other learning targets besides the optimal policy, an agent designer is responsible for supplying this target to the agent a priori whereas we adaptively compute an information-theoretically sound target grounded in rate-distortion theory. 

In contrast to this class of approaches, optimism-based methods tend to obey PAC-MDP guarantees~\citep{kakade2003sample,strehl2009reinforcement} which, given a fixed parameter $\eps > 0$, offer a high-probability bound on the total number of timesteps for which the agent's behavior is worse than $\eps$-sub-optimal. Through this tolerance parameter $\eps$, an agent designer can express a preference for efficiently identifying a deliberately sub-optimal solution; our work can be seen as providing an analogous knob for Bayesian reinforcement-learning methods that deliberately pursue a satisficing solution while also remaining competitive with regret guarantees for optimism-based methods~\citep{dann2015sample,dann2017unifying,jin2018q,zanette2019tighter}. In this way, our theoretical guarantees are more general than those for PSRL~\citep{osband2013more,osband2014model,abbasi2014bayesian,osband2017posterior,agrawal2017optimistic}. Importantly, the nature of our contribution is not to be confused with the PAC-BAMDP framework of \citet{kolter2009near} which characterizes algorithms that adhere to a high-probability bound on the total number of sub-optimal timesteps relative to the Bayes-optimal policy~\citep{asmuth2009bayesian,sorg2010variance}. We refer readers to the work of \citet{ghavamzadeh2015bayesian} for a broader survey of Bayesian reinforcement-learning methods, including those which do not employ posterior sampling~\citep{strens2000bayesian}, but instead entertain other approximations~\citep{dearden1998bayesian,dearden1999model,wang2005bayesian,castro2007using,araya2012near,guez2012efficient,guez2013scalable,guez2014bayes} to tractably solve the resulting Bayes-Adaptive Markov Decision Process (BAMDP)~\citep{duff2002optimal}, typically while foregoing rigorous theoretical guarantees.

A perhaps third distinct class of provably-efficient reinforcement-learning algorithms~\citep{krishnamurthy2016pac,jiang2017contextual,dann2018oracle,du2019provably,sun2019model} proceeds by iteratively selecting an element of a function class (typically denoting a collection of regressors for either a value function or transition model), inducing a policy from the chosen function, and then carefully eliminating all hypotheses of the function class that are inconsistent with the observed data resulting from policy rollouts in the environment. To the extent that one might be willing to characterize this high-level algorithmic template as an iterative, manual compression and refinement of the initial function class, our algorithm can be seen as bringing the appropriate tool of rate-distortion theory to bear on the inherent lossy compression problem and developing the complementary information-theoretic analysis.

The concept of designing algorithms to learn such near-optimal or satisficing solutions has been well-studied in the multi-armed bandit setting~\citep{bubeck2012regret,lattimore2020bandit}. Indeed, the need to forego optimizing for an optimal arm arises naturally in various contexts~\citep{bubeck2011x,kleinberg2008multi,rusmevichientong2010linearly,ryzhov2012knowledge,deshpande2012linear,berry1997,wang2008algorithms,bonald2013two}. A general study of such satisficing solutions through the lens of information theory was first proposed by \citet{russo2017time,russo2018satisficing,russo2022satisficing} and later extended to develop practical algorithms by \citet{arumugam2021deciding,arumugam2021the}. Our work provides the natural, theoretical extension of these ideas to the full reinforcement-learning setting, leaving investigation of practical instantiations to future work (see Section \ref{sec:disc}). The algorithm and regret bound we provide bears some resemblance to the compressed Thompson sampling algorithm of \citet{dong2018information} for bandit problems. Crucially, while the compressive statistic of the environment utilized by their algorithm is computed once a priori, our algorithm recomputes its learning target in each episode, refining it as the agent's knowledge of the true environment accumulates. Similar to these prior works, we leverage rate-distortion theory~\citep{shannon1959coding} as a principled tool for a mathematically-precise characterization of satisficing solutions. We simply note that our use of rate-distortion theory for reinforcement learning in this work stands in stark contrast to that of prior work which examines state abstraction in reinforcement learning~\citep{abel2019state} or attempts to control the entropy of the resulting policy~\citep{tishby2011information,rubin2012trading,shafieepoorfard2016rationally}.

We also recognize the connection between this work and prior work at the intersection of information theory and control theory~\citep{witsenhausen1971separation,mitter1999information,mitter2001control,borkar2001markov,tatikonda2004control,kostina2019rate}. These works parallel our setting in their consideration for an agent that must stabilize a system with limited \textit{observational} capacity, augmenting the standard control objective subject to a constraint on the rate of the channel that processes raw observations; this problem formulation more closely aligns with a partially-observable Markov Decision Process~\citep{aastrom1965optimal,kaelbling1998planning} or an agent learning with a state abstraction~\citep{li2006towards,abel2016near,van2006performance}. In contrast, our work is concerned with an overall limit on the total amount of information an agent may acquire from the environment and, in turn, how that translates into its selection of a feasible learning target. That said, we suspect there could be a strong, subtle synergy between these prior works and the capacity-sensitive performance guarantees for our algorithm (see Section \ref{sec:dist_rate_bounds}).

\section{Discussion}
\label{sec:disc}

In this section, we outline connections between VSRL and follow-up work to the value equivalence principle~\citep{grimm2021proper}, explore opportunities for even further compression through state abstraction~\citep{li2006towards,abel2016near}, and contemplate potential avenues for how our theory might inform practice.

\subsection{Proper Value Equivalence}

While the value equivalence principle examines a single application of each Bellman operator, in follow-up work  \citet{grimm2021proper} introduce the notion of proper value equivalence,  which considers the limit of infinitely many applications or, stated more concisely, the fixed points of the associated operators. A model $\widetilde{\mc{M}}$ is proper value equivalent if $V^\pi_{\mc{M}^\star,h} = V^\pi_{\widetilde{\mc{M}},h}$, $\forall \pi \in \Pi$, $h \in [H]$. This notion allows for a simpler parameterization through the policy class $\Pi$ alone, without the need for a complementary value function class $\mc{V}$. Conveniently, through Proposition 2 of \citet{grimm2021proper}, it follows that to obtain the set of proper value equivalent models with respect to $\Pi$, one need only find the set of models that are value equivalent for each $\pi \in \Pi$ and its induced value function, $V^\pi$. In our context, we can establish an approximate version of this by using the distortion function $d_{\Pi,\mc{V}}$ where $\mc{V} = \{V^\pi \mid \pi \in \Pi^H\}$ (recall that previous results obeyed the less stringent condition that $\mc{V} \supseteq \{V^\pi \mid \pi \in \Pi^H\}$). 

\citet{grimm2021proper} go on to study proper value equivalence for the set of all deterministic policies, $\Pi = \{\mc{S} \ra \mc{A}\}$ and, through their Corollary 1, show that an optimal policy for any model which is proper value equivalent to $\Pi$ is also optimal in the original MDP $\mc{M}^\star$. Again, we recall that our prior guarantees were made under the less restrictive assumption that $\Pi \supseteq \{\mc{S} \ra \mc{A}\}$. Coupling these insights on proper value equivalence together, we see that when VSRL is run with $\Pi = \{\mc{S} \ra \mc{A}\}$ and $\mc{V} = \{V^\pi \mid \pi \in \Pi^H \}$, the agent aims to recover an approximately proper value-equivalent model of the true environment and, when $D = 0$, the optimal policy associated with this compressed MDP will be optimal for $\mc{M}^\star$. Finally, through their Proposition 5, \citet{grimm2021proper}  identify the set of all proper value equivalent models with respect to $\{\mc{S} \ra \mc{A}\}$ as the largest possible value equivalence class that is guaranteed to yield optimal performance in the true environment. Meanwhile, our Lemma \ref{lemma:dominance} again establishes the information-theoretic analogue of this claim; namely, that VSRL configured to learn a model from this largest value equivalence class requires the fewest bits of information from the true environment. The importance of proper value equivalence culminates with experiments that highlight how MuZero~\citep{schrittwieser2020mastering} succeeds by optimizing a proper value-equivalent loss function. We leave to future work the question of how VSRL might pave the way towards more principled exploration strategies for practical algorithms like MuZero.

\subsection{Greater Compression via State Abstraction}
\label{sec:state_abstraction}

A core disconnect between VSRL and contemporary deep model-based reinforcement learning approaches is that our lossy compression problem forces VSRL to identify a model defined with respect to the original state space whereas methods in the latter category learn a model with respect to a state abstraction. Indeed, algorithms like MuZero and its predecessors~\citep{silver2017predictron,oh2017value,schrittwieser2020mastering} never approximate reward functions and transition models with respect to the raw image observations generated by the environment, but instead incrementally learn some latent representation of state upon which a corresponding model is approximated for planning. This philosophy is born out of several years of work that elucidate the important of state abstraction as a key tool for avoiding the irrelevant information encoded in environment states and addressing the challenge of generalization for sample-efficient reinforcement learning large-scale environments~\citep{whitt1978approximations,bertsekas1989adaptive,dean1997model,ferns2004metrics,jong2005state,li2006towards,van2006performance,ferns2012methods,jiang2015abstraction,abel2016near,abel2018state,abel2019state,dong2019provably,du2019provably,arumugam2020randomized,misra2020kinematic,agarwal2020flambe,abel2020value,abel2020thesis,dong2021simple}. In this section, we briefly introduce a small extension of VSRL that builds on these insights to accommodate lossy MDP compressions defined on a simpler, abstract state space (also referred to as aleatoric or situational state by \citet{lu2021reinforcement,dong2021simple}).

Let $\Phi \subseteq \{\mc{S} \ra [Z]\}$ denote a class of state abstractions or quantizers which map environment states to some discrete, finite abstract state space containing a known, fixed number of abstract states $Z \in \bN$. For any abstract action-value function $Q_\phi \in \{[Z] \times \mc{A} \ra \bR\}$ and any state abstraction $\phi \in \Phi$, we denote by $Q_\phi \circ \phi \in \{\mc{S} \times \mc{A} \ra \bR\}$ the composition of the state abstraction and abstract value function such that $Q_\phi \circ \phi$ is a value function for the original MDP. We adopt a similar convention for any policy $\pi_\phi \in \{[Z] \ra \Delta(\mc{A})\}$ such that $\pi_\phi \circ \phi \in \{\mc{S} \ra \Delta(\mc{A})\}$. We now consider carrying out the rate-distortion optimization of VSRL in each episode over abstract MDPs such that $\widetilde{\mc{M}}_k \in \mathfrak{M}_\phi \triangleq \{[Z] \times \mc{A} \ra [0,1]\} \times \{[Z] \times \mc{A} \ra \Delta([Z])\}$. Just as before, we take the input information source to our lossy compression problem in each episode $k \in [K]$ as the agent's current beliefs over the true MDP, $\bP(\mc{M}^\star \in \cdot \mid H_k)$. Unlike the preceding sections, our distortion function $d: \mathfrak{M} \times \mathfrak{M}_\phi \ra \bR_{\geq 0}$ must now quantify the loss of fidelity incurred by using a compressed abstract MDP in lieu of the true environment MDP. Consequently, we define a new distortion function $$d_{\Phi}(\mc{M}, \widehat{\mc{M}}) = \sup\limits_{\phi \in \Phi} \sup\limits_{h \in [H]} ||Q^\star_{\mc{M},h} - Q^\star_{\widehat{\mc{M}},h} \circ \phi||_\infty^2 = \sup\limits_{\phi \in \Phi} \sup\limits_{h \in [H]} \max\limits_{(s,a) \in \mc{S} \times \mc{A}} | Q^\star_{\mc{M},h}(s,a) - Q^\star_{\widehat{\mc{M}},h}(\phi(s),a)|^2,$$
whose corresponding rate-distortion function is given by
\begin{align*}
    \mc{R}^{\Phi}_k(D) = \inf\limits_{\widetilde{\mc{M}} \in \Lambda_k(D)} \bI_k(\mc{M}^\star; \widetilde{\mc{M}}) \qquad \Lambda_k(D) \triangleq \{\widetilde{\mc{M}}: \Omega \ra \mathfrak{M} \mid \bE_k\left[d_{\Phi}(\mc{M}^\star, \widetilde{\mc{M}})\right] \leq D\} .
\end{align*}

Unlike when performing a lossy compression where $\widetilde{\mc{M}}_k \in \mathfrak{M}$, the channel that represents the identity mapping is no longer a viable option as we must now generate an abstract MDP that resides in $\mathfrak{M}_\phi$. Consequently, we require the following assumption on $\Phi$ to ensure that the set of channels over which we compute the infimum of $\mc{R}^{\Phi}_k(D)$ is non-empty.

\begin{assumption}
For each $k \in [K]$, we have that $\Lambda_k(D) \neq \emptyset$. 
\label{assume:phi}
\end{assumption}

\begin{center}
% \hfill
\begin{minipage}{0.95\textwidth}
\begin{algorithm}[H]
   \caption{Compressed Value-equivalent Sampling for Reinforcement Learning (Compressed-VSRL)}
   \label{alg:cvsrl}
\begin{algorithmic}
   \STATE {\bfseries Input:} Prior distribution $\bP(\mc{M}^\star \in \cdot \mid H_1)$, Distortion threshold $D \in \bR_{\geq 0}$, State abstraction class $\Phi$, Distortion function $d_{\Phi}: \mathfrak{M} \times \mathfrak{M}_\phi \ra \bR_{\geq 0}$,
   \FOR{$k \in [K]$}
   \STATE Compute channel $\bP(\widetilde{\mc{M}}_k \in \cdot | \mc{M}^\star)$ achieving $\mc{R}^{\Phi}_k(D)$ limit
%   \STATE Set state abstraction $\phi_k$ to achieve the infimum: $\inf\limits_{\phi \in \Phi} \sup\limits_{h \in [H]} ||Q^\star_{\mc{M}^\star,h} - Q^\star_{\widehat{\mc{M}}_k,h} \circ \phi||_\infty^2$
   \STATE Sample MDP $M^\star \sim \bP(\mc{M}^\star \in \cdot \mid H_k)$
   \STATE Sample compressed MDP $M_k \sim \bP(\widetilde{\mc{M}}_k \in \cdot \mid \mc{M}^\star = M^\star)$
   \STATE Set state abstraction $\phi_k$ to achieve the infimum: $\inf\limits_{\phi \in \Phi} \sup\limits_{h \in [H]} ||Q^\star_{M^\star,h} - Q^\star_{M_k,h} \circ \phi||_\infty^2$
   \STATE Compute optimal policy $\pi^\star_{M_k}$ and set $\pi^{(k)} = \pi^\star_{M_k} \circ \phi_k$
   \STATE Execute $\pi^{(k)}$ and observe trajectory $\tau_k$
   \STATE Update history $H_{k+1} = H_k \cup \tau_k$
   \STATE Induce posterior $\bP(\mc{M}^\star \in \cdot \mid H_{k+1})$
   \ENDFOR
\end{algorithmic}
\end{algorithm}
\end{minipage}
% \hfill
\end{center}

We present our Compressed-VSRL extension as Algorithm \ref{alg:cvsrl} which incorporates an additional step beyond VSRL to govern the choice of state abstraction utilized in conjunction with the sampled compressed MDP in each episode.

We strongly suspect that an analysis paralleling that of Corollaries \ref{thm:info_regret_bound} and \ref{thm:info_regret_bound_qdist}, with an appropriately defined information ratio, can be carried out for Compressed-VSRL as well. However, for the sake of brevity and since the result is neither immediate nor trivial, we leave the information-theoretic regret bound stated as a conjecture.

\begin{conjecture}
Fix any $D > 0$. For any prior distribution $\bP(\mc{M}^\star \in \cdot \mid H_1)$, if $\Gamma_k \leq \overline{\Gamma}$ for all $k \in [K]$, then CVSRL (Algorithm \ref{alg:cvsrl}) with distortion function $d_{\Phi}$ has
$$\textsc{BayesRegret}(K, \pi^{(1)},\ldots,\pi^{(K)}) \leq \sqrt{\overline{\Gamma}K\mc{R}^{\Phi}_1(D)} + 2K(H+1)\sqrt{D}.$$
\label{conj:info_regret_bound_phi}
\end{conjecture}

The significance of Conjecture \ref{conj:info_regret_bound_phi} for allowing a simple, bounded agent to contend with a complex environment manifests when considering analogues to Theorems \ref{thm:tabular_regret} and \ref{thm:tabular_regret_qdist}. Specifically, for any finite-horizon, episodic MDP with a finite action space $(|\mc{A}| < \infty)$, one may upper bound the rate-distortion function via the entropy in the abstract model $\bH_1(\mc{R}_\phi, \mc{T}_\phi)$. Using the same proof technique as in the preceding results, this facilitates an upper bound $\mc{R}^{\Phi}_1(D) \leq \widetilde{\mc{O}}\left(Z^2|\mc{A}|\right)$ which lacks any dependence on the complexity of the (potentially infinite) environment state space, $\mc{S}$.

\subsection{From Theory to Practice}

While the performance guarantees of VSRL hold for any finite-horizon, episodic MDP, it is important to reconcile that generality with the practicality of the instantiating the algorithm. The three key barriers to practical, scalable implementations of VSRL applied to complex tasks of interest are the representation of epistemic uncertainty, the computation of the rate-distortion function, and the synthesis of optimal policies for sampled MDPs. The first point is a fundamental obstacle to Bayesian reinforcement-learning algorithms and recent work in deep reinforcement learning has found success with simple, albeit computationally-inefficient, ensembles of networks~\citep{osband2016deep,lu2017ensemble,osband2018randomized} or even hypermodels~\citep{dwaracherla2020hypermodels}. As progress is made towards more computationally-efficient models for representing and resolving epistemic uncertainty through Bayesian deep learning~\citep{osband2021epistemic,osband2021evaluating}, there will be greater potential for a practical implementation of VSRL.

For addressing the second issue, a classic option for computing the channel that achieves the rate-distortion limit is the Blahut-Arimoto algorithm~\citep{blahut1972computation,arimoto1972algorithm} which, in theory, is a well-defined procedure even for random variables defined on abstract alphabets~\citep{csiszar1974extremum,csiszar1974computation}. In practice, however, computing such a channel for continuous outputs remains an open challenge~\citep{dauwels2005numerical}; still, several analyses and refinements have been made to the algorithm so far~\citep{boukris1973upper,rose1994mapping,sayir2000iterating,matz2004information,chiang2004geometric,niesen2007adaptive,vontobel2008generalization,naja2009geometrical,yu2010squeezing}, and the reinforcement-learning community stands to greatly benefit from further improvements. Continuous information sources, however, are less problematic as one may draw a sufficiently large number of i.i.d. samples and substitute this empirical distribution for the source, leading to the so-called plug-in estimator of the rate-distortion function for which consistency and sample-complexity guarantees are known~\citep{harrison2008estimation,palaiyanur2008uniform}. Moreover, empirical successes for such estimators have already been demonstrated in the multi-armed bandit setting~\citep{arumugam2021deciding,arumugam2021the}.

The last issue touches upon the fact that while tabular problems admit several planning algorithms for recovering the optimal policy associated with the sampled MDP in each episode, the same cannot be said for arbitrary state-action spaces. At best, one might hope for simply recovering an approximation to this policy through some high-dimensional model-based planning algorithm. We leave the questions of how to practically implement such a procedure and understand its impact on our theory to future work.

Of course, all of the aforementioned issues arise when trying to directly implement VSRL roughly as described by Algorithm \ref{alg:vsrl}. An alternative, however, is to ask how one might take existing practical algorithms already operating at scale (such as MuZero~\citep{schrittwieser2020mastering}) and bring those methods closer to the spirit of VSRL? Since these practical model-based reinforcement-learning algorithms are already engaging with some form of state abstraction~\citep{li2006towards,abel2016near,van2006performance}, this might entail further consideration for information-theoretic approaches to guiding representation learning~\citep{abel2019state,shafieepoorfard2016rationally} as a proxy to engaging with a rate-distortion trade-off. Additionally, one of the core insights developed in this work is the formalization of model simplifications arising out of the value equivalence principle as a form of lossy compression. Curiously, recent meta reinforcement-learning approaches applied to multi-task or contextual MDPs~\citep{hallak2015contextual} arrive at a similar need of obtaining compressed representations of the underlying MDP model during a meta-exploration phase in order to facilitate few-shot learning~\citep{rakelly2019efficient,zintgraf2019varibad,liu2021decoupling}. Given such approaches for learning latent/contextual variational encoders, we strongly suspect that the known connections between rate-distortion theory and information bottlenecks~\citep{tishby2000information,alemi2018fixing} represent a viable path to bridging the ideas between our theory and value equivalence in practice. Notably, while these methods adopt a probabilistic inference perspective and do articulate the encoders as posterior distributions over the underlying task, they lack any representation of epistemic uncertainty~\citep{o2020making}, similar to MuZero; consequently, this still leaves open the earlier obstacle of how best to represent and maintain notions of epistemic uncertainty in large-scale agents.

A third competing perspective is to recognize that recent empirical successes in Bayesian reinforcement learning often avoid representing uncertainty over the model of the environment in favor of the underlying optimal action-value functions, $\{Q^\star_{\mc{M}^\star,h}\}_{h \in [H]}$. Such approaches apply an algorithm known as Randomized Value Functions (RVF), rather than PSRL~\citep{osband2016deep,osband2016generalization,osband2018randomized,o2018uncertainty,osband2019deep}. Naturally, the optimal action-value functions of the true MDP $\mc{M}^\star$ and its model are related and, in fact, there is an equivalence between RVF and PSRL~\citep{osband2016deepthesis}. By maintaining epistemic uncertainty over the action-value functions $\bP(\{Q^\star_{\mc{M}^\star,h}\}_{h \in [H]} \in \cdot \mid H_k)$ rather than the underlying model, RVF methods are amenable to practical instantiation with deep neural networks~\citep{osband2016deep,osband2018randomized,o2018uncertainty,osband2021epistemic}. Moreover, extensions of this line of work have gone on to consider other scalable avenues for leveraging such principled, practical solutions to the exploration challenge through successor features~\citep{dayan1993improving,janz2019successor} and temporal-difference errors~\citep{flennerhag2020temporal}. Overall, we strongly suspect that VSRL can also play an analogous role to PSRL in this sense, providing a sound theoretical foundation that gives rise to subsequent practical algorithms of a slightly different flavor.

\section{Proof of Theorem \ref{thm:regret_decomp}}
\label{sec:regret_decomp_proof}

Before we can prove Theorem \ref{thm:regret_decomp}, we require the following lemma whose proof we adapt from \citet{osband2013more}:
\begin{lemma}
Let $\mc{M},\widehat{\mc{M}}$ be two arbitrary finite-horizon, episodic MDPs with models $(\mc{R}, \mc{T})$ and $(\widehat{\mc{R}}, \widehat{\mc{T}})$, respectively. Then, for any non-stationary policy $\pi = (\pi_1,\ldots,\pi_H) \in \{\mc{S} \ra \Delta(\mc{A})\}^H$, $$V^{\pi}_{\mc{M},1} - V^{\pi}_{\widehat{\mc{M}},1} = \sum\limits_{h=1}^H \bE\left[\mc{B}^{\pi_{h}}_{\mc{M}}V^{\pi}_{\mc{M},h+1}(s_h) - \mc{B}^{\pi_{h}}_{\widehat{\mc{M}}}V^{\pi}_{\mc{M},h+1}(s_h)\right] = \sum\limits_{h=1}^H \bE\left[\mc{B}^{\pi_{h}}_{\mc{M}}V^{\pi}_{\widehat{\mc{M}},h+1}(s_h) - \mc{B}^{\pi_{h}}_{\widehat{\mc{M}}}V^{\pi}_{\widehat{\mc{M}},h+1}(s_h)\right].$$
\label{lemma:planning_err_decomp}
\end{lemma}
\begin{proof}
% \begin{dproof}
By simply applying the Bellman equations, we have
\begin{align*}
    V^{\pi}_{\mc{M},1} - V^{\pi}_{\widehat{\mc{M}},1} &= \bE\left[V^{\pi}_{\mc{M},1}(s_1) - V^{\pi}_{\widehat{\mc{M}},1}(s_1) \right] \\
    &= \bE\left[\mc{B}^{\pi_1}_{\mc{M}}V^{\pi}_{\mc{M},2}(s_1) - \mc{B}^{\pi_1}_{\widehat{\mc{M}}}V^{\pi}_{\widehat{\mc{M}},2}(s_1) \right] \\
    &= \bE\left[\mc{B}^{\pi_1}_{\mc{M}}V^{\pi}_{\mc{M},2}(s_1) - \mc{B}^{\pi_1}_{\widehat{\mc{M}}}V^{\pi}_{\mc{M},2}(s_1) + \mc{B}^{\pi_1}_{\widehat{\mc{M}}}V^{\pi}_{\mc{M},2}(s_1) - \mc{B}^{\pi_1}_{\widehat{\mc{M}}}V^{\pi}_{\widehat{\mc{M}},2}(s_1) \right] \\
    &= \bE\left[\mc{B}^{\pi_1}_{\mc{M}}V^{\pi}_{\mc{M},2}(s_1) - \mc{B}^{\pi_1}_{\widehat{\mc{M}}}V^{\pi}_{\mc{M},2}(s_1) + \bE_{s_2 \sim \widehat{\mc{T}}(\cdot \mid s_1, a_1)}\left[ V^{\pi}_{\mc{M},2}(s_2) - V^{\pi}_{\widehat{\mc{M}},2}(s_2) \right]\right] \\
    &= \sum\limits_{h=1}^2 \bE\left[\mc{B}^{\pi_h}_{\mc{M}}V^{\pi}_{\mc{M},h+1}(s_h) - \mc{B}^{\pi_h}_{\widehat{\mc{M}}}V^{\pi}_{\mc{M},h+1}(s_h)\right] + \bE\left[ V^{\pi}_{\mc{M},3}(s_3) - V^{\pi}_{\widehat{\mc{M}},3}(s_3)\right]\\
    &= \ldots \\
    &= \sum\limits_{h=1}^H \bE\left[\mc{B}^{\pi_h}_{\mc{M}}V^{\pi}_{\mc{M},h+1}(s_h) - \mc{B}^{\pi_h}_{\widehat{\mc{M}}}V^{\pi}_{\mc{M},h+1}(s_h)\right] + \ubr{\bE\left[ V^{\pi}_{\mc{M},H+1}(s_{H+1}) - V^{\pi}_{\widehat{\mc{M}},H+1}(s_{H+1})\right]}_{=0}\\
    &= \sum\limits_{h=1}^H \bE\left[\mc{B}^{\pi_h}_{\mc{M}}V^{\pi}_{\mc{M},h+1}(s_h) - \mc{B}^{\pi_h}_{\widehat{\mc{M}}}V^{\pi}_{\mc{M},h+1}(s_h)\right].
\end{align*}
For the second identity, we have nearly identical steps:
\begin{align*}
    V^{\pi}_{\mc{M},1} - V^{\pi}_{\widehat{\mc{M}},1} &= \bE\left[V^{\pi}_{\mc{M},1}(s_1) - V^{\pi}_{\widehat{\mc{M}},1}(s_1) \right] \\
    &= \bE\left[\mc{B}^{\pi_1}_{\mc{M}}V^{\pi}_{\mc{M},2}(s_1) - \mc{B}^{\pi_1}_{\widehat{\mc{M}}}V^{\pi}_{\widehat{\mc{M}},2}(s_1) \right] \\ 
    &= \bE\left[\mc{B}^{\pi_1}_{\mc{M}}V^{\pi}_{\mc{M},2}(s_1) - \mc{B}^{\pi_1}_{\mc{M}}V^{\pi}_{\widehat{\mc{M}},2}(s_1) + \mc{B}^{\pi_1}_{\mc{M}}V^{\pi}_{\widehat{\mc{M}},2}(s_1) -  \mc{B}^{\pi_1}_{\widehat{\mc{M}}}V^{\pi}_{\widehat{\mc{M}},2}(s_1) \right] \\ 
    &= \bE\left[\mc{B}^{\pi_1}_{\mc{M}}V^{\pi}_{\widehat{\mc{M}},2}(s_1) -  \mc{B}^{\pi_1}_{\widehat{\mc{M}}}V^{\pi}_{\widehat{\mc{M}},2}(s_1) + \bE_{s_2 \sim \mc{T}(\cdot \mid s_1, a_1)}\left[V^{\pi}_{\mc{M},2}(s_2) - V^{\pi}_{\widehat{\mc{M}},2}(s_2)\right]\right] \\
    &= \sum\limits_{h=1}^2 \bE\left[\mc{B}^{\pi_h}_{\mc{M}}V^{\pi}_{\widehat{\mc{M}},h+1}(s_h) - \mc{B}^{\pi_h}_{\widehat{\mc{M}}}V^{\pi}_{\widehat{\mc{M}},h+1}(s_h)\right] + \bE\left[ V^{\pi}_{\mc{M},3}(s_3) - V^{\pi}_{\widehat{\mc{M}},3}(s_3)\right]\\
    &= \ldots \\
    &= \sum\limits_{h=1}^H \bE\left[\mc{B}^{\pi_h}_{\mc{M}}V^{\pi}_{\widehat{\mc{M}},h+1}(s_h) - \mc{B}^{\pi_h}_{\widehat{\mc{M}}}V^{\pi}_{\widehat{\mc{M}},h+1}(s_h)\right] + \ubr{\bE\left[ V^{\pi}_{\mc{M},H+1}(s_{H+1}) - V^{\pi}_{\widehat{\mc{M}},H+1}(s_{H+1})\right]}_{=0}\\
    &= \sum\limits_{h=1}^H \bE\left[\mc{B}^{\pi_h}_{\mc{M}}V^{\pi}_{\widehat{\mc{M}},h+1}(s_h) - \mc{B}^{\pi_h}_{\widehat{\mc{M}}}V^{\pi}_{\widehat{\mc{M}},h+1}(s_h)\right].
\end{align*}
% \end{dproof}
\end{proof}

\begin{theorem}
Take any $\Pi \supseteq \{\mc{S} \ra \mc{A}\}$, any $\mc{V} \supseteq \{V^\pi \mid \pi \in \Pi^H\}$, and fix any $D \geq 0$. For each episode $k \in [K]$, let $\widetilde{\mc{M}}_k$ be any MDP that achieves the rate-distortion limit of $\mc{R}^{\Pi,\mc{V}}_k(D)$ with information source $\bP(\mc{M}^\star \in \cdot \mid H_k)$ and distortion function $d_{\Pi,\mc{V}}$. Then, 
$$\textsc{BayesRegret}(K, \pi^{(1)},\ldots,\pi^{(K)}) \leq \bE\left[\sum\limits_{k=1}^K \bE_k\left[V^{\star}_{\widetilde{\mc{M}}_k,1} - V^{\pi^{(k)}}_{\widetilde{\mc{M}}_k,1}\right]\right] + 2KH\sqrt{D}.$$
\end{theorem}
\begin{proof}
% \begin{dproof}
By applying definitions from Section \ref{sec:problem_form} and applying the tower property of expectation, we have that 
$$\textsc{BayesRegret}(K, \pi^{(1)},\ldots,\pi^{(K)}) = \bE\left[\sum\limits_{k=1}^K \bE_k\left[\Delta_k\right]\right].$$ Examining the $k$th episode in isolation and applying the definition of episodic regret, we have
\begin{align*}
    \bE_k\left[\Delta_k\right] &= \bE_k\left[V^\star_{\mc{M}^\star,1} - V^{\pi^{(k)}}_{\mc{M}^\star, 1}\right] \\
    &= \bE_k\left[V^\star_{\mc{M}^\star,1} - V^\star_{\widetilde{\mc{M}}_k,1} + V^\star_{\widetilde{\mc{M}}_k,1} - V^{\pi^{(k)}}_{\widetilde{\mc{M}}_k,1} + V^{\pi^{(k)}}_{\widetilde{\mc{M}}_k,1} - V^{\pi^{(k)}}_{\mc{M}^\star, 1}\right] \\
    &= \bE_k\left[V^\star_{\mc{M}^\star,1} - V^{\pi^\star_{\mc{M}^\star}}_{\widetilde{\mc{M}}_k,1} + \ubr{V^{\pi^\star_{\mc{M}^\star}}_{\widetilde{\mc{M}}_k,1} - V^\star_{\widetilde{\mc{M}}_k,1}}_{\leq 0} + V^\star_{\widetilde{\mc{M}}_k,1} - V^{\pi^{(k)}}_{\widetilde{\mc{M}}_k,1} + V^{\pi^{(k)}}_{\widetilde{\mc{M}}_k,1} - V^{\pi^{(k)}}_{\mc{M}^\star, 1}\right] \\
    &\leq \bE_k\left[V^\star_{\mc{M}^\star,1} - V^{\pi^\star_{\mc{M}^\star}}_{\widetilde{\mc{M}}_k,1} + V^\star_{\widetilde{\mc{M}}_k,1} - V^{\pi^{(k)}}_{\widetilde{\mc{M}}_k,1} + V^{\pi^{(k)}}_{\widetilde{\mc{M}}_k,1} - V^{\pi^{(k)}}_{\mc{M}^\star, 1}\right] \\
\end{align*}
For brevity, we let $\pi^\star_h \triangleq \pi^\star_{\mc{M}^\star, h}$ and observe that an application of Lemma \ref{lemma:planning_err_decomp} yields
\begin{align*}
    \bE_k\left[V^\star_{\mc{M}^\star,1} - V^{\pi^\star_{\mc{M}^\star}}_{\widetilde{\mc{M}}_k,1} \right] &= \sum\limits_{h=1}^H \bE_k\left[\mc{B}^{\pi^\star_h}_{\mc{M}^\star}V^\star_{\mc{M}^\star,h+1}(s_h) - \mc{B}^{\pi^\star_h}_{\widetilde{\mc{M}}_k}V^\star_{\mc{M}^\star,h+1}(s_h)\right] \\
    &\leq \sum\limits_{h=1}^H \bE_k\left[\big|\mc{B}^{\pi^\star_h}_{\mc{M}^\star}V^\star_{\mc{M}^\star,h+1}(s_h) - \mc{B}^{\pi^\star_h}_{\widetilde{\mc{M}}_k}V^\star_{\mc{M}^\star,h+1}(s_h)\big|\right] \\
    &= \sum\limits_{h=1}^H \bE_k\left[\sqrt{\left(\mc{B}^{\pi^\star_h}_{\mc{M}^\star}V^\star_{\mc{M}^\star,h+1}(s_h) - \mc{B}^{\pi^\star_h}_{\widetilde{\mc{M}}_k}V^\star_{\mc{M}^\star,h+1}(s_h)\right)^2}\right] \\
    &\leq \sum\limits_{h=1}^H \bE_k\left[\sqrt{||\mc{B}^{\pi^\star_h}_{\mc{M}^\star}V^\star_{\mc{M}^\star,h+1} - \mc{B}^{\pi^\star_h}_{\widetilde{\mc{M}}_k}V^\star_{\mc{M}^\star,h+1}||_\infty^2}\right] \\
    &\leq \sum\limits_{h=1}^H \sqrt{\bE_k\left[||\mc{B}^{\pi^\star_h}_{\mc{M}^\star}V^\star_{\mc{M}^\star,h+1} - \mc{B}^{\pi^\star_h}_{\widetilde{\mc{M}}_k}V^\star_{\mc{M}^\star,h+1}||_\infty^2\right]} \\
    &\leq \sum\limits_{h=1}^H \sqrt{\bE_k\left[\sup\limits_{\substack{\pi \in \Pi \\ V \in \mc{V}}}||\mc{B}^{\pi}_{\mc{M}^\star}V - \mc{B}^{\pi}_{\widetilde{\mc{M}}_k}V||_\infty^2\right]} \\
    &= \sum\limits_{h=1}^H \sqrt{\bE_k\left[d_{\Pi,\mc{V}}(\mc{M}^\star,\widetilde{\mc{M}}_k)\right]} \\
    &\leq H\sqrt{D},
\end{align*}
where the third inequality invokes Jensen's inequality, the fourth inequality holds as $\Pi \supseteq \{\mc{S} \ra \mc{A}\}$ and $\mc{V} \supseteq \{V^\pi \mid \pi \in \Pi^H\}$ ensures that $V^\star_{\mc{M}^\star,h} \in \mc{V}$, $\forall h \in [H]$, and the final inequality holds since $\widetilde{\mc{M}}_k$ achieves the rate-distortion limit in the $k$th episode, by assumption.

We follow the same sequence of steps to obtain
\begin{align*}
    \bE_k\left[V^{\pi^{(k)}}_{\widetilde{\mc{M}}_k,1} - V^{\pi^{(k)}}_{\mc{M}^\star, 1} \right] &= \sum\limits_{h=1}^H \bE_k\left[\mc{B}^{\pi^{(k)}_h}_{\mc{M}^\star}V^{\pi^{(k)}}_{\mc{M}^\star,h+1}(s_h) - \mc{B}^{\pi^{(k)}_h}_{\widetilde{\mc{M}}_k}V^{\pi^{(k)}}_{\mc{M}^\star,h+1}(s_h)\right] \\
    &\leq \sum\limits_{h=1}^H \bE_k\left[\big|\mc{B}^{\pi^{(k)}_h}_{\mc{M}^\star}V^{\pi^{(k)}}_{\mc{M}^\star,h+1}(s_h) - \mc{B}^{\pi^{(k)}_h}_{\widetilde{\mc{M}}_k}V^{\pi^{(k)}}_{\mc{M}^\star,h+1}(s_h)\big|\right] \\
    &= \sum\limits_{h=1}^H \bE_k\left[\sqrt{\left(\mc{B}^{\pi^{(k)}_h}_{\mc{M}^\star}V^{\pi^{(k)}}_{\mc{M}^\star,h+1}(s_h) - \mc{B}^{\pi^{(k)}_h}_{\widetilde{\mc{M}}_k}V^{\pi^{(k)}}_{\mc{M}^\star,h+1}(s_h)\right)^2}\right] \\
    &\leq \sum\limits_{h=1}^H \bE_k\left[\sqrt{||\mc{B}^{\pi^{(k)}_h}_{\mc{M}^\star}V^{\pi^{(k)}}_{\mc{M}^\star,h+1} - \mc{B}^{\pi^{(k)}_h}_{\widetilde{\mc{M}}_k}V^{\pi^{(k)}}_{\mc{M}^\star,h+1}||_\infty^2}\right] \\
    &\leq \sum\limits_{h=1}^H \sqrt{\bE_k\left[||\mc{B}^{\pi^{(k)}_h}_{\mc{M}^\star}V^{\pi^{(k)}}_{\mc{M}^\star,h+1} - \mc{B}^{\pi^{(k)}_h}_{\widetilde{\mc{M}}_k}V^{\pi^{(k)}}_{\mc{M}^\star,h+1}||_\infty^2\right]} \\
    &\leq \sum\limits_{h=1}^H \sqrt{\bE_k\left[\sup\limits_{\substack{\pi \in \Pi \\ V \in \mc{V}}}||\mc{B}^{\pi}_{\mc{M}^\star}V - \mc{B}^{\pi}_{\widetilde{\mc{M}}_k}V||_\infty^2\right]} \\
    &= \sum\limits_{h=1}^H \sqrt{\bE_k\left[d_{\Pi,\mc{V}}(\mc{M}^\star,\widetilde{\mc{M}}_k)\right]} \\
    &\leq H\sqrt{D}.
\end{align*}

Substituting back into our original expression, we have 
\begin{align*}
    \bE_k\left[\Delta_k\right] &= \bE_k\left[V^\star_{\mc{M}^\star,1} - V^{\pi^{(k)}}_{\mc{M}^\star, 1}\right] \\
    &\leq \bE_k\left[V^\star_{\mc{M}^\star,1} - V^{\pi^\star_{\mc{M}^\star}}_{\widetilde{\mc{M}}_k,1} + V^\star_{\widetilde{\mc{M}}_k,1} - V^{\pi^{(k)}}_{\widetilde{\mc{M}}_k,1} + V^{\pi^{(k)}}_{\widetilde{\mc{M}}_k,1} - V^{\pi^{(k)}}_{\mc{M}^\star, 1}\right] \\
    &\leq \bE_k\left[V^\star_{\widetilde{\mc{M}}_k,1} - V^{\pi^{(k)}}_{\widetilde{\mc{M}}_k,1}\right] + 2H\sqrt{D}.
\end{align*}

Applying this upper bound on episodic regret in each episode yields
\begin{align*}
    \textsc{BayesRegret}(K, \pi^{(1)},\ldots,\pi^{(K)}) &= \bE\left[\sum\limits_{k=1}^K \bE_k\left[\Delta_k\right]\right] \\
    &\leq \bE\left[\sum\limits_{k=1}^K \bE_k\left[V^\star_{\widetilde{\mc{M}}_k,1} - V^{\pi^{(k)}}_{\widetilde{\mc{M}}_k,1}\right] \right] + 2KH\sqrt{D},
\end{align*}
as desired.
% \end{dproof}
\end{proof}

\section{Proof of Lemma \ref{lemma:cum_info_bound}}

In this section, we develop counterparts to the results of \citet{arumugam2021deciding} for the reinforcement-learning setting which relate each rate-distortion function $\mc{R}^{\Pi,\mc{V}}_k(D)$ to the information accumulated by the agent over the course of learning. Recall that $\tau_k = (s^{(k)}_1, a^{(k)}_1, r^{(k)}_1, \ldots,s^{(k)}_{H}, a^{(k)}_{H}, r^{(k)}_{H}, s^{(k)}_{H+1})$ is a random variable denoting the trajectory experienced by the agent in the $k$th episode given the history $H_k$. Let MDP $M_k$ be the MDP sampled in the $k$th episode.

\begin{lemma}
For all $k \in [K]$, $$\bE_k\left[\mc{R}^{\Pi,\mc{V}}_{k+1}(D)\right] \leq \mc{R}^{\Pi,\mc{V}}_k(D) - \bI_k(\widetilde{\mc{M}}_k; \tau_k \mid M_k).$$
\label{lemma:exp_rdf_decr}
\end{lemma}
\begin{proof}
% \begin{dproof}
Recall that, by definition $H_{k+1} = (H_k, \tau_k)$. For all $k \in [K]$, observe that, conditioned on the true MDP $\mc{M}^\star$ and sampled MDP $M_k$ which generated the history $H_{k+1}$, we have that for any compressed MDP $\widetilde{\mc{M}}$, $\bP(H_{k+1}, \widetilde{\mc{M}} \mid \mc{M}^\star, M_k) = \bP(H_{k+1} \mid \mc{M}^\star, M_k) \bP(\widetilde{\mc{M}} \mid \mc{M}^\star, M_k).$ Using this independence $H_{k+1} \perp \widetilde{\mc{M}} \mid \mc{M}^\star, M_k$ $\forall k \in [K]$, we have that $$0 = \bI_k(H_{k+1}; \widetilde{\mc{M}} \mid \mc{M}^\star, M_k) = \bI_k(H_k, \tau_k; \widetilde{\mc{M}} \mid \mc{M}^\star, M_k) = \bI_k(\tau_k; \widetilde{\mc{M}} \mid \mc{M}^\star, M_k).$$
Moreover, we know that the sampled MDP $M_k$ does not affect our uncertainty in the true MDP $\mc{M}^\star$ such that $$\bI_k(\mc{M}^\star; \widetilde{\mc{M}}) = \bI_k(\mc{M}^\star; \widetilde{\mc{M}} \mid M_k).$$ By the chain rule of mutual information, $$\bI_k(\mc{M}^\star; \widetilde{\mc{M}}) = \bI_k(\mc{M}^\star; \widetilde{\mc{M}} \mid M_k) = \bI_k(\mc{M}^\star; \widetilde{\mc{M}} \mid M_k) + \bI_k(\tau_k; \widetilde{\mc{M}} \mid \mc{M}^\star, M_k) = \bI_k(\mc{M}^\star, \tau_k; \widetilde{\mc{M}} \mid M_k).$$ Applying the chain rule a second time yields $$\bI_k(\mc{M}^\star; \widetilde{\mc{M}}) =  \bI_k(\mc{M}^\star, \tau_k; \widetilde{\mc{M}} \mid M_k) = \bI_k(\widetilde{\mc{M}}; \tau_k \mid M_k) + \bI_k(\mc{M}^\star; \widetilde{\mc{M}} \mid \tau_k, M_k).$$

By definition of the rate-distortion function, we have

$$\bE_k\left[\mc{R}^{\Pi,\mc{V}}_{k+1}(D)\right] = \bE_k\left[\inf\limits_{\widetilde{\mc{M}} \in \Lambda_{k+1}(D)} \bI_{k+1}(\mc{M}^\star; \widetilde{\mc{M}})\right], \qquad \Lambda_{k+1}(D) = \{\widetilde{\mc{M}}: \Omega \ra \mathfrak{M} \mid \bE_{k+1}[d_{\Pi,\mc{V}}(\mc{M}^\star, \widetilde{\mc{M}})] \leq D\}.$$ 

Recall that, by definition, $\widetilde{\mc{M}}_k$ achieves the rate-distortion limit of $\mc{R}^{\Pi,\mc{V}}_{k}(D)$, implying that $\bE_{k}[d_{\Pi,\mc{V}}(\mc{M}^\star, \widetilde{\mc{M}}_k)] \leq D$. By the tower property of expectation, we recover that $$\bE_k\left[\bE_{k+1}[d_{\Pi,\mc{V}}(\mc{M}^\star, \widetilde{\mc{M}}_k)]\right] = \bE_{k}[d_{\Pi,\mc{V}}(\mc{M}^\star, \widetilde{\mc{M}}_k)] \leq D,$$ and so, in expectation given the current history $H_k$, $\widetilde{\mc{M}}_k \in \Lambda_{k+1}(D)$. Thus, we have that 
$$\bE_k\left[\mc{R}^{\Pi,\mc{V}}_{k+1}(D)\right] = \bE_k\left[\inf\limits_{\widetilde{\mc{M}} \in \Lambda_{k+1}(D)} \bI_{k+1}(\mc{M}^\star; \widetilde{\mc{M}})\right] \leq \bE_k\left[\bI_{k+1}(\mc{M}^\star; \widetilde{\mc{M}}_k)\right].$$
Re-arranging terms from our previous chain rule expansions, we may expand the integrand as
\begin{align*}
     \bE_k\left[\bI_{k+1}(\mc{M}^\star; \widetilde{\mc{M}}_k)\right] &= \bE_k\left[\bI_{k}(\mc{M}^\star; \widetilde{\mc{M}}_k \mid \tau_k, M_k)\right] \\
    &= \bE_k\left[\bI_{k}(\mc{M}^\star; \widetilde{\mc{M}}_k) - \bI_k(\widetilde{\mc{M}}_k, \tau_k \mid M_k)\right] \\
    &= \bI_{k}(\mc{M}^\star; \widetilde{\mc{M}}_k) - \bI_k(\widetilde{\mc{M}}_k, \tau_k \mid M_k) \\
    &= \mc{R}^{\Pi,\mc{V}}_k(D) - \bI_k(\widetilde{\mc{M}}_k; \tau_k \mid M_k),
\end{align*}
where the penultimate line follows since both mutual information terms are $\sigma(H_k)$-measurable and the final line follows by definition of $\widetilde{\mc{M}}_k$.
% \end{dproof}
\end{proof}

At the beginning of each episode, our generalization of PSRL will identify a compressed MDP $\widetilde{\mc{M}}_k$ that achieves the rate-distortion limit based on the current history $H_k$. As data accumulates and the agent's knowledge of the true MDP is refined, this satisficing MDP $\widetilde{\mc{M}}_k$ will be recomputed to reflect that updated knowledge. The previous lemma shows that the expected number of bits the agent must identify to learn this new target MDP decreases as this adaptation occurs, highlighting two possible sources of improvement: (1) shifting from a compressed MDP $\widetilde{\mc{M}}_k$ to $\widetilde{\mc{M}}_{k+1}$ and (2) a decrease of $\bI_k(\widetilde{\mc{M}}; \tau_k \mid M_k)$ that occurs from observing the trajectory $\tau_k$. The former reflects the agent's improved ability in synthesizing an approximately value-equivalent MDP to pursue instead of $\mc{M}^\star$ while the latter captures information gained about the previous target $\widetilde{\mc{M}}_k$ from the experienced trajectory $\tau_k$.

\begin{fact}[\citep{cover2012elements}]
For any $\Pi,\mc{V}$ and all $k \in [K]$, $\mc{R}^{\Pi,\mc{V}}_k(D)$ is a non-negative, convex, and monotonically-decreasing function in $D$.
\label{fact:rdf_props}
\end{fact}

Let $\widetilde{\mc{M}}$ be a compressed MDP that is exactly value-equivalent to $\mc{M}^\star$ which, by definition, implies a distortion of exactly zero. Further recall that $\mc{M}^\star$ is itself a MDP that achieves zero distortion, albeit one that has no guarantee of achieving the rate-distortion limit. Fact \ref{fact:rdf_props} yields the following chain of inequalities that hold for all $k \in [K]$ and $D \geq 0$: $$\mc{R}^{\Pi,\mc{V}}_k(D) \leq \bI_k(\mc{M}^\star; \widetilde{\mc{M}}) \leq \bI_k(\mc{M}^\star; \mc{M}^\star) = \bH_k(\mc{M}^\star).$$ This chain of inequalities confirms an important goal of satisficing in PSRL; namely, that the compressed MDP an agent attempts to solve in each episode $k \in [K]$, $\widetilde{\mc{M}}_k$, requires fewer bits of information than what is needed to fully identify the true MDP $\mc{M}^\star$. This gives rise to the following corollary:
\begin{corollary}
For any $k \in [K]$, $$\bE_k\left[\sum\limits_{k'=k}^K \bI_{k'}(\widetilde{\mc{M}}_{k'}; \tau_{k'} \mid M_{k'})\right] \leq \bH_k(\mc{M}^\star).$$
\end{corollary}
Instead of proving this corollary, we prove the following lemma which yields the corollary through Fact \ref{fact:rdf_props}:
\begin{lemma}
For any $k \in [K]$, $$\bE_k\left[\sum\limits_{k'=k}^K \bI_{k'}(\widetilde{\mc{M}}_{k'}; \tau_{k'} \mid M_{k'})\right] \leq \mc{R}^{\Pi,\mc{V}}_k(D).$$
\label{lemma:cum_info_bound}
\end{lemma}
\begin{proof}
% \begin{dproof}
Observe that by Lemma \ref{lemma:exp_rdf_decr}, for all $k \in [K]$, $$\bI_k(\widetilde{\mc{M}}_k;\tau_k \mid M_k) \leq \mc{R}^{\Pi,\mc{V}}_k(D) - \bE_k\left[\mc{R}^{\Pi,\mc{V}}_{k+1}(D)\right].$$ Directly substituting in, we have $$\bE_k\left[\sum\limits_{k'=k}^K \bI_{k'}(\widetilde{\mc{M}}_{k'}; \tau_{k'} \mid M_{k'})\right] \leq \bE_k\left[\sum\limits_{k'=k}^K \left(\mc{R}^{\Pi,\mc{V}}_{k'}(D) - \bE_{k'}\left[\mc{R}^{\Pi,\mc{V}}_{k'+1}(D)\right]\right)\right].$$ Applying linearity of expectation and breaking apart the sum yields $$\bE_k\left[\sum\limits_{k'=k}^K \bI_{k'}(\widetilde{\mc{M}}_{k'}; \tau_{k'} \mid M_{k'})\right] \leq \sum\limits_{k'=k}^K \bE_k\left[\mc{R}^{\Pi,\mc{V}}_{k'}(D)\right] - \sum\limits_{k'=k}^K \bE_k\left[\bE_{k'}\left[\mc{R}^{\Pi,\mc{V}}_{k'+1}(D)\right]\right].$$ Note that the first term can simply be separated into $$\sum\limits_{k'=k}^K \bE_k\left[\mc{R}^{\Pi,\mc{V}}_{k'}(D)\right] = \bE_k\left[\mc{R}^{\Pi,\mc{V}}_{k}(D)\right] + \sum\limits_{k'=k+1}^K \bE_k\left[\mc{R}^{\Pi,\mc{V}}_{k'}(D)\right] = \mc{R}^{\Pi,\mc{V}}_{k}(D) + \sum\limits_{k'=k+1}^K \bE_k\left[\mc{R}^{\Pi,\mc{V}}_{k'}(D)\right].$$ Meanwhile, since $\sigma(H_k) \subseteq \sigma(H_{k'})$, the tower property of expectation yields $$\sum\limits_{k'=k}^K \bE_k\left[\bE_{k'}\left[\mc{R}^{\Pi,\mc{V}}_{k'+1}(D)\right]\right] = \sum\limits_{k'=k}^K \bE_k\left[\mc{R}^{\Pi,\mc{V}}_{k'+1}(D)\right] = \sum\limits_{k'=k+1}^K \bE_k\left[\mc{R}^{\Pi,\mc{V}}_{k'}(D)\right].$$
Combining the expansions results in
\begin{align*}
    \bE_k\left[\sum\limits_{k'=k}^K \bI_{k'}(\widetilde{\mc{M}}_{k'}; \tau_{k'} \mid M_{k'})\right] &\leq \sum\limits_{k'=k}^K \bE_k\left[\mc{R}^{\Pi,\mc{V}}_{k'}(D)\right] - \sum\limits_{k'=k}^K \bE_k\left[\bE_{k'}\left[\mc{R}^{\Pi,\mc{V}}_{k'+1}(D)\right]\right] \\
    &= \mc{R}^{\Pi,\mc{V}}_{k}(D) + \sum\limits_{k'=k+1}^K \bE_k\left[\mc{R}^{\Pi,\mc{V}}_{k'}(D)\right] - \sum\limits_{k'=k+1}^K \bE_k\left[\mc{R}^{\Pi,\mc{V}}_{k'}(D)\right] \\
    &= \mc{R}^{\Pi,\mc{V}}_{k}(D).
\end{align*}
% \end{dproof}
\end{proof}

\section{Proof of Theorem \ref{thm:info_sat_regret_bound}}

In this section, we prove a general, information-theoretic satisficing Bayesian regret bound. Central to our analysis is the information ratio in the $k$th episode: $$\Gamma_k \triangleq \frac{\bE_k\left[V^{\star}_{\widetilde{\mc{M}}_k,1} - V^{\pi^{(k)}}_{\widetilde{\mc{M}}_k,1}\right]^2}{\bI_k(\widetilde{\mc{M}}_k; \tau_k, M_k)}, \qquad \forall k \in [K].$$

\begin{theorem}[Information-Theoretic Satisficing Regret Bound]
If $\Gamma_k \leq \overline{\Gamma}$, for all $k \in [K]$, then $$\bE_k\left[\sum\limits_{k=1}^K \bE\left[V^{\star}_{\widetilde{\mc{M}}_k,1} - V^{\pi^{(k)}}_{\widetilde{\mc{M}}_k,1}\right]\right] \leq \sqrt{\overline{\Gamma}K\mc{R}^{\Pi,\mc{V}}_1(D)}.$$
% $$\bE\left[\sum\limits_{k=1}^K \bE\left[V^{\star}_{\widetilde{\mc{M}}_k,1} - V^{\pi^{(k)}}_{\widetilde{\mc{M}}_k,1} \bigm| H_k \right] \Biggm| \mc{M}^\star \right] \leq \sqrt{\overline{\Gamma}K\bE\left[\mc{R}^{\Pi,\mc{V}}_1(D) \bigm| \mc{M}^\star \right]}.$$
% \label{thm:info_sat_regret_bound}
\end{theorem}
\begin{proof}
% \begin{dproof}
The definition of the information ratio $\Gamma_k$ for each term in the sum followed by the fact that $\Gamma_k \leq \overline{\Gamma}$, $\forall k \in [K]$ yields
$$\bE\left[\sum\limits_{k=1}^K \bE_k\left[V^{\star}_{\widetilde{\mc{M}}_k,1} - V^{\pi^{(k)}}_{\widetilde{\mc{M}}_k,1}\right]\right] = \bE\left[\sum\limits_{k=1}^K \sqrt{\Gamma_k \bI_k(\widetilde{\mc{M}}_k; \tau_k, M_k)} \right] \leq \sqrt{\overline{\Gamma}}\bE\left[\sum\limits_{k=1}^K \sqrt{\bI_k(\widetilde{\mc{M}}_k; \tau_k, M_k)}\right].$$ Applying the tower property of expectation and Jensen's inequality in sequence yields $$\sqrt{\overline{\Gamma}}\bE\left[\sum\limits_{k=1}^K \sqrt{\bI_k(\widetilde{\mc{M}}_k; \tau_k, M_k)} \right] \leq \sqrt{\overline{\Gamma}}\bE\left[\sum\limits_{k=1}^K \sqrt{\bE_k\left[\bI_k(\widetilde{\mc{M}}_k; \tau_k, M_k)\right]} \right].$$ By the Cauchy-Schwarz inequality, we have that $$\sqrt{\overline{\Gamma}}\bE\left[\sum\limits_{k=1}^K \sqrt{\bE_k\left[\bI_k(\widetilde{\mc{M}}_k; \tau_k, M_k)\right]} \right] \leq \sqrt{\overline{\Gamma}}\bE\left[ \sqrt{K \sum\limits_{k=1}^K \bE_k\left[\bI_k(\widetilde{\mc{M}}_k; \tau_k, M_k)\right]} \right].$$ Recall that the sampled $M_k$ by itself offers no information about $\widetilde{\mc{M}}_k$. Consequently, by the chain rule of mutual information, we have $$\bI_k(\widetilde{\mc{M}}_k; \tau_k, M_k) = \bI_k(\widetilde{\mc{M}}_k; M_k) + \bI_k(\widetilde{\mc{M}}_k; \tau_k \mid M_k) = \bI_k(\widetilde{\mc{M}}_k; \tau_k \mid M_k).$$ Therefore, $$\sqrt{\overline{\Gamma}}\bE\left[ \sqrt{K \sum\limits_{k=1}^K \bE_k\left[\bI_k(\widetilde{\mc{M}}_k; \tau_k, M_k)\right]} \right] = \sqrt{\overline{\Gamma}}\bE\left[ \sqrt{K \sum\limits_{k=1}^K \bE_k\left[\bI_k(\widetilde{\mc{M}}_k; \tau_k \mid M_k)\right]} \right].$$ Directly applying Lemma \ref{lemma:cum_info_bound} followed by Jensen's inequality yields $$\sqrt{\overline{\Gamma}}\bE\left[ \sqrt{K \sum\limits_{k=1}^K \bE_k\left[\bI_k(\widetilde{\mc{M}}_k; \tau_k \mid M_k)\right]} \right] \leq \sqrt{\overline{\Gamma}}\bE\left[ \sqrt{K \mc{R}^{\Pi,\mc{V}}_1(D)}\right] \leq \sqrt{\overline{\Gamma}K \bE\left[\mc{R}^{\Pi,\mc{V}}_1(D)\right]}.$$
Since the expectation is with respect to the prior $\bP(\mc{M}^\star \in \cdot \mid H_1)$ and $\mc{R}^{\Pi,\mc{V}}_1(D)$ is $\sigma(H_1)$-measurable, we have $$\sqrt{\overline{\Gamma}K \bE\left[\mc{R}^{\Pi,\mc{V}}_1(D)\right]} = \sqrt{\overline{\Gamma}K\mc{R}^{\Pi,\mc{V}}_1(D)},$$ as desired.
% \end{dproof}
\end{proof}

\section{Proof of Lemma \ref{lemma:dominance}}

In this section, we clarify how the shrinkage or growth of the policy class $\Pi$ and value function class $\mc{V}$ affect the rate-distortion function at the $k$th episode, $\mc{R}^{\Pi,\mc{V}}_k(D)$.

\begin{lemma}[Dominance with Approximate Value Equivalence]
For any two $\Pi,\Pi'$ and any $\mc{V}, \mc{V}'$ such that $\Pi' \subseteq \Pi \subseteq \{\mc{S} \ra \Delta(\mc{A})\}$ and $\mc{V}' \subseteq \mc{V} \subseteq \{\mc{S} \ra \bR\}$, we have $$\mc{R}^{\Pi,\mc{V}}_k(D) \geq \mc{R}^{\Pi',\mc{V}'}_k(D), \qquad \forall k \in [K], D > 0.$$
% \label{thm:dominance}
\end{lemma}
\begin{proof}
% \begin{dproof}
Recall that the distortion function $d:\mathfrak{M} \times \mathfrak{M} \ra \bR_{\geq 0}$ with respect to policy class $\Pi$ and value function class $\mc{V}$ is given by $$d_{\Pi,\mc{V}}(\mc{M},\widehat{\mc{M}}) = \sup\limits_{\substack{\pi \in \Pi \\ V \in \mc{V}}} ||\mc{B}^\pi_{\mc{M}}V - \mc{B}^\pi_{\widehat{\mc{M}}}V||_\infty^2 = \sup\limits_{\substack{\pi \in \Pi \\ V \in \mc{V}}} \left(\max\limits_{s \in \mc{S}} |\mc{B}^\pi_{\mc{M}}V(s) - \mc{B}^\pi_{\widehat{\mc{M}}}V(s)| \right)^2,$$ with an analogous definition holding for the distortion function $d_{\Pi',\mc{V}'}$ under $\Pi'$ and $\mc{V}'$. In the parlance of \citet{stjernvall1983dominance}, we have that $d_{\Pi,\mc{V}}$ dominates $d_{\Pi',\mc{V}'}$ if for all source distributions $\bP(\mc{M}^\star \in \cdot \mid H_k)$ and all distortion thresholds $D > 0$, $$\mc{R}^{\Pi,\mc{V}}_k(D) \geq \mc{R}^{\Pi',\mc{V}'}_k(D).$$ In words, a distortion function $d_1$ that dominates another distortion function $d_2$ requires more bits of information in order to achieve the rate-distortion limit for all information sources and at all distortion thresholds. From this definition, it is clear that statement of the theorem holds if we can establish a dominance relationship between $d_{\Pi,\mc{V}}$ and $d_{\Pi',\mc{V}'}$. 

Recognizing the significant amount of calculation needed to exhaustively verify a dominance relationship by hand, \citet{stjernvall1983dominance} prescribes six sufficient conditions for establishing dominance (with varying degrees of strength) between distortion functions; we will leverage the second of these characterizations (C2). 

Fix an arbitrary source distribution $\bP(\mc{M}^\star \in \cdot \mid H_k)$ and distortion threshold $D > 0$. We denote by $\widetilde{\mc{M}}_k$ the MDP that achieves the rate-distortion limit $\mc{R}^{\Pi,\mc{V}}_k(D)$  under our chosen source, distortion threshold, and distortion function $d_{\Pi,\mc{V}}$. By definition of the supremum, we have that for any two MDPs $\mc{M}, \widehat{\mc{M}}$ $$d_{\Pi',\mc{V}'}(\mc{M}, \widehat{\mc{M}}) = \sup\limits_{\substack{\pi \in \Pi' \\ V \in \mc{V}'}} ||\mc{B}^\pi_{\mc{M}}V - \mc{B}^\pi_{\widehat{\mc{M}}}V||_\infty^2 \leq \sup\limits_{\substack{\pi \in \Pi \\ V \in \mc{V}}} ||\mc{B}^\pi_{\mc{M}}V - \mc{B}^\pi_{\widehat{\mc{M}}}V||_\infty^2 = d_{\Pi,\mc{V}}(\mc{M}, \widehat{\mc{M}}).$$ Consequently, since $\widetilde{\mc{M}}_k$ achieves the rate-distortion limit, we have $$\bE_k\left[d_{\Pi',\mc{V}'}(\mc{M}^\star, \widetilde{\mc{M}}_k)\right] \leq \bE_k\left[d_{\Pi,\mc{V}}(\mc{M}^\star, \widetilde{\mc{M}}_k)\right] \leq D.$$ Observe that, since our information source and distortion threshold were arbitrary, we have that for all sources $\bP(\mc{M}^\star \in \cdot \mid H_k)$ and all thresholds $D > 0$ with $\widetilde{\mc{M}}_k$ achieving the rate-distortion limit under distortion $d_{\Pi,\mc{V}}$, there exists a Markov chain $\mc{M}^\star - \widetilde{\mc{M}}_k - \widetilde{\mc{M}}'_k$ such that $\widetilde{\mc{M}}_k = \widetilde{\mc{M}}'_k$ (the mapping between them is the identity function) and $\bE\left[d_{\Pi',\mc{V}'}(\mc{M}^\star, \widetilde{\mc{M}}'_k)\right] \leq D.$ Thus, by Theorem 2 of \citet{stjernvall1983dominance} (specifically, C2 $\implies$ D4), we have that $d_{\Pi,\mc{V}}$ dominates $d_{\Pi',\mc{V}'}$ for any $\Pi' \subseteq \Pi \subseteq \{\mc{S} \ra \Delta(\mc{A})\}$ and $\mc{V}' \subseteq \mc{V} \subseteq \{\mc{S} \ra \bR\}$. As previously discussed, the claim of the theorem follows as an immediate consequence, by definition of dominance.
% \end{dproof}
\end{proof}

\section{Proof of Lemma \ref{lemma:fano_error_lowerbound}}

Fano's inequality~\citep{fano1952TransInfoLectNotes} is a key result in information theory that relates conditional entropy to the probability of error in a discrete, multi-way hypothesis testing problem. The traditional form of the result, however, determines an error as the inability to exactly recover the random variable being estimated. Naturally, given the lossy compression context of this work, a more useful analysis will use a lack of adherence to the distortion upper bound as the more appropriate notion of error. For this purpose, we require a more general result of the same flavor as those developed by \citet{duchi2013distance}; in particular, we leverage an extension of their generalized Fano's inequality which is given as Question 7.1 in \citep{duchi21ItLectNotes}, whose proof we provide and adapt to our setting for completeness. We first require the following lemma:
\begin{lemma}
Let $P$ and $Q$ be two arbitrary probability measures on the same measurable space such that $P \ll Q$. Then, $$\kl{P}{Q} \geq \log\left(\frac{1}{Q(P > 0)}\right) = \log\left(\frac{1}{Q(\supp{P})}\right).$$ 
\label{lemma:kl_log_supp}
\end{lemma}
\begin{proof}
% \begin{dproof}
The proof is immediate via a generalization of the traditional log-sum inequality~\citep{cover2012elements}. Specifically, since $P \ll Q$, we have $$\kl{P}{Q} = \int \log\left(\frac{dP}{dQ}\right) dP = \int\limits_{P > 0} \log\left(\frac{dP}{dQ}\right) dP \geq \left(\int dP\right)\log\left(\frac{\int dP}{\int\limits_{P > 0} dQ}\right) =  \log\left(\frac{1}{Q(P > 0)}\right).$$
% \end{dproof}
\end{proof}

\begin{theorem}
Take any $\Pi \subseteq \{\mc{S} \ra \Delta(\mc{A})\}$ and $\mc{V} \subseteq \{\mc{S} \ra \bR\}$. For any $D \geq 0$ and any $k \in [K]$, define $\delta = \sup\limits_{\widehat{M} \in \mathfrak{M}} \bP(d_{\Pi,\mc{V}}(\mc{M}^\star,\widehat{M}) \leq D \mid H_k).$ Then, 
$$\sup\limits_{\widetilde{\mc{M}} \in \Lambda_k(D)} \bP(d_{\Pi,\mc{V}}(\mc{M}^\star, \widetilde{\mc{M}}) > D  \mid H_k) \geq 1 - \frac{\mc{R}^{\Pi,\mc{V}}_k(D) + \log(2)}{\log\left(\frac{1}{\delta}\right)}.$$
% \label{thm:fano_error_lowerbound}
\end{theorem}
\begin{proof}
% \begin{dproof}
For any episode $k \in [K]$, recall that the agent's beliefs over the true MDP $\mc{M}^\star$ are distributed according to $\bP(\mc{M}^\star \in \cdot \mid H_k)$. Let $\widetilde{\mc{M}}$ be an arbitrary random variable denoting a compressed MDP taking values in the set $\mathfrak{M}$ and, for a fixed distortion threshold $D$, we let $\mc{N} \subset \mathfrak{M} \times \mathfrak{M}$ denote the measurable subset of $\mathfrak{M} \times \mathfrak{M}$ that consists of all pairs of MDP which are approximately value equivalent; that is, $(M, \widehat{M}) \in \mc{N}$ $\Longleftrightarrow$ $d_{\Pi,\mc{V}}(M, \widehat{M}) \leq D$. For any MDP $\widehat{M} \in \mathfrak{M}$, we define a slice $$\mc{N}_{\widehat{M}} \triangleq \{M \in \mathfrak{M} \mid (M,\widehat{M}) \in \mc{N}\},$$ as the collection of MDPs that are approximately value equivalent to a given $\widehat{M}$. In the context of Fano's inequality and our lossy compression problem, $\mc{N}_{\widehat{M}}$ is the set of original or uncompressed MDPs for which a channel output of $\widehat{M}$ should not be considered an error. Furthermore, define $$p^{\text{max}} \triangleq \sup\limits_{\widehat{M} \in \mathfrak{M}} \bP(\mc{M}^\star \in \mc{N}_{\widehat{M}} \mid H_k) \qquad p^{\text{min}} \triangleq \inf\limits_{\widehat{M} \in \mathfrak{M}} \bP(\mc{M}^\star \in \mc{N}_{\widehat{M}} \mid H_k).$$ Recall that for $p \in [0,1]$, we have the binary entropy function $h_2(p) = -p\log(p) - (1-p)\log(1-p)$. 

Define the indicator random variable $E = \indic((\mc{M}^\star, \widetilde{\mc{M}}) \notin \mc{N})$. Recalling that $$\bI(X;Y) = \bE\left[\kl{\bP(Y \in \cdot \mid X)}{\bP(Y \in \cdot)}\right],$$ we have 
\begin{multline*}
    \bI_k(\mc{M}^\star; (\widetilde{\mc{M}}, E)) \bE\left[\kl{\bP_k(\mc{M}^\star \in \cdot \mid \widetilde{\mc{M}}, E)}{\bP_k(\mc{M}^\star \in \cdot)}\right] \\
    = \bP_k(E = 1) \cdot \bE\left[\kl{\bP_k(\mc{M}^\star \in \cdot \mid \widetilde{\mc{M}}, E = 1)}{\bP_k(\mc{M}^\star \in \cdot)}\right] \\ + \bP_k(E = 0) \cdot \bE\left[\kl{\bP_k(\mc{M}^\star \in \cdot \mid \widetilde{\mc{M}}, E = 0)}{\bP_k(\mc{M}^\star \in \cdot)}\right].
\end{multline*}
At this point, we observe that for any $\widehat{M} \in \mathfrak{M}$, $$\supp{\bP_k(\mc{M}^\star \in \cdot \mid \widetilde{\mc{M}} = \widehat{M}, E = 0)} \subset \mc{N}_{\widehat{M}} \qquad \supp{\bP_k(\mc{M}^\star \in \cdot \mid \widetilde{\mc{M}} = \widehat{M}, E = 1)} \subset \mc{N}_{\widehat{M}}^c,$$ by definition of the slice $\mc{N}_{\widehat{M}}$. Thus, 
\begin{align*}
    \bP(\mc{M}^\star \in \supp{\bP_k(\mc{M}^\star \in \cdot \mid \widetilde{\mc{M}} = \widehat{M}, E = 0)} \mid H_k) &\leq \bP(\mc{M}^\star \in \mc{N}_{\widehat{M}}\mid H_k) \\
    \bP(\mc{M}^\star \in \supp{\bP_k(\mc{M}^\star \in \cdot \mid \widetilde{\mc{M}} = \widehat{M}, E = 1)} \mid H_k) &\leq \bP(\mc{M}^\star \in \mc{N}_{\widehat{M}}^c\mid H_k) = 1 - \bP(\mc{M}^\star \in \mc{N}_{\widehat{M}}\mid H_k)
\end{align*} and, consequently, we have by Lemma \ref{lemma:kl_log_supp} that $$\kl{\bP_k(\mc{M}^\star \in \cdot \mid \widetilde{\mc{M}} = \widehat{M}, E = 0)}{\bP_k(\mc{M}^\star \in \cdot)} \geq \log\left(\frac{1}{\bP(\mc{M}^\star \in \mc{N}_{\widehat{M}}\mid H_k)}\right) \geq \log\left(\frac{1}{p^{\text{max}}}\right),$$ $$\kl{\bP_k(\mc{M}^\star \in \cdot \mid \widetilde{\mc{M}} = \widehat{M}, E = 1)}{\bP_k(\mc{M}^\star \in \cdot)} \geq \log\left(\frac{1}{1-\bP(\mc{M}^\star \in \mc{N}_{\widehat{M}}\mid H_k)}\right) \geq \log\left(\frac{1}{1-p^{\text{min}}}\right).$$ Applying these lower bounds to our original mutual information term, we see that
\begin{align*}
    \bI_k(\mc{M}^\star; (\widetilde{\mc{M}}, E)) &\geq \bP(E = 1 \mid H_k) \log\left(\frac{1}{1-p^{\text{min}}}\right) + \bP(E = 0 \mid H_k)\log\left(\frac{1}{p^{\text{max}}}\right) \\
    &= \bP(E = 1 \mid H_k) \log\left(\frac{1}{1-p^{\text{min}}}\right) + \left(1-\bP(E = 1 \mid H_k)\right)\log\left(\frac{1}{p^{\text{max}}}\right) \\
    &= \bP(E = 1 \mid H_k)\log\left(\frac{p^{\text{max}}}{1-p^{\text{min}}}\right) + \log\left(\frac{1}{p^{\text{max}}}\right).
\end{align*}
Now applying the chain rule of mutual information, the definition of mutual information, the non-negativity of entropy and the fact that conditioning reduces entropy in sequence, we obtain 
\begin{align*}
    \bI_k(\mc{M}^\star; (\widetilde{\mc{M}}, E)) &= \bI_k(\mc{M}^\star; \widetilde{\mc{M}}) + \bI_k(\mc{M}^\star; E \mid \widetilde{\mc{M}}) \\
    &= \bI_k(\mc{M}^\star; \widetilde{\mc{M}}) + \bH_k(E \mid \widetilde{\mc{M}}) - \bH_k(E \mid \widetilde{\mc{M}}, \mc{M}^\star) \\
    &\leq \bI_k(\mc{M}^\star; \widetilde{\mc{M}}) + \bH_k(E \mid \widetilde{\mc{M}}) \\
    &\leq \bI_k(\mc{M}^\star; \widetilde{\mc{M}}) + \bH_k(E) \\
    &\leq \bI_k(\mc{M}^\star; \widetilde{\mc{M}}) + \bH(E) \\
\end{align*}
Combining the upper and lower bounds while multiplying through by $-1$ yields $$h_2(\bP(E = 1)) + \bP(E = 1 \mid H_k)\log\left(\frac{1-p^{\text{min}}}{p^{\text{max}}}\right) \geq \log\left(\frac{1}{p^{\text{max}}}\right) - \bI_k(\mc{M}^\star; \widetilde{\mc{M}}).$$ Recognizing that we have the following upper bounds
\begin{align*}
    \log(2) + \bP(E = 1 \mid H_k)\log\left(\frac{1}{p^{\text{max}}}\right) &\geq h_2(\bP(E = 1)) + \bP(E = 1 \mid H_k)\log\left(\frac{1}{p^{\text{max}}}\right) \\
    &\geq h_2(\bP(E = 1)) + \bP(E = 1 \mid H_k)\log\left(\frac{1-p^{\text{min}}}{p^{\text{max}}}\right),
\end{align*}and re-arranging terms yields $$\bP(E = 1 \mid H_k) \geq \frac{\log\left(\frac{1}{p^{\text{max}}}\right) - \bI_k(\mc{M}^\star; \widetilde{\mc{M}}) - \log(2)}{\log\left(\frac{1}{p^{\text{max}}}\right)} = 1 - \frac{\bI_k(\mc{M}^\star; \widetilde{\mc{M}}) + \log(2)}{\log\left(\frac{1}{\delta}\right)},$$ where $\delta = \sup\limits_{\widehat{M} \in \mathfrak{M}} \bP(d_{\Pi,\mc{V}}(\mc{M}^\star,\widehat{M}) \leq D \mid H_k).$ Noting that $$\bP(E = 1 \mid H_k) = \bP((\mc{M}^\star,\widetilde{\mc{M}}) \notin \mc{N} \mid H_k) = \bP(d_{\Pi,\mc{V}}(\mc{M}^\star, \widetilde{\mc{M}}) > D  \mid H_k),$$ and taking the supremum on both sides, we have 
\begin{align*}
    \sup\limits_{\widetilde{\mc{M}} \in \Lambda_k(D)} \bP(d_{\Pi,\mc{V}}(\mc{M}^\star, \widetilde{\mc{M}}) > D  \mid H_k) &\geq \sup\limits_{\widetilde{\mc{M}} \in \Lambda_k(D)}\left[1 - \frac{\bI_k(\mc{M}^\star; \widetilde{\mc{M}}) + \log(2)}{\log\left(\frac{1}{\delta}\right)}\right] \\
    &= 1 - \inf\limits_{\widetilde{\mc{M}} \in \Lambda_k(D)} \frac{\bI_k(\mc{M}^\star; \widetilde{\mc{M}}) + \log(2)}{\log\left(\frac{1}{\delta}\right)} \\
    &= 1 - \frac{\mc{R}^{\Pi,\mc{V}}_k(D) + \log(2)}{\log\left(\frac{1}{\delta}\right)},
\end{align*}
as desired.
% \end{dproof}
\end{proof}

\section{Proof of Theorem \ref{thm:tabular_regret}}

In specializing to the tabular MDP setting, we wish to simplify our information-theoretic Bayesian regret bound (Corollary \ref{thm:info_regret_bound})  into one that only depends on the standard problem-specific quantities ($|\mc{S}|$, $|\mc{A}|$, $K$, $H$). To do this, we will necessarily decompose mutual information into its constituent entropy terms. Inconveniently, while mutual information is well-defined for arbitrary random variables, entropy is infinite for continuous random variables (like the reward function and transition function random variables, $\mc{R}^\star$ and $\mc{T}^\star$). Rather than resorting to differential entropy, which lacks several desirable properties of Shannon entropy, we explicitly replace these random variables by their discretized analogues, obtained via a sufficiently-fine quantization of their ranges a priori such that the differential entropy of the original random variables is well-approximated by the associated metric entropy or $\eps$-entropy~\citep{kolmogorov1959varepsilon}, courtesy of Theorem 8.3.1 of \citep{cover2012elements}. 

Recall that, for any $\eps > 0$, a $\eps$-cover of a set $\Theta$ with respect to a (semi)-metric $\rho: \Theta \times \Theta \ra \bR_{\geq 0}$ is a set $\{\theta_1,\ldots,\theta_N\}$ with $\theta_i \in \Theta$, $\forall i \in [N]$, such that for any other point $\theta \in \Theta$, $\exists$ $n \in [N]$ such that $\rho(\theta,\theta_n) \leq \eps$. The $\eps$-covering number of $\Theta$ is defined as $$\mc{N}(\eps, \Theta, \rho) \triangleq \inf \{N \in \bN \mid \exists \text{ an }\eps\text{-cover } \{\theta_1,\ldots,\theta_N\} \text{ of } \Theta\}.$$ Conversely, a $\eps$-packing of a set $\Theta$ with respect to $\rho$ is a set $\{\theta_1,\ldots,\theta_M\}$ with $\theta_i \in \Theta$, $\forall i \in [M]$, such that for any distinct $i,j \in [N]$, we have $\rho(\theta_i,\theta_j) \geq \eps$. The $\eps$-packing number of a set $\Theta$ is defined as $$\mc{M}(\eps, \Theta, \rho) \triangleq \sup\{M \in \bN \mid \exists \text{ an }\eps\text{-packing } \{\theta_1,\ldots,\theta_M\} \text{ of } \Theta \}.$$ With slight abuse of notation, for any norm $||\cdot||$ on a set $\Theta$, we write $\mc{N}(\eps, \Theta, ||\cdot||)$ to denote the $\eps$-covering number under the metric induced by $||\cdot||$, and similarly for the $\eps$-packing number $\mc{M}(\eps, \Theta, ||\cdot||)$. Theorem IV of \citep{kolmogorov1959varepsilon} establishes the following relationship between the $\eps$-covering number and $\eps$-packing number that we will use to upper bound metric entropy:
\begin{fact}
For any metric space $(\Theta, \rho)$ and any $\eps > 0$, $\mc{N}(\eps, \Theta, \rho) \leq \mc{M}(\eps, \Theta, \rho).$
\label{fact:cover_pack_bound}
\end{fact}

This allows for a generalization of Lemma 7.6 of \citep{duchi21ItLectNotes} to norm balls of arbitrary radius whose proof we include for completeness.
\begin{lemma}
For any norm $||\cdot||$, let $\bB^d = \{\theta \in \bR^d \mid ||\theta|| \leq 1\}$ denote the unit $||\cdot||$-ball in $\bR^d$. For any $r \in (0,\infty)$, we let $r\bB^d = \{\theta \in \bR^d \mid ||\theta|| \leq r\}$ denote the scaling of the unit ball by $r$ or, equivalently, the $||\cdot||$-ball of radius $r$. Then, for any $\eps \in (0,r]$, $$\log\left(\mc{N}(\eps, r\bB^d, ||\cdot||)\right) \leq d \log\left(1 + \frac{2r}{\eps}\right).$$
\label{lemma:metric_entropy_bound}
\end{lemma}
\begin{proof}
% \begin{dproof}
Let $\vol{\cdot}$ be the function that denotes the volume of an input ball in $\bR^d$ such that $\vol{r\bB^d} = r^d$. Since an $\eps$-packing requires filling $r\bB^d$ with disjoint balls of diameter $\eps$, we have $$\mc{M}(\eps, r\bB^d, ||\cdot||) \vol{\frac{\eps}{2}\bB^d} = \sum\limits_{i=1}^{\mc{M}(\eps, r\bB^d, ||\cdot||)} \vol{\frac{\eps}{2}\bB^d} \leq \vol{\left(r + \frac{\eps}{2}\right)\bB^d}.$$ Dividing through by $\vol{\frac{\eps}{2}\bB^d}$ yields $$\mc{M}(\eps, r\bB^d, ||\cdot||) \leq \frac{\vol{\left(r + \frac{\eps}{2}\right)\bB^d}}{\vol{\frac{\eps}{2}\bB^d}} = \left(\frac{r + \frac{\eps}{2}}{\frac{\eps}{2}}\right)^d = \left(1 + \frac{2r}{\eps}\right)^d.$$
Applying Fact \ref{fact:cover_pack_bound} gives us $$\mc{N}(\eps, r\bB^d, ||\cdot||) \leq \mc{M}(\eps, r\bB^d, ||\cdot||) \leq  \left(1 + \frac{2r}{\eps}\right)^d,$$ and taking logarithms on both sides renders the desired inequality.
% \end{dproof}
\end{proof}

\begin{theorem}
Take any $\Pi \supseteq \{\mc{S} \ra \mc{A}\}$, any $\mc{V} \supseteq \{V^\pi \mid \pi \in \Pi^H\}$, and let $D = 0$. For any prior distribution $\bP(\mc{M}^\star \in \cdot \mid H_1)$ over tabular MDPs, if $\Gamma_k \leq \overline{\Gamma}$ for all $k \in [K]$, then VSRL (Algorithm \ref{alg:vsrl}) has
$$\textsc{BayesRegret}(K, \pi^{(1)},\ldots,\pi^{(K)}) \leq \mc{O}\left(|\mc{S}|\sqrt{\overline{\Gamma}|\mc{A}|K}\right).$$
% \label{thm:tabular_regret}
\end{theorem}
\begin{proof}
% \begin{dproof}
Using Fact \ref{fact:rdf_props}, we have that $$\mc{R}^{\Pi,\mc{V}}_1(D) \leq \bH_1(\mc{M}^\star) = \bH_1(\mc{R}^\star, \mc{T}^\star) = \bH_1(\mc{R}^\star) + \bH_1(\mc{T}^\star \mid \mc{R}^\star) = \bH_1(\mc{R}^\star) + \bH_1(\mc{T}^\star),$$ where the first equality recognizes that all randomness in the true MDP $\mc{M}^\star$ is driven by the model $(\mc{R}^\star,\mc{T}^\star)$, the second equality applies the chain rule of entropy, and the final equality recognizes that the reward function and transition function random variables are independent.

For some fixed $\eps_\mc{R} > 0$, consider the $\frac{\eps_\mc{R}}{2}$-cover of the unit interval $[0,1]$ with respect to the $L_1$-norm $||\cdot||_1$ as a quantization into bins of width $\eps_\mc{R}$. Observe that the true environment reward function $\mc{R}^\star: \mc{S} \times \mc{A} \ra [0,1]$ is well-approximated by mapping state-action pairs onto this $\frac{\eps_\mc{R}}{2}$-cover, for a sufficiently small $\eps_\mc{R} > 0$. Consequently, we treat $\mc{R}^\star$ as a discrete random variable where $|\text{supp}(\mc{R}^\star)| = \mc{N}(\frac{\eps_\mc{R}}{2}, [0,1], ||\cdot||_1)^{|\mc{S}||\mc{A}|}$. Recall that, for a discrete random variable $X$ with support on $\mc{X}$, $\bH(X) \leq \log\left(|\mc{X}|\right).$ Applying this upper bound and Lemma \ref{lemma:metric_entropy_bound} in sequence, we have that $$\bH_1(\mc{R}^\star) \leq |\mc{S}||\mc{A}|\log\left(\mc{N}(\frac{\eps_\mc{R}}{2}, [0,1], ||\cdot||_1)\right) \leq |\mc{S}||\mc{A}|\log\left(1 + \frac{4}{\eps_\mc{R}}\right).$$ Applying the same sequence of steps \textit{mutatis mutandis} for the transition function $\mc{T}^\star$ under a $\frac{\eps_\mc{T}}{2}$-cover, for some fixed $\eps_\mc{T} > 0$, we also have $$\bH_1(\mc{T}^\star) \leq |\mc{S}|^2|\mc{A}|\log\left(\mc{N}(\frac{\eps_\mc{T}}{2}, [0,1], ||\cdot||_1)\right) \leq |\mc{S}|^2|\mc{A}|\log\left(1 + \frac{4}{\eps_\mc{T}}\right).$$ Applying these bounds following the earlier rate-distortion function upper bound to the result of Corollary \ref{thm:info_regret_bound} with $D = 0$, we have $$\textsc{BayesRegret}(K, \pi^{(1)},\ldots,\pi^{(K)}) \leq \sqrt{\overline{\Gamma}K\left(|\mc{S}||\mc{A}|\log\left(1 + \frac{4}{\eps_\mc{R}}\right) + |\mc{S}|^2|\mc{A}|\log\left(1 + \frac{4}{\eps_\mc{T}}\right) \right)}.$$
% \end{dproof}
\end{proof}

\section{Proof of Theorem \ref{thm:regret_decomp_qdist}}

Our proof of Theorem \ref{thm:regret_decomp_qdist} utilizes the following fact, widely known as the performance-difference lemma, adapted to the finite-horizon setting whose proof we replicate here.
\begin{lemma}[Performance-Difference Lemma~\citep{kakade2002approximately}]
For any finite-horizon MDP $\langle \mc{S}, \mc{A}, \mc{R}, \mc{T}, \beta, H \rangle$ and any two non-stationary policies $\pi_1,\pi_2 \in \Pi^H$, let $\rho^{\pi_2}(\tau)$ denote the distribution over trajectories induced by policy $\pi_2$. Then, $$V^{\pi_1}_1 - V^{\pi_2}_1 = \bE_{\tau \sim \rho^{\pi_2}}\left[\sum\limits_{h=1}^H \left(V^{\pi_1}_h(s_h) - Q^{\pi_1}_h(s_h,a_h)\right)\right].$$
\label{lemma:perform_diff}
\end{lemma}
\begin{proof}
% \begin{dproof}
\begin{align*}
    V^{\pi_1}_1 - V^{\pi_2}_1 &= \bE_{s_1 \sim \beta}\left[V^{\pi_1}_1(s_1) - V^{\pi_2}_1(s_1)\right] \\
    &= \bE_{s_1 \sim \beta}\left[V^{\pi_1}_1(s_1) - \bE_{\tau \sim \rho^{\pi_2}}\left[\sum\limits_{h=1}^H \mc{R}(s_h,a_h) \bigm| s_1 \right]\right] \\
    &= \bE_{\tau \sim \rho^{\pi_2}}\left[V^{\pi_1}_1(s_1) - \sum\limits_{h=1}^H \mc{R}(s_h,a_h) \right] \\
    &= \bE_{\tau \sim \rho^{\pi_2}}\left[V^{\pi_1}_1(s_1) + \sum\limits_{h=2}^H V^{\pi_1}_h(s_h) - \sum\limits_{h=1}^H \left(\mc{R}(s_h,a_h) - V^{\pi_1}_{h+1}(s_{h+1}) \right)\right] \\
    &= \bE_{\tau \sim \rho^{\pi_2}}\left[\sum\limits_{h=1}^H V^{\pi_1}_h(s_h) - \left(\mc{R}(s_h,a_h) + V^{\pi_1}_{h+1}(s_{h+1}) \right)\right] \\
    &= \bE_{\tau \sim \rho^{\pi_2}}\left[\sum\limits_{h=1}^H \left(V^{\pi_1}_h(s_h) - \left(\mc{R}(s_h,a_h) + \bE\left[V^{\pi_1}_{h+1}(s_{h+1}) \bigm| s_h, a_h\right] \right)\right)\right] \\
    &= \bE_{\tau \sim \rho^{\pi_2}}\left[\sum\limits_{h=1}^H \left(V^{\pi_1}_h(s_h) - Q^{\pi_1}_h(s_h,a_h) \right)\right],
\end{align*} 
where the penultimate line invokes the tower property of expectation.
% \end{dproof}
\end{proof}

\begin{theorem}
Fix any $D \geq 0$ and, for each episode $k \in [K]$, let $\widetilde{\mc{M}}_k$ be any MDP that achieves the rate-distortion limit of $\mc{R}^{Q^\star}_k(D)$ with information source $\bP(\mc{M}^\star \in \cdot \mid H_k)$ and distortion function $d_{Q^\star}$. Then, $$\textsc{BayesRegret}(K, \pi^{(1)},\ldots,\pi^{(K)}) \leq \bE\left[\sum\limits_{k=1}^K \bE_k\left[V^{\star}_{\widetilde{\mc{M}}_k,1} - V^{\pi^{(k)}}_{\widetilde{\mc{M}}_k,1}\right]\right] + (2H+2)K\sqrt{D}.$$
% \label{thm:regret_decomp_qdist}
\end{theorem}
\begin{proof}
% \begin{dproof}
By applying definitions from Section \ref{sec:problem_form} and applying the tower property of expectation, we have that 
$$\textsc{BayesRegret}(K, \pi^{(1)},\ldots,\pi^{(K)}) = \bE\left[\sum\limits_{k=1}^K \bE_k\left[\Delta_k\right]\right].$$ Examining the $k$th episode in isolation and applying the definition of episodic regret, we have
\begin{align*}
    \bE_k\left[\Delta_k\right] &= \bE_k\left[V^\star_{\mc{M}^\star,1} - V^{\pi^{(k)}}_{\mc{M}^\star, 1}\right] \\
    &= \bE_k\left[V^\star_{\mc{M}^\star,1} - V^\star_{\widetilde{\mc{M}}_k,1} + V^\star_{\widetilde{\mc{M}}_k,1} - V^{\pi^{(k)}}_{\widetilde{\mc{M}}_k,1} + V^{\pi^{(k)}}_{\widetilde{\mc{M}}_k,1} - V^{\pi^{(k)}}_{\mc{M}^\star, 1}\right] \\
    &= \bE_k\left[V^\star_{\mc{M}^\star,1} - V^\star_{\widetilde{\mc{M}}_k,1} + V^\star_{\widetilde{\mc{M}}_k,1} - V^{\pi^{(k)}}_{\widetilde{\mc{M}}_k,1} + \ubr{V^{\pi^{(k)}}_{\widetilde{\mc{M}}_k,1} - V^\star_{\widetilde{\mc{M}}_k,1}}_{\leq 0} + V^\star_{\widetilde{\mc{M}}_k,1} - V^{\pi^{(k)}}_{\mc{M}^\star, 1}\right] \\
    &\leq \bE_k\left[V^\star_{\mc{M}^\star,1} - V^\star_{\widetilde{\mc{M}}_k,1} + V^\star_{\widetilde{\mc{M}}_k,1} - V^{\pi^{(k)}}_{\widetilde{\mc{M}}_k,1} + V^\star_{\widetilde{\mc{M}}_k,1} - V^{\pi^{(k)}}_{\mc{M}^\star, 1}\right] \\
    &= \bE_k\left[V^\star_{\mc{M}^\star,1} - V^\star_{\widetilde{\mc{M}}_k,1} + V^\star_{\widetilde{\mc{M}}_k,1} - V^{\pi^{(k)}}_{\widetilde{\mc{M}}_k,1} + V^\star_{\widetilde{\mc{M}}_k,1} - V^\star_{\mc{M}^\star,1} + V^\star_{\mc{M}^\star,1} - V^{\pi^{(k)}}_{\mc{M}^\star, 1}\right].
\end{align*}
Observe that 
\begin{align*}
    \bE_k\left[V^\star_{\mc{M}^\star,1} - V^\star_{\widetilde{\mc{M}}_k,1}\right] &\leq \bE_k\left[||V^\star_{\mc{M}^\star,1} - V^\star_{\widetilde{\mc{M}}_k,1}||_\infty\right] \\
    &= \bE_k\left[\max\limits_{s \in \mc{S}} |V^\star_{\mc{M}^\star,1}(s) - V^\star_{\widetilde{\mc{M}}_k,1}(s)|\right] \\
    &= \bE_k\left[\max\limits_{s \in \mc{S}} |\max\limits_{a \in \mc{A}} Q^\star_{\mc{M}^\star,1}(s,a) - \max\limits_{a' \in \mc{A}} Q^\star_{\widetilde{\mc{M}}_k,1}(s,a')|\right] \\
    &\leq \bE_k\left[\max\limits_{s \in \mc{S}}\max\limits_{a \in \mc{A}} | Q^\star_{\mc{M}^\star,1}(s,a) -  Q^\star_{\widetilde{\mc{M}}_k,1}(s,a)|\right] \\
    &= \bE_k\left[||Q^\star_{\mc{M}^\star,1} - Q^\star_{\widetilde{\mc{M}}_k,1}||_\infty\right] \\
    &= \bE_k\left[\sqrt{||Q^\star_{\mc{M}^\star,1} - Q^\star_{\widetilde{\mc{M}}_k,1}||_\infty^2}\right] \\
    &\leq \sqrt{\bE_k\left[||Q^\star_{\mc{M}^\star,1} - Q^\star_{\widetilde{\mc{M}}_k,1}||_\infty^2\right]} \\
    &\leq \sqrt{\bE_k\left[\sup\limits_{h \in H} ||Q^\star_{\mc{M}^\star,h} - Q^\star_{\widetilde{\mc{M}}_k,h}||_\infty^2\right]} \\
    &= \sqrt{\bE_k\left[d_{Q^\star}(\mc{M}^\star,\widetilde{\mc{M}}_k)\right]} \\
    &\leq \sqrt{D},
\end{align*}
where the penultimate inequality is due to Jensen's inequality and the final inequality holds as $\widetilde{\mc{M}}_k$ achieves the rate-distortion limit under $d_{Q^\star}$, by assumption. Moreover, the exact argument can be repeated to see that
\begin{align*}
    \bE_k\left[V^\star_{\widetilde{\mc{M}}_k,1} - V^\star_{\mc{M}^\star,1}\right] &\leq \bE_k\left[||V^\star_{\widetilde{\mc{M}}_k,1} - V^\star_{\mc{M}^\star,1}||_\infty\right] \\
    &= \bE_k\left[||V^\star_{\mc{M}^\star,1} - V^\star_{\widetilde{\mc{M}}_k,1}||_\infty\right] \\
    &\leq \sqrt{D}.
\end{align*}
Combining these two inequalities yields
\begin{align*}
    \bE_k\left[\Delta_k\right] &\leq \bE_k\left[V^\star_{\mc{M}^\star,1} - V^\star_{\widetilde{\mc{M}}_k,1} + V^\star_{\widetilde{\mc{M}}_k,1} - V^{\pi^{(k)}}_{\widetilde{\mc{M}}_k,1} + V^\star_{\widetilde{\mc{M}}_k,1} - V^\star_{\mc{M}^\star,1} + V^\star_{\mc{M}^\star,1} - V^{\pi^{(k)}}_{\mc{M}^\star, 1}\right] \\
    &\leq \bE_k\left[V^\star_{\widetilde{\mc{M}}_k,1} - V^{\pi^{(k)}}_{\widetilde{\mc{M}}_k,1} + V^\star_{\mc{M}^\star,1} - V^{\pi^{(k)}}_{\mc{M}^\star, 1}\right] + 2\sqrt{D}.
\end{align*}

Observe that by virtue of posterior sampling ~\citep{russo2014learning,osband2013more,osband2017posterior} the compressed MDP being targeted by the agent $\widetilde{\mc{M}}_k$ and the sampled MDP $M_k$ are identically distributed, conditioned upon the information available within any history $H_k$, and so we have $$\bE_k\left[V^\star_{\mc{M}^\star,1} - V^{\pi^{(k)}}_{\mc{M}^\star, 1}\right] = \bE_k\left[V^\star_{\mc{M}^\star,1} - V^{\pi^\star_{M_k}}_{\mc{M}^\star, 1}\right] = \bE_k\left[V^\star_{\mc{M}^\star,1} - V^{\pi^\star_{\widetilde{\mc{M}}_k}}_{\mc{M}^\star, 1}\right].$$ Now applying the performance-difference lemma (Lemma \ref{lemma:perform_diff}), we see that
\begin{align*}
    \bE_k\left[V^\star_{\mc{M}^\star,1} - V^{\pi^\star_{\widetilde{\mc{M}}_k}}_{\mc{M}^\star, 1}\right] &= \bE_k\left[\bE_{\rho^{\pi^\star_{\widetilde{\mc{M}}_k}}}\left[\sum\limits_{h=1}^H \left(V^\star_{\mc{M}^\star,h}(s_h) - Q^\star_{\mc{M}^\star,h}(s_h,a_h)\right)\right]\right] \\
    &= \bE_k\left[\bE_{\rho^{\pi^\star_{\widetilde{\mc{M}}_k}}}\left[\sum\limits_{h=1}^H \left(\max\limits_{a \in \mc{A}} Q^\star_{\mc{M}^\star,h}(s_h, a) - Q^\star_{\mc{M}^\star,h}(s_h,a_h)\right)\right]\right] \\
    &\leq \bE_k\left[\bE_{\rho^{\pi^\star_{\widetilde{\mc{M}}_k}}}\left[\sum\limits_{h=1}^H \Big|\max\limits_{a \in \mc{A}} Q^\star_{\mc{M}^\star,h}(s_h, a) - Q^\star_{\mc{M}^\star,h}(s_h,a_h)\Big|\right]\right].
\end{align*}
Define $a^\star = \argmax\limits_{a \in \mc{A}} Q^\star_{\mc{M}^\star,h}(s_h, a)$ such that
\begin{align*}
    \bE_k\left[V^\star_{\mc{M}^\star,1} - V^{\pi^\star_{\widetilde{\mc{M}}_k}}_{\mc{M}^\star, 1}\right] &= \bE_k\left[\bE_{\rho^{\pi^\star_{\widetilde{\mc{M}}_k}}}\left[\sum\limits_{h=1}^H \Big|\max\limits_{a \in \mc{A}} Q^\star_{\mc{M}^\star,h}(s_h, a) - Q^\star_{\mc{M}^\star,h}(s_h,a_h)\Big|\right]\right] \\
    &= \bE_k\left[\bE_{\rho^{\pi^\star_{\widetilde{\mc{M}}_k}}}\left[\sum\limits_{h=1}^H \Big| Q^\star_{\mc{M}^\star,h}(s_h, a^\star) - Q^\star_{\mc{M}^\star,h}(s_h,a_h)\Big|\right]\right] \\
    &= \bE_k\left[\bE_{\rho^{\pi^\star_{\widetilde{\mc{M}}_k}}}\left[\sum\limits_{h=1}^H \Big| Q^\star_{\mc{M}^\star,h}(s_h, a^\star) - Q^\star_{\widetilde{\mc{M}}_k,h}(s_h, a^\star) + Q^\star_{\widetilde{\mc{M}}_k,h}(s_h, a^\star) - Q^\star_{\mc{M}^\star,h}(s_h,a_h)\Big|\right]\right].
\end{align*}
Applying the triangle inequality and examining each difference in isolation, we have
\begin{align*}
    \bE_k\left[\bE_{\rho^{\pi^\star_{\widetilde{\mc{M}}_k}}}\left[\sum\limits_{h=1}^H \Big|Q^\star_{\mc{M}^\star,h}(s_h, a^\star) - Q^\star_{\widetilde{\mc{M}}_k,h}(s_h, a^\star) \Big| \right]\right]
    &\leq \bE_k\left[\bE_{\rho^{\pi^\star_{\widetilde{\mc{M}}_k}}}\left[\sum\limits_{h=1}^H || Q^\star_{\mc{M}^\star,h} - Q^\star_{\widetilde{\mc{M}}_k,h}||_\infty \right]\right] \\
    &\leq H \bE_k\left[\sup\limits_{h \in H} || Q^\star_{\mc{M}^\star,h} - Q^\star_{\widetilde{\mc{M}}_k,h}||_\infty \right] \\
    &= H \bE_k\left[\sup\limits_{h \in H} \sqrt{|| Q^\star_{\mc{M}^\star,h} - Q^\star_{\widetilde{\mc{M}}_k,h}||_\infty^2} \right] \\
    &\leq H \sqrt{\bE_k\left[\sup\limits_{h \in H} || Q^\star_{\mc{M}^\star,h} - Q^\star_{\widetilde{\mc{M}}_k,h}||_\infty^2 \right]} \\
    &= H \sqrt{\bE_k\left[d_{Q^\star}(\mc{M}^\star, \widetilde{\mc{M}}_k)\right]} \\
    &\leq H \sqrt{D},
\end{align*}
where the penultimate inequality follows from Jensen's inequality and the final inequality follows since $\widetilde{\mc{M}}_k$ achieves the rate-distortion limit. 

For the remaining term, we have
\begin{align*}
    \bE_k\left[V^\star_{\mc{M}^\star,1} - V^{\pi^\star_{\widetilde{\mc{M}}_k}}_{\mc{M}^\star, 1}\right] &\leq H\sqrt{D} + \bE_k\left[\bE_{\rho^{\pi^\star_{\widetilde{\mc{M}}_k}}}\left[\sum\limits_{h=1}^H \Big| Q^\star_{\widetilde{\mc{M}}_k,h}(s_h, a^\star) - Q^\star_{\mc{M}^\star,h}(s_h,a_h)\Big|\right]\right] \\
    &= H\sqrt{D} + \bE_k\left[\bE_{\rho^{\pi^\star_{\widetilde{\mc{M}}_k}}}\left[\sum\limits_{h=1}^H \Big| Q^\star_{\widetilde{\mc{M}}_k,h}(s_h, a^\star) - Q^\star_{\widetilde{\mc{M}}_k,h}(s_h, a_h) + Q^\star_{\widetilde{\mc{M}}_k,h}(s_h, a_h) - Q^\star_{\mc{M}^\star,h}(s_h,a_h)\Big|\right]\right] \\
    &\leq H\sqrt{D} + \bE_k\left[\bE_{\rho^{\pi^\star_{\widetilde{\mc{M}}_k}}}\left[\sum\limits_{h=1}^H \Big|Q^\star_{\widetilde{\mc{M}}_k,h}(s_h, a_h) - Q^\star_{\mc{M}^\star,h}(s_h,a_h)\Big|\right]\right],
\end{align*}
where the inequality follows since $a_h$ is drawn from the optimal policy of $\widetilde{\mc{M}}_k$, $\pi^\star_{\widetilde{\mc{M}}_k}$, and so $Q^\star_{\widetilde{\mc{M}}_k,h}(s_h, a_h) \geq Q^\star_{\widetilde{\mc{M}}_k,h}(s_h, a^\star)$. Repeating the identical argument from above yields
\begin{align*}
    \bE_k\left[\bE_{\rho^{\pi^\star_{\widetilde{\mc{M}}_k}}}\left[\sum\limits_{h=1}^H \Big|Q^\star_{\widetilde{\mc{M}}_k,h}(s_h, a_h) - Q^\star_{\mc{M}^\star,h}(s_h,a_h)\Big|\right]\right] &\leq \bE_k\left[\bE_{\rho^{\pi^\star_{\widetilde{\mc{M}}_k}}}\left[\sum\limits_{h=1}^H || Q^\star_{\mc{M}^\star,h} - Q^\star_{\widetilde{\mc{M}}_k,h}||_\infty \right]\right] \\
    &\leq H\sqrt{D}.
\end{align*}
Substituting back, we see that
\begin{align*}
    \bE_k\left[V^\star_{\mc{M}^\star,1} - V^{\pi^{(k)}}_{\mc{M}^\star, 1}\right] &\leq \bE_k\left[\bE_{\rho^{\pi^\star_{\widetilde{\mc{M}}_k}}}\left[\sum\limits_{h=1}^H \Big|\max\limits_{a \in \mc{A}} Q^\star_{\mc{M}^\star,h}(s_h, a) - Q^\star_{\mc{M}^\star,h}(s_h,a_h)\Big|\right]\right] \leq 2H\sqrt{D}.
\end{align*}
Thus, we may complete our bound as
\begin{align*}
    \bE_k\left[\Delta_k\right] &\leq \bE_k\left[V^\star_{\mc{M}^\star,1} - V^\star_{\widetilde{\mc{M}}_k,1} + V^\star_{\widetilde{\mc{M}}_k,1} - V^{\pi^{(k)}}_{\widetilde{\mc{M}}_k,1} + V^\star_{\widetilde{\mc{M}}_k,1} - V^\star_{\mc{M}^\star,1} + V^\star_{\mc{M}^\star,1} - V^{\pi^{(k)}}_{\mc{M}^\star, 1}\right] \\
    &\leq \bE_k\left[V^\star_{\widetilde{\mc{M}}_k,1} - V^{\pi^{(k)}}_{\widetilde{\mc{M}}_k,1} + V^\star_{\mc{M}^\star,1} - V^{\pi^{(k)}}_{\mc{M}^\star, 1}\right] + 2\sqrt{D} \\
    &\leq \bE_k\left[V^\star_{\widetilde{\mc{M}}_k,1} - V^{\pi^{(k)}}_{\widetilde{\mc{M}}_k,1}\right] + (2H+2)\sqrt{D}.
\end{align*}
Applying this upper bound on episodic regret in each episode yields
\begin{align*}
    \textsc{BayesRegret}(K, \pi^{(1)},\ldots,\pi^{(K)}) &= \bE\left[\sum\limits_{k=1}^K \bE_k\left[\Delta_k\right]\right] \\
    &\leq \bE\left[\sum\limits_{k=1}^K \bE_k\left[V^\star_{\widetilde{\mc{M}}_k,1} - V^{\pi^{(k)}}_{\widetilde{\mc{M}}_k,1}\right] \right] + 2K(H+1)\sqrt{D},
\end{align*}
as desired.
% \end{dproof}
\end{proof}

\section{Proof of Lemma \ref{lemma:info_bottle_qdist}}

To show Lemma \ref{lemma:info_bottle_qdist}, we prove the following more general result which applies whenever a distortion function adheres to a specific functional form.

Let $V, \widehat{V}$ be two arbitrary random variables defined on the same measurable space $(\mc{V}, \mathbb{V})$ and define the associated rate-distortion function as $$\mc{R}(D) = \inf\limits_{\widehat{V} \in \Lambda(D)} \bI(V;\widehat{V}) = \inf\limits_{\widehat{V} \in \Lambda(D)} \kl{\bP((V, \widehat{V}) \in \cdot)}{\bP(V \in \cdot) \times \bP(\widehat{V} \in \cdot)},$$ where the distortion function $d:\mc{V} \times \mc{V} \ra \bR_{\geq 0}$ has the form $d(v,\widehat{v}) = \ell(f(v),f(\widehat{v}))$ for any two known, deterministic functions, $f: \mc{V} \ra \mc{Z}$ and a semi-metric $\ell: \mc{Z} \times \mc{Z} \ra \bR_{\geq 0}$. Effectively, this structural constraint says that our distortion measure between the original $V$ and compressed $\widehat{V}$ only depends on the statistics $f(V)$ and $f(\widehat{V})$. Under such a constraint, we may prove the following lemma
\begin{lemma}
If $D = 0$ and $\widehat{V}$ achieves the rate-distortion limit, then we have the Markov chain $V \ra f(V) \ra \widehat{V}$
\label{lemma:info_bottle}
\end{lemma}
\begin{proof}
% \begin{dproof}
Assume for the sake of contradiction that there exists a random variable $\widehat{V}$ that achieves the rate-distortion limit with $D = 0$ but does not induce the Markov chain $V \ra f(V) \ra \widehat{V}$. Since mutual information is non-negative and $\bI(V; \widehat{V} \mid f(V)) = 0$ implies the Markov chain $V \ra f(V) \ra \widehat{V}$, it must be the case that $\bI(V; \widehat{V} \mid f(V)) > 0$. Consider an independent random variable $\widehat{V}' \sim \bP(\widehat{V} \mid f(V))$ such that $$\bI(V; \widehat{V}') = \bI(V;\widehat{V}) - \ubr{\bI(V; \widehat{V} \mid f(V))}_{> 0} < \bI(V;\widehat{V}) = \mc{R}(D).$$ Clearly, we have retained all bits of information needed to preserve $f(V)$ in $\widehat{V}'$, thereby achieving the same expected distortion constraint. However, this implies that $\widehat{V}'$ achieves a strictly lower rate, contradicting our assumption that $\widehat{V}$ achieves the rate-distortion limit. Therefore, it must be the case that when $D = 0$ and $\widehat{V}$ achieves the rate-distortion limit, we have $\bI(V; \widehat{V} \mid f(V)) = 0$ which implies the Markov chain $V \ra f(V) \ra \widehat{V}$.

% However, this implies that there are exactly $\bI(V; \widehat{V} \mid f(V))$ bits of information preserved in the compression $\widehat{V}$ from $V$ which do not impact $f(V)$. Since our distortion measure is exclusively concerned with preserving $f(V)$ and we have $\bE\left[d(V,\widehat{V})\right] = \bE\left[\ell(f(V),f(\widehat{V}))\right] \leq 0$ $\implies$ $f(V) \overset{\text{a.s.}}{=} f(\widehat{V})$, there must exist another compression $\widehat{V}'$ which achieves the same expected distortion but excludes these $\bI(V; \widehat{V} \mid f(V))$ bits of information. However, this implies that $\widehat{V}'$ achieves a lower rate, contradicting our assumption that $\widehat{V}$ achieves the rate-distortion limit.

% Therefore, it must be the case that when $D = 0$ and $\widehat{V}$ achieves the rate-distortion limit, we have $\bI(V; \widehat{V} \mid f(V)) = 0$ which implies the Markov chain $V \ra f(V) \ra \widehat{V}$.
% \end{dproof}
\end{proof}

\begin{lemma}
For each episode $k \in [K]$ and for $D = 0$, let $\widetilde{\mc{M}}_k$ be a MDP that achieves the rate-distortion limit of $\mc{R}^{Q^\star}_k(D)$ with information source $\bP(\mc{M}^\star \mid H_k)$ and distortion function $d_{Q^\star}$. Then, we have the Markov chain $\mc{M}^\star \ra Q^\star_{\mc{M}^\star} \ra \widetilde{\mc{M}}_k$, where $Q^\star_{\mc{M}^\star} = \{Q^\star_{\mc{M}^\star,h}\}_{h \in [H]}$ is the collection of random variables denoting the optimal action-value functions of $\mc{M}^\star$.
% \label{lemma:info_bottle_qdist}
\end{lemma}
\begin{proof}
% \begin{dproof}
Recall that our distortion function, $$d_{Q^\star}(\mc{M}, \widehat{\mc{M}}) = \sup\limits_{h \in [H]} ||Q^\star_{\mc{M},h} - Q^\star_{\widehat{\mc{M}},h}||_\infty^2 = \sup\limits_{h \in [H]} \max\limits_{(s,a) \in \mc{S} \times \mc{A}} | Q^\star_{\mc{M},h}(s,a) - Q^\star_{\widehat{\mc{M}},h}(s,a)|^2,$$ only depends on the MDPs $\mc{M}$ and $\widehat{\mc{M}}$ through their respective optimal action-value functions, $\{Q^\star_{\mc{M},h}\}_{h \in [H]}$ and $\{Q^\star_{\widehat{\mc{M}},h}\}_{h \in [H]}$. Consequently, the claim holds immediately by applying Lemma \ref{lemma:info_bottle} where $f$ computes the optimal action-value functions of an input MDP for each timestep $h \in [H]$ and $\ell$ is the metric induced by the infinity norm on $\bR^{|\mc{S}| \times |\mc{A}|}$.  
% \end{dproof}
\end{proof}

\section{Proof of Theorem \ref{thm:tabular_regret_qdist}}

Our proof of Theorem \ref{thm:tabular_regret_qdist} proceeds by leveraging Lemma \ref{lemma:info_bottle_qdist} (instead of Fact \ref{fact:rdf_props}) before following the same style of argument as used in Theorem \ref{thm:tabular_regret}.
\begin{theorem}
For $D = 0$ and any prior distribution $\bP(\mc{M}^\star \in \cdot \mid H_1)$ over tabular MDPs, if $\Gamma_k \leq \overline{\Gamma}$ for all $k \in [K]$, then VSRL with distortion function $d_{Q^\star}$ has
$$\textsc{BayesRegret}(K, \pi^{(1)},\ldots,\pi^{(K)}) \leq \widetilde{\mc{O}}\left(\sqrt{\overline{\Gamma}|\mc{S}||\mc{A}|KH}\right).$$
% \label{thm:tabular_regret_qdist}
\end{theorem}
\begin{proof}
% \begin{dproof}
Starting with the information-theoretic regret bound in Corollary \ref{thm:info_regret_bound_qdist}, observe that for $\mc{M}^\star \sim \bP(\mc{M}^\star \in \cdot \mid H_1)$, we have the Markov chain $\mc{M}^\star \ra Q^\star_{\mc{M}^\star} \ra \widetilde{\mc{M}}_1$, by virtue of Lemma \ref{lemma:info_bottle_qdist}. By the data-processing inequality, we immediately recover the following chain of inequalities: $$\mc{R}^{Q^\star}_1(D) \leq \bI_1(\mc{M}^\star; \widetilde{\mc{M}}_1) \leq \bI_1(\mc{M}^\star; Q^\star_{\mc{M}^\star}).$$ Recognizing that the optimal value functions are a deterministic function of the MDP $\mc{M}^\star$ itself, we have $$\bI_1(\mc{M}^\star; Q^\star_{\mc{M}^\star}) = \bH_k(Q^\star_{\mc{M}^\star}) - \bH_k(Q^\star_{\mc{M}^\star} \mid \mc{M}^\star) = \bH_k(Q^\star_{\mc{M}^\star}) = \bH_k(Q^\star_{\mc{M}^\star,1},\ldots,Q^\star_{\mc{M}^\star,H}) \leq \sum\limits_{h=1}^H \bH_k(Q^\star_{\mc{M}^\star,h}),$$ where the final inequality follows by applying the chain rule of entropy and the fact that conditioning reduces entropy, in sequence.

At this point, recalling the salient exposition in the proof of Theorem \ref{thm:tabular_regret} concerning the use of metric entropy for such function-valued random variables, we proceed to consider the $\eps_{Q^\star}$-cover of the interval $[0,H]$ with respect to the $L_1$-norm $||\cdot||_1$, for some fixed $0 < \eps_{Q^\star} < H$. Since, for a sufficiently small choice of $\eps_{Q^\star}$, $Q^\star_{\mc{M}^\star,h}$ is well-approximated as a discrete random variable for any $h \in [H]$, we recall that the entropy of a discrete random variable $X$ taking values on $\mc{X}$ is bounded as $\bH(X) \leq \log\left(|\mc{X}|\right)$. Applying this upper bound and Lemma \ref{lemma:metric_entropy_bound} in sequence, we have that $$\sum\limits_{h=1}^H \bH_k(Q^\star_{\mc{M}^\star,h}) \leq |\mc{S}||\mc{A}|H\log\left(\mc{N}(\frac{\eps_{Q^\star}}{2}, [0,H], ||\cdot||_1)\right) \leq |\mc{S}||\mc{A}|H\log\left(1 + \frac{4H}{\eps_{Q^\star}}\right).$$ Applying these upper bounds to the result of Corollary \ref{thm:tabular_regret_qdist} and recalling that $D = 0$, we have $$\textsc{BayesRegret}(K, \pi^{(1)},\ldots,\pi^{(K)}) \leq \sqrt{\overline{\Gamma}K|\mc{S}||\mc{A}|H\log\left(1 + \frac{4H}{\eps_{Q^\star}}\right)}.$$
% \end{dproof}
\end{proof}

\end{document}